\documentclass{article}

\usepackage{PRIMEarxiv}

\usepackage[utf8]{inputenc} 
\usepackage[T1]{fontenc}    
\usepackage{hyperref}       
\usepackage{url}            
\usepackage{booktabs}       
\usepackage{amsfonts}       
\usepackage{nicefrac}       
\usepackage{microtype}      
\usepackage{lipsum}
\usepackage{fancyhdr}       
\usepackage{graphicx}       
\graphicspath{{media/}}     

\usepackage[numbers]{natbib}
\usepackage[english]{babel}
\usepackage{amsmath}
\usepackage{stmaryrd}
\usepackage{comment}
\usepackage{bbm}
\usepackage{amssymb}
\usepackage{amsthm}
\usepackage{bbold}
\usepackage{algorithm2e}
\usepackage{subfigure}
\usepackage{subcaption}

\newcommand{\Lu} {\mathbf L^{1}}
\newcommand{\Ca} {\mathbf C^{\alpha}}

\newcommand{\ccc} {f}
\newcommand{\JJ} {{j_{M}}}
\newcommand{\LL} {{j_{m}}}
\newcommand{\C} {{\bf C}}
\newcommand{\R} {{\mathbb R}}
\newcommand{\N} {{\mathbb N}}
\newcommand{\V} {{\bf V}}
\newcommand{\Ld} {{\bf L}^2}

\newcommand\EE{\mathbb{E}}

\newcommand {\rb} {\rangle}
\newcommand {\lb} {\langle}
\newcommand {\om} {\omega}

\DeclareMathOperator*{\argmin}{argmin}

\newtheorem{definition}{Definition}[section]
\newtheorem{theorem}{Theorem}[section]
\newtheorem*{theorem*}{Theorem}
\newtheorem{conjecture}[theorem]{Conjecture}

\newtheorem{lemma}[theorem]{Lemma}

\usepackage{natbib}

\title{Beyond sparse denoising in frames: minimax estimation with a scattering transform}

\author{
  Nathanaël Cuvelle--Magar\thanks{Corresponding author.}\\
  Ecole normale supérieure\\
  \texttt{nmagar@di.ens.fr} \\
   \And
  Stéphane Mallat \\
  Collège de France\\
  Ecole normale supérieure\\
  Flatiron Institute, Simons Foundation\\
  \texttt{stephane.mallat@ens.fr} \\
}

\begin{document}
\maketitle

\begin{abstract}

A considerable amount of research in harmonic analysis has been devoted to non-linear estimators of signals contaminated by additive Gaussian noise. They are implemented by thresholding coefficients in a frame, which provide a sparse signal representation, or by minimising their $\ell^1$ norm. However, sparse estimators in frames are not sufficiently rich to adapt to complex signal regularities.
For cartoon images whose edges are piecewise $\C^\alpha$ curves, wavelet, curvelet and Xlet frames are suboptimal if the Lipschitz exponent $\alpha \leq 2$ is an unknown parameter. Deep convolutional neural networks have recently obtained much better numerical results, which reach the minimax asymptotic bounds for all $\alpha$. 
Wavelet scattering coefficients have been introduced as simplified convolutional neural network models. They are computed
by transforming the modulus of wavelet coefficients with a second wavelet transform. We introduce a denoising estimator by jointly minimising and maximising the $\ell^1$ norms of different subsets of scattering coefficients. We prove that these $\ell^1$ norms capture different types of geometric image regularity. Numerical experiments show that this denoising estimator reaches the minimax asymptotic bound for cartoon images for all Lipschitz exponents $\alpha \leq 2$. We state this numerical result as a mathematical conjecture.
It provides a different harmonic analysis approach to suppress noise from signals, and to specify the geometric regularity of functions. It also opens a mathematical bridge between harmonic analysis and denoising estimators with deep convolutional network.

\end{abstract}


\section{Introduction}

Signals may be contaminated by additive Gaussian white noise during measurements or transmission. Such noise removal has thus been a central signal processing topic since the early work of Wiener on optimal linear denoisers \cite{Wiener1949ExtrapolationIA}. Donoho and Johnstone \cite{Donoho1994IdealSA} have proved that nearly optimal denoising estimators are obtained over
signal classes which have sparse decomposition coefficients over a basis or a frame. The denoising estimator is computed by thresholding the noisy signal coefficients in this frame, or by minimising their $\ell^1$ norm as explained in Section \ref{sec:sparse_denoising}. 
In wavelet bases, it defines nearly minimax estimators over balls of Besov spaces \cite{donoho1996density}, which correspond to functions having specific local isotropic regularity properties. 
Section \ref{sec:wavelet} gives a brief review. 
For images having anisotropic directional regularities, a considerable effort has been devoted to build frames such as curvelets or shearlets, which define sparse representations.
However, bases and frames have a limited flexibility to adapt to various geometric regularities. A major issue in harmonic analysis addressed in this paper is to develop new mathematical tools to capture sparsity properties and improve these estimators. The paper gives a mathematical conjecture sustained by numerical results and partial mathematical justifications to do so, after reviewing important harmonic analysis results.

Images with edges having some geometric regularity are important examples introduced by Tsybakov and Korostelev \cite{Korostelev1993MinimaxTO}. We shall consider an image model \cite{Peyr2008OrthogonalBB} of piecewise $\C^\alpha$ (uniformly Lipschitz $\alpha$) geometric images, which are ${\bf C}^\alpha$ over disjoint image domains separated by regular edges. These edges are themselves regular ${\bf C}^\alpha$ curves. Edge junctions define corners which are
geometric singularities. Sparse representations are obtained with curvelets and shearlets if $\alpha = 2$ and if there is no corner, as explained in Section \ref{sec:Xlets}. However, there is no frame providing an optimally sparse representation for such images having corners, for a Lipschitz exponent $\alpha \leq 2$ which may vary. Indeed, it would require to adapt the geometry of frame vectors to the geometric regularity of edges in the image. If the edge regularity varies then we must choose the decomposition vectors in dictionaries of functions which are much larger than a frame. Optimally sparse representations have been obtained by selecting an orthonormal basis in a dictionary of bandlets 
\cite{Peyr2008OrthogonalBB,Dossal2008BandletIE}, whose geometry is adapted to each noisy 
image in terms of location, orientation and degree of regularity of edges. However, this approach becomes progressively too complex as the signal diversity increases, because the dictionary size must also increase.

More recently, Gaussian noise removal has gained considerable attention for AI data generation with score diffusion. New data are generated by transporting a Gaussian white noise with a stochastic differential equation which performs a progressive denoising, with a drift term 
computed from a denoiser implemented with a deep neural network \cite{song2021scorebased, ho2020denoising, kadkhodaie2021stochastic}.
Denoising capabilities of deep convolutional neural networks go well beyond the performance of previous estimators based on wavelets, curvelets, Xlets and bandlets but their mathematical properties are not understood.

A deep convolutional network computes a cascade of convolutional operators which have a narrow square support, and pointwise non-linearities. Simplified wavelet scattering models have been introduced in \cite{mallat2012group, ScatteringBruna} by replacing these convolutional operators with wavelet transforms and a modulus non-linearity. As in deep convolutional networks, all wavelet filters have a square support, as explained in Section \ref{sec:dyadic_wavelet}. The main theorem of the paper in Section \ref{sec:geometric_regularity_scattering} proves that applying a second wavelet transform over the modulus of a first wavelet transform can still take advantage of the anisotropic geometric regularity of edges. This regularity appears through 
the fast decay of the $\ell^1$ norm of a subset of scattering coefficients at fine scales.
More surprisingly, the concentration of large wavelet coefficients along edge curves is also captured by the $\ell^1$ norms of another subset of scattering coefficients, which become large and not small. Geometric image regularities are thus captured by both small and large $\ell^1$ norms of scattering coefficients, which only involve wavelet filters.

A denoising algorithm is implemented by minimising a weighted average of these scattering $\ell^1$ norms with positive and negative weights to insure a joint minimisation and maximisation of these norms. Numerical simulations in Section \ref{sec:numerics} demonstrate that this estimator reaches the minimax asymptotic rate over images having $\C^\alpha$ edges for all values of $\alpha \leq 2$. All
experiments are reproducible with the software available on \url{https://github.com/ScatteringDenoisingCode/SD}. Based on these numerical results, we state a mathematical conjecture on the minimax denoising rate of such denoising algorithm, which is not proved. This conjecture gives a new characterization of sparsity properties through minimisation and maximisation of different $\ell^1$ norms. It does not require to use large size dictionaries with vectors whose support are adapted to the image geometry. It is a mathematical bridge towards deep convolutional network estimators.

\section{Denoising geometrically regular images}
\label{sec:background}

This first section reviews the minimax optimality of denoising estimators, which relies on sparsity properties in harmonic analysis bases or frames. Section \ref{sec:sparse_denoising} reviews minimax estimations with regularised estimators promoting sparsity. Section \ref{sec:wavelet} reviews sparse estimations in wavelet bases, which are suboptimal to remove noise from images having some form 
of geometric regularity. Section \ref{sec:Xlets} reviews results obtained with curvelets and shearlets frames, bandlet dictionaries and deep convolutional neural networks.

\subsection{Variational and sparse denoising estimations}
\label{sec:sparse_denoising}

Suppose that $f \in \Ld[0,1]^2$ is contaminated by the addition of a Wiener measure of variance $\sigma^2$ (Gaussian white noise): $g = f + dB$, where $B$ is a Brownian motion of variance $\sigma^2$.
We shall consider functions $f$ in a set $\Lambda$ having specific regularity properties. 
An estimator $\hat f(g)$ computes an approximation of $f$ from $g$. It has a mean-squared error (MSE)
$\EE_g (\|\hat f(g) - f \|^2)$, where the expected value is computed over the Gaussian white noise distribution of $g$. 
The minimax error is the infimum which is achievable by optimizing the estimator $\hat f$
\[
\epsilon_{m} (\sigma) = \inf_{\hat f} \sup_{f \in \Lambda} \EE_g (\|\hat f(g) - f \|^2) . 
\]
If $f$ is a realization of a random process whose probability distribution is supported within $\Lambda$, then the minimum mean-squared error is
calculated with an expected value over the distribution of $f$ and $g$
\[
\epsilon_{mms} (\sigma) =  \inf_{\hat f}  \EE_{f,g} (\|\hat f(g) - f \|^2)~~\mbox{and}~~\epsilon_{mms} \leq \epsilon_m.
\]
Most work in harmonic analysis have been focused on the calculation of the minimax error $\epsilon_m$ which gives an upper bound on $\epsilon_{mms}$ overall distributions.
Depending upon the properties of $\Lambda$, our goal is to define a denoising estimator $\hat f(g)$ whose maximum error over $\Lambda$
\[
\epsilon_{M} (\sigma) = \sup_{f \in \Lambda} \EE_{g} (\|\hat f(g) - f \|^2) 
\]
is close to the theoretical bounds computed over the minimax error $\epsilon_m$. This maximum error is difficult to compute numerically. It is replaced in numerical experiments by a mean-squared error
$\epsilon_{ms} (\sigma) = \EE_{f,g} (\|\hat f(g) - f \|^2) $, 
for random processes whose typical realizations nearly reach the maximum error.

\paragraph{Regularised denoising estimators}
In numerical applications, the measurement device such as a camera gives
a finite number $d$ of coefficients
which depend linearly on the input signal $g$. It thus characterises an orthogonal projection $P_\V g$ of $g$
in a vectorial subspace $\V$ of dimension $d$. The measurements project the signal as well as the noise. The projected noise $Z = P_\V (dB)$ is a random 
Gaussian white noise vector of dimension $d$, so 
\begin{equation}
\label{eq:discrete-noise}
P_\V g = P_\V f + Z~~\mbox{where}~~Z\sim {\cal N}(0 , \sigma^2 Id) .
\end{equation}
The problem is thus restricted to the space $\V$ of dimension $d$. This dimension varies depending on the noise level. Indeed, the resolution at which the original signal could be restored from its noisy counterpart is a function of $\sigma$.

We shall consider denoising estimators $\hat f$ of $P_\V f$ from $P_\V g$,
which takes advantage of the prior information that $f$ belongs to $\Lambda$.
This may be captured by a regularity measure $U(f).$
Regularized estimators impose that $U(\hat f)$ is bounded by $\sup_{f \in \Lambda} U(f)$. For example, a Tikhonov estimation \cite{tikhonov1963solution, tikhonov1977solutions} assumes that 
$\Lambda$ is a ball in a Sobolev space which is specified by a Sobolev norm $U(f)$.
Given the regularity prior $U$, an estimation $\hat f(g) \in \V$ of $P_V f$ is obtained as a solution of the variational problem
\begin{equation}
\label{eq:variational_denoising}
    \hat f(g) = \argmin_{h \in \V} \frac 1 {2} \|h - P_\V g \|^2 + \sigma^2 U(h) .
\end{equation}
The error can be decomposed into an approximation and an estimation error
\begin{equation}
\label{eq:approx-est}
\|f - \hat f(g) \|^2 = \|f - P_\V f \|^2 + \|P_\V f - \hat f(g) \|^2 .
\end{equation}
In the following we will assume that the discretisation dimension $d$ is large enough so that
the leading term is the estimation error $\|P_\V f - \hat f(g) \|^2.$
When the prior set $\Lambda$ is more complex than a ball in a Banach space, 
a major difficulty is to find $U(f)$ which defines a regularised estimator which is asymptotically minimax in $\Lambda$.

\paragraph{Sparsity with $\ell^1$ norms in bases or frames}
An important body of research in applied harmonic analysis shows that 
nearly minimax estimators may be obtained by finding a basis $\{\psi_n \}_{n \leq d}$ of $\V$, where the decomposition coefficients
$W f = \{ \lb f , \psi_n \rb \}_{n \leq d}$ are sufficiently sparse for all $f \in \Lambda$. 
This sparsity can be measured by an $\ell^1$ norm of the vector $Wf$ of decomposition coefficients
\begin{equation}
    \label{sparse-prior}
U(f) = \lambda\, \|W f\|_1 .
\end{equation}
If the basis $\{\psi_n \}_{n\leq d}$ is orthonormal then
one can verify \cite{donoho_soft_thresholding, antoniadis2001regularization} that the solution $\hat f(g)$ of the variational problem (\ref{eq:variational_denoising}) is
computed by thresholding the coefficients of $Wf$ with the soft-thresholding function
\[
\rho_T (a) = {\rm sign}(a) \max(|a| - T, 0) 
\]
with  $T= \lambda\, \sigma$. The estimation is calculated by applying the inverse $W^T$ of $W$ 
\begin{equation}
\label{eq:soft-thresh}
\hat f(g) = W^{T} \rho_{\lambda \sigma} (W g)  . 
\end{equation}
This thresholding can also be interpreted as a two layers neural network, whose weights are defined by $W$ and $W^T$,
with the point-wise discontinuity defined by the soft thresholding that can also be decomposed into a difference of two $ReLUs$. 
If the basis $\{\psi_n \}_{n \leq d}$ is not orthogonal in which case $W$ in (\ref{sparse-prior}) is not unitary then $\hat f(g)$ 
is computed with an iterative proximal algorithm \cite{Daubechies2003AnIT, kowalski2014thresholding}. Such algorithms can be efficiently approximated with a deep neural network \cite{Gregor2010LearningFA, Liu2018ALISTAAW}.

\paragraph{Asymptotic optimality}
Donoho and Johnstone \cite{Donoho1994IdealSA, donoho2002noising} 
have proved that thresholding estimators can achieve
nearly minimax errors in a basis which provides a nearly ``optimal" sparse decomposition.
We briefly summarize the main results.
In the following we write $u(d) \sim v(d)$ if there exists $B \geq A > 0$ such that for all $d \in \N$ 
\[
A\, u(d) \leq v(d) \leq B\, u(d) .
\]
A fundamental theorem of Donoho and Johnstone \cite{Donoho1994IdealSA} proves that 
a thresholding estimator (\ref{eq:soft-thresh}) for $\lambda = \sqrt{2 \log d}$ satisfies
\[
\EE (\|P_\V f - \hat f(g) \|^2) \leq (2 \log d + 1) \big( \sigma^2 +  \sum_{n=1}^d \min(|\lb f , \psi_n \rb|^2,\sigma^2) \big) .
\]
The upper bound is controlled by the sparsity of the decomposition coefficients
$|\lb f , \psi_n \rb|$. This sparsity is measured by sorting them in decreasing amplitude order. We denote $|\lb f , \psi_{n_r} \rb|$ the sorted coefficient of rank $r$: $|\lb f , \psi_{n_r} \rb| \geq |\lb f , \psi_{n_{r+1}} \rb|$. 
If $|\lb f , \psi_{n_r} \rb|^2 \sim r^{-(\alpha+1)}$ one can verify that
\[
\sum_{n=1}^d \min(|\lb f , \psi_n \rb|^2,\sigma^2) \sim \sigma^{2 \alpha/(\alpha+1)}.
\]
Suppose that $\Lambda$ is a ``tail compact" set, which means that there exists $C > 0$ and $\epsilon > 0$ such that for all $f \in \Lambda$ and all $M > 0$, we have 
$\sum_{n \geq M} |\lb f , \psi_n \rb|^2 \leq C\, M^{-\epsilon}$ . By choosing $d = \sigma^{-\beta}$ for $\beta$ large enough we then verify with 
(\ref{eq:approx-est}) that there exists $C > 0$ such that the maximum error on
$\Lambda$ satisfies
\[
\epsilon_M (\sigma) \leq C\, \sigma^{2 \alpha/(\alpha+1)} |\log \sigma|~.
\]
The last step proves that this upper bound gives the minimax error over $\Lambda$
up to the $\log \sigma$ term. \citet{donoho1993unconditional} proves that if there exists two balls $\Gamma_1$ and $\Gamma_2$ of a Banach space such that $\Gamma_1 \subset \Lambda \subset \Gamma_2$ and if the sparsity basis
$\{ \psi_n \}_{n \in \N}$ defines an unconditional basis of this Banach space then
there exists $C' > 0$ such that the minimax error satisfies
\[
\epsilon_m (\sigma) \geq C'\,\sigma^{2 \alpha/(\alpha+1)} 
\]
and hence
\begin{equation}
\label{eq:asymptot-optm}
\epsilon_m (\sigma) \leq \epsilon_M (\sigma) \leq C/C'\, \epsilon_m (\sigma)\, |\log \sigma|~.
\end{equation}
Neglecting the $\log \sigma$ term, we shall then say that the thresholding estimator
is asymptotically minimax. Next section applies these results to wavelet bases which define unconditional bases of all Besov spaces \cite{meyer1992wavelets}.

\subsection{Denoising in wavelet orthonormal bases}
\label{sec:wavelet}

\paragraph{Wavelet orthonormal bases}
A wavelet orthonormal basis of $\Ld[0,1]^2$ is constructed with three mother
wavelets $\psi^k$ for $1 \leq k \leq 3$, which are dilated $\psi^k_j (u) = 2^{-j} \psi^k (2^{-j} u)$ at different scales $2^j < 1$, and translated, modulo boundary
conditions \cite{Mallat1989ATF}.
It defines a wavelet orthonormal basis 
\[
\{ \psi_j^k (2^j n - u) \}_{j < 0, 1 \leq k \leq 3, 2^j n \in [0,1]^2}. 
\]
The inner products of $f$ with translated wavelets can also be rewritten as convolutions sampled at intervals $2^j$:
\[
\lb f(u) , \psi^k_j (2^j n - u) \rb = f * \psi^k_j (2^j n) .
\]
A finite dimensional approximation $P_\V f$ of $f$ at the scale $2^\LL$ is defined by setting to zero all wavelet coefficients
at scales $2^j \leq 2^\LL$. It projects $f$
in the space $\V$ of dimension $d = 2^{-2\LL}$ generated by wavelets having a scale $2^j \geq 2^\LL$. This space captures the low frequencies of $f$. It
also admits an orthonormal basis of $2^{-2\LL}$ translated
scaling functions $\{ \phi_\LL(2^\LL n - u) \}_{2^\LL n \in [0,1]^2}$ where $\phi_\LL(u) = 2^{-\LL} \phi(2^{-\LL} u)$. The Fourier transform of $\phi_\LL$ is mostly concentrated in the low-frequency
square $[-2^{-\LL} \pi, 2^{-\LL} \pi]$. The orthogonal projection $P_\V f$ is thus also represented by a discrete image of 
$d = 2^{-2\LL}$ coefficients $f * \phi_\LL (2^\LL n)$ which are local averages of $f$ in the neighbourhood of each grid point $2^\LL n$. This discrete image is typically provided as input measures.

One may also limit the maximum wavelet scale to $2^\JJ \leq 1$. The $2^{-2\JJ}$ wavelets at larger scales generate the same space as the family of scaling functions
$\{ \phi_\JJ(2^\JJ n - u) \}_{2^\JJ n \in [0,1]^2}$. The wavelet representation up to the scale $2^\JJ$ of $P_\V f$ is thus defined by
\[
W f = \Big\{ f * \psi^k_{j} (2^j n) \Big\}_{\LL \ge j \leq \JJ,2^{-j} n \in [0,1]^2,1 \leq k \leq 3} \cup \Big\{  f * \phi_\JJ (2^\JJ n) \Big\}_{ 2^{-\JJ} n \in [0,1]^2} .
\]
These discrete wavelet coefficients are computed at scales $2^j \geq 2^\LL$ with a fast discrete wavelet transform, with $O(d)$ operations \cite{mallat1999wavelet}. 
The amplitudes of wavelet coefficients depend upon the local regularity of $f$.
The more regular $f$  the smaller $| f * \psi_{j}^k(2^j n)|$ at fine scales $2^j$, as quantified by equation (\ref{decayCalpha}) for Lipschitz $\alpha$ functions. Figures \ref{fig:original_image} and \ref{fig:original_wav_decomp} show respectively a piecewise regular image along with its discrete orthogonal wavelet coefficients. Large coefficients are located near edges, where the image is discontinuous.
A wavelet thresholding estimator sets to zero all coefficients below a threshold
$T$ proportional to the noise standard deviation $\sigma$. 
Setting to zero a wavelet coefficient eliminates a high frequency component, which is equivalent to eliminate the noise with a local averaging operator. This averaging operator adapts to the amplitude of wavelet coefficients and hence to the local signal regularity.

\paragraph{Wavelet thresholding optimality for uniformly Lipschitz models}

The Lipschitz regularity of a function $f$ in the two-dimensional neighbourhood of a point $u \in \Omega$ is defined by the error of a polynomial approximation.
A function $f \in \Ld[0,1]^2$ is uniformly Lipschitz $\alpha$ 
on $\Omega$ with Lipschitz constant $L > 0$ if 
for any $u \in \Omega$ there exists a polynomial $q_u$ of degree smaller than $\alpha$ such that
\[
\forall u' \in \Omega~~,~~|f (u') - q_u (u')| \leq L\, |u - u'|^\alpha .  
\]
The minimal value of $L$ defines the norm $\|f \|_{\Ca}$ of $f$ in the H\"older space $\Ca$ over $\Omega$.
The set $\Lambda \subset \Ld(\Omega)$ of all functions $f$ which are Lipschitz $\alpha$ for a constant $L$ defines a determinist model of images corresponding to a ball in this H\"older space.

We shall suppose that  wavelet $\psi^k$ has $m$ vanishing moments, which means that for any polynomial $q(u)$ of degree $m$ we have $$\int q(u) \psi^k(u) du = 0.$$  Wavelet coefficients provide a characterization of H\"older spaces ${\mathbf C}^{\alpha}[0,1]^2$ of uniformly Lipschitz $\alpha$ functions.
For $\alpha \le m$, one can prove \cite{Jaffard1991PointwiseST, meyer1992wavelets} that $f \in {\mathbf C}^{\alpha}[0,1]^2$ if and only if there exists $C > 0$ such that for all $2^j n \leq [0,1]^2$
\begin{equation}
\label{decayCalpha}
|f *  \psi_{j}^k (2^j n) |  \leq C 2^{(\alpha +1) j} .
\end{equation}
One can verify that this inequality implies that sorted wavelet coefficients
of $f$ of rank $r$ decay at most like $r^{-(\alpha+1)/2}$. Wavelets define unconditional
bases of H\"older spaces \cite{meyer1992wavelets}. It results from (\ref{eq:asymptot-optm}) that a wavelet thresholding estimator has a maximum error $\epsilon_M (\sigma)$ which is asymptotically minimax with
\[
\epsilon_m (\sigma) \sim \sigma^{2\alpha/(\alpha + 1)} .
\]
In that sense a wavelet thresholding estimator is nearly optimal to suppress noise for all levels of uniform regularity
$\alpha \leq m$ and it adapts to the regularity exponent $\alpha$ which may be unknown.
Such results can be extended to ball of any Besov space \cite{donoho1996density} which gives other conditions to specify the regularity of $f$.

\begin{figure}
    \centering
    \subfigure[]{\label{fig:original_image}\includegraphics[width=0.24\textwidth]{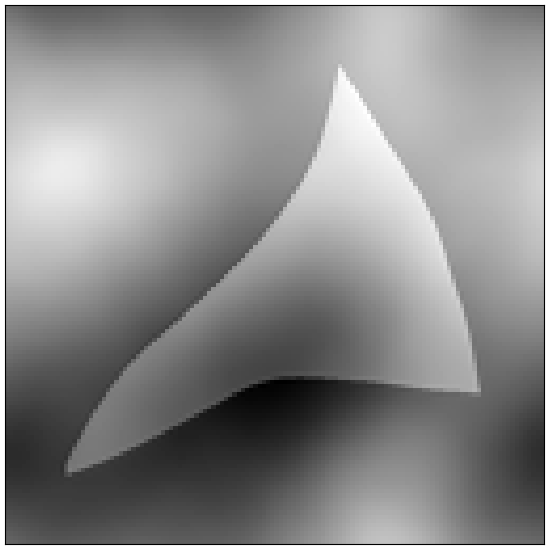}}
    \hspace{1cm}
    \subfigure[]{\label{fig:original_wav_decomp}\includegraphics[width=0.24\textwidth]{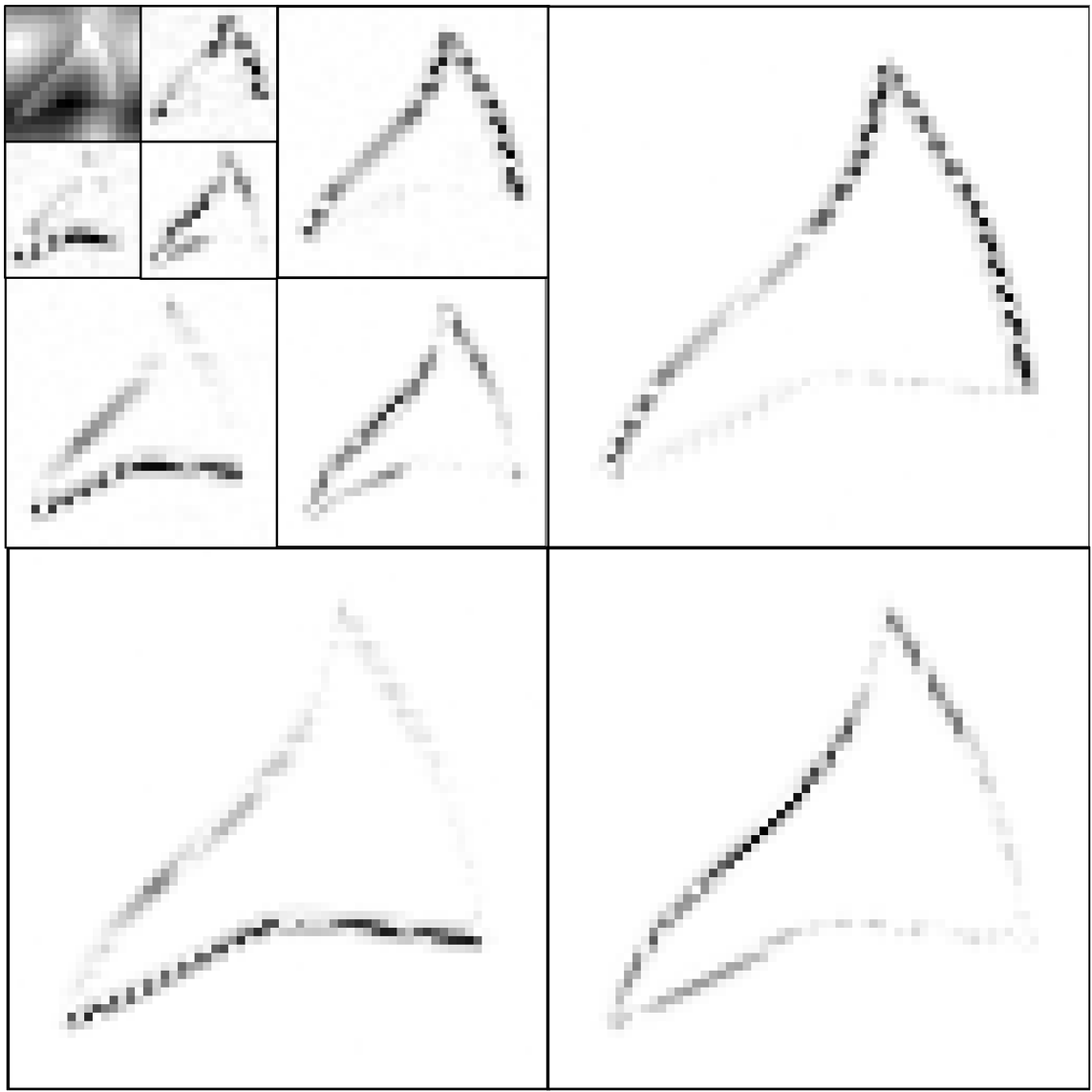}}
    \caption{(a) A $\mathrm C^2$-geometrically regular image discretised in dimension $d=128^2$, (b) 3 levels discrete orthogonal wavelet decomposition of this image. Large wavelet coefficients could clearly be seen along edges. Symlets with four vanishing moments are used \cite{Daubechies1992TenLO}.}
\end{figure}

\begin{figure}
    \centering
    \subfigure[]{\label{fig:noisy_image}\includegraphics[width=0.24\textwidth]{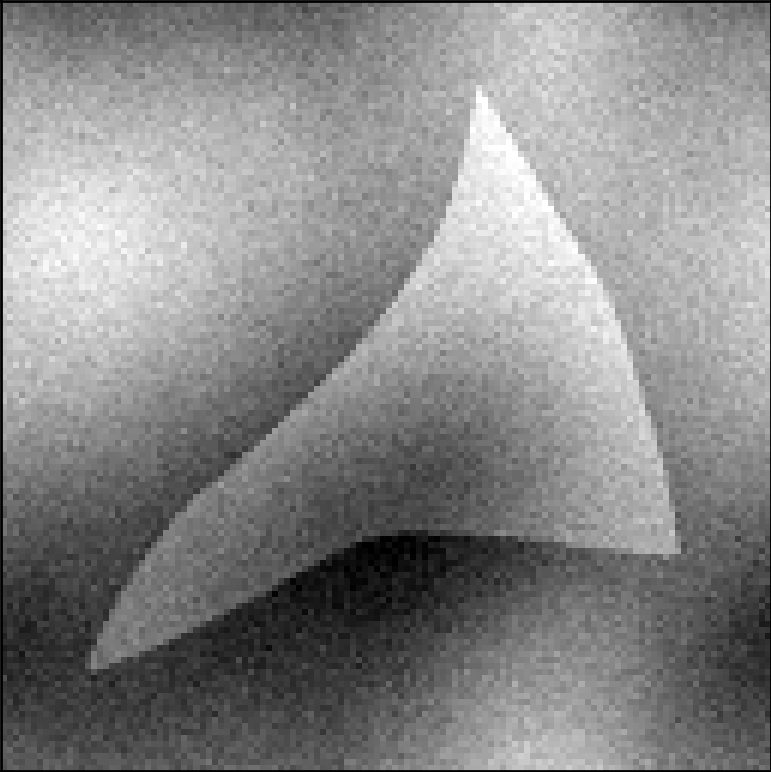}}
    \hspace{1cm}
    \subfigure[]{\label{fig:noisy_wav_decomp}\includegraphics[width=0.24\textwidth]{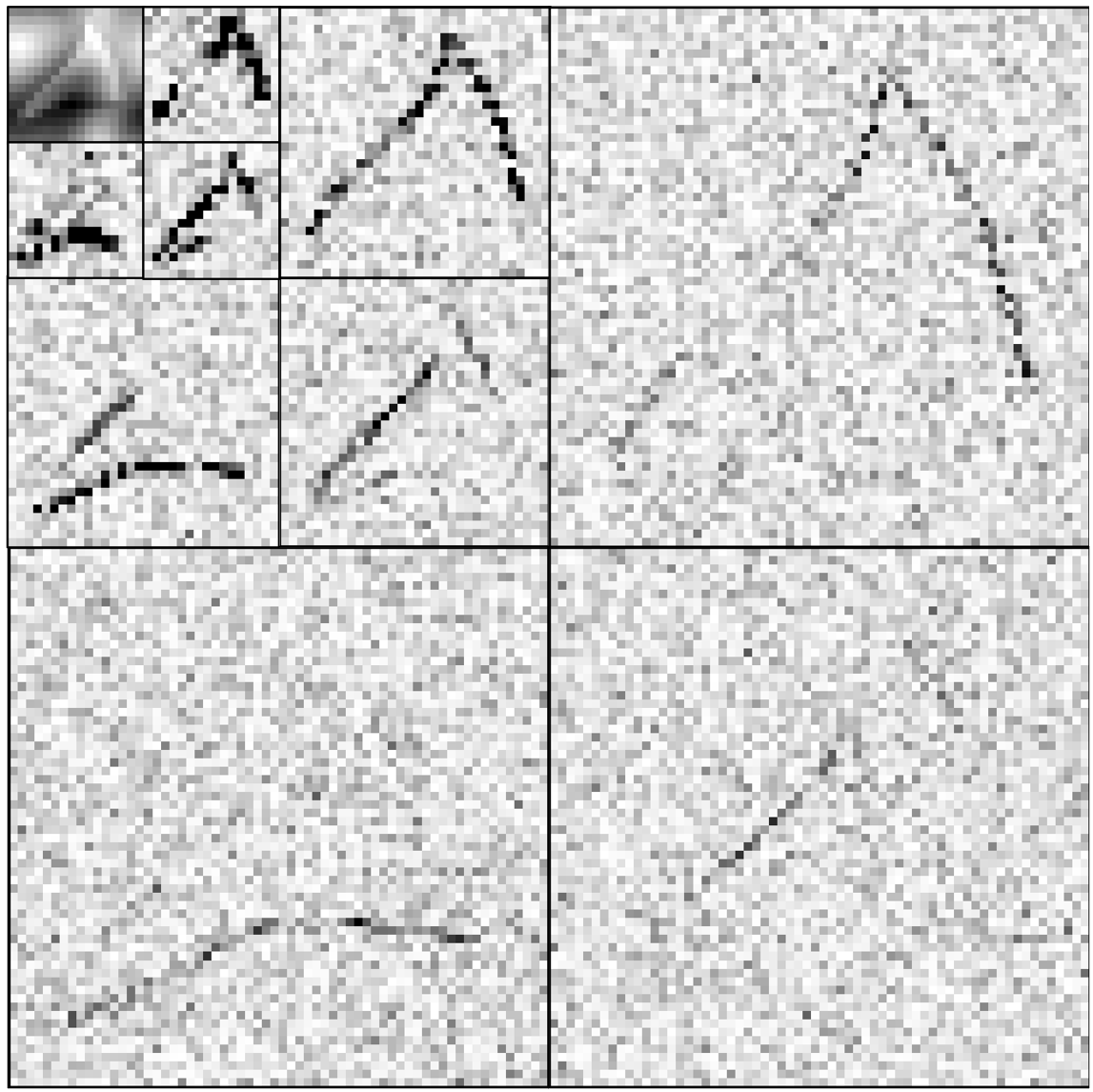}}\\
    \subfigure[]{\label{fig:denoised_wav_decomp}\includegraphics[width=0.24\textwidth]{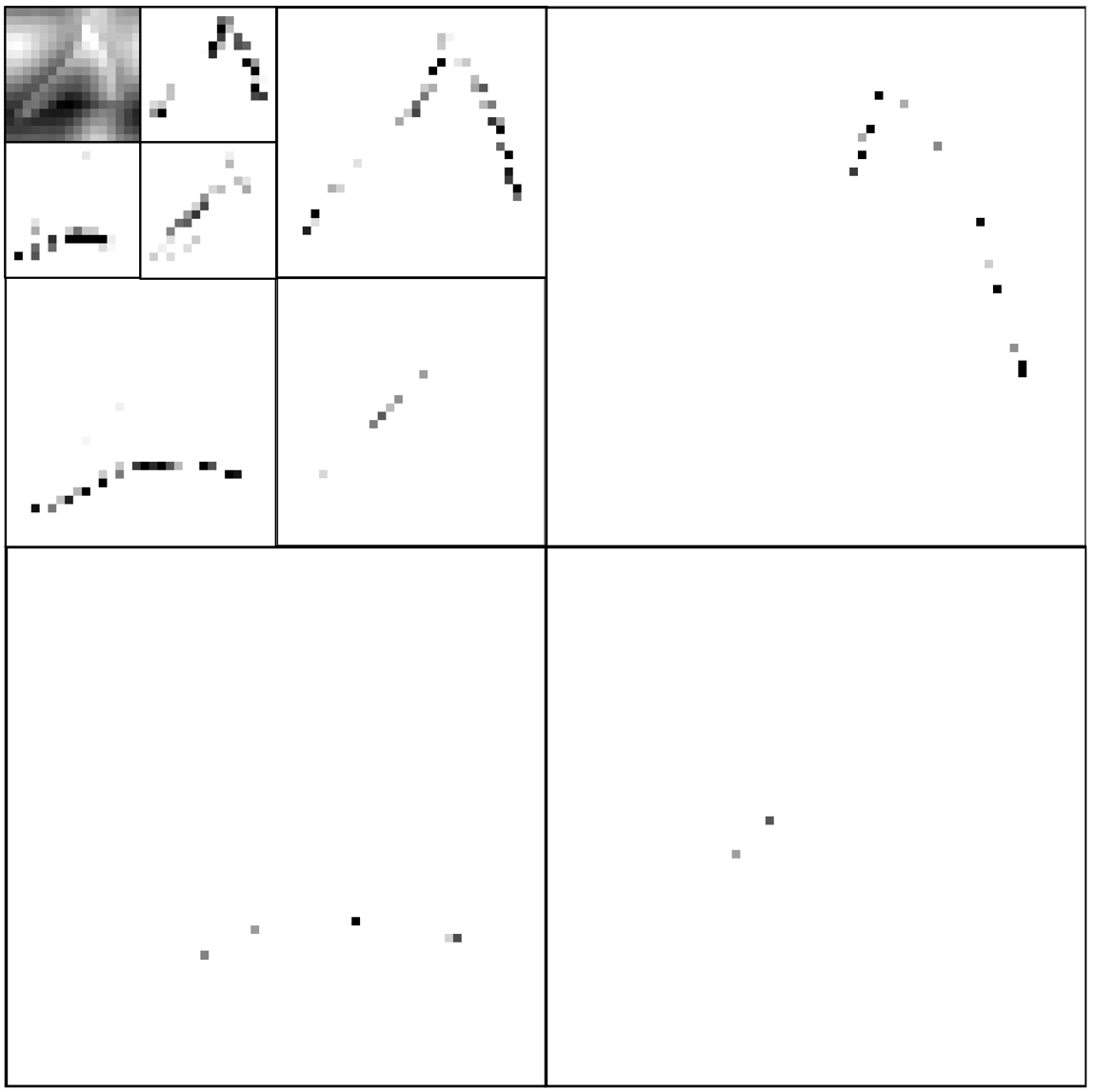}} 
    \hspace{1cm}
     \subfigure[]{\label{fig:denoised_threshold_image}\includegraphics[width=0.24\textwidth]{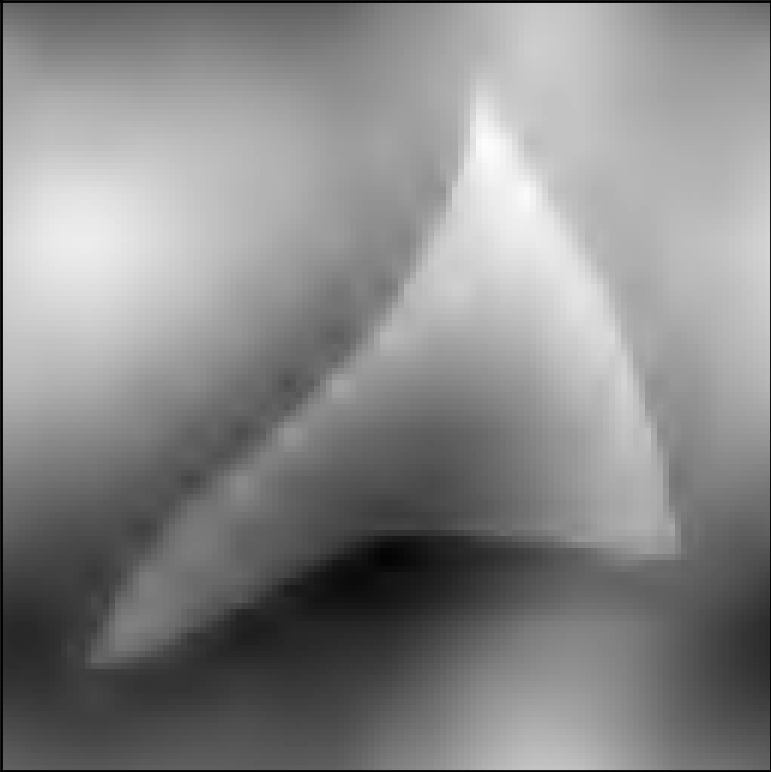}}

    \caption{(a): the noisy image is the sum of a $\bf C^2$ geometrically regular image ranging in $[-1,1]$ with an additive Gaussian white noise of variance $\sigma^2 = 0.005$. (b): orthogonal wavelet coefficients of the noisy image computed on $3$ scales $2^j$. Larger coefficients are darker. (c): non-zero wavelet coefficients after soft thresholding are shown in black. (d): denoised estimated image recovered from thresholded wavelet coefficients. The PSNR is 28.97dB for the noisy image and 33.29dB for its denoised counterpart.}
    \label{fig:foobar}
\end{figure}

\paragraph{Wavelet sub-optimality over $\C^\alpha$ geometric models}
Along an edge curve, image values have regular variations when moving parallel to the edge curve but they may be discontinuous when crossing the edge boundary. The image regularity is therefore anisotropic. We consider models of $\C^\alpha$ geometrically regular images, whose definition was introduced in \citet{Peyr2008OrthogonalBB}, inspired by \citet{Donoho1999WedgeletsNM, Korostelev1993MinimaxTO}.

\begin{definition}
\label{definition:C_alpha_gr}
(geometrically $\mathbf C^{\alpha}$)
    A function $f \in \mathrm \Ld[0,1]^2$ is said to be geometrically ${\bf C}^{\alpha}$ with Lipschitz constant $L$ if $f$ is uniformly Lipschitz $\alpha$ with constant $L$ over $[0,1^2]- \{\gamma_i \}_{1 \le i \le n}$, where the $\gamma_i$ are uniformly Lipschitz $\alpha$ curves with constant $L$, and these curves do not intersect tangentially.
\end{definition}

The set $\Lambda \subset \Ld[0,1]^2$ of all functions $f$ which are geometrically $\bf C^\alpha$ with constant $L$ defines a model of images which incorporate edges along regular curves, which may intersect at junctions.
Wavelet coefficients have a large amplitude along edges,
as illustrated in Figure \ref{fig:original_wav_decomp}. A wavelet thresholding algorithm
keeps the noise added to these large amplitude wavelet coefficients, which dominate the overall error of a wavelet thresholding estimator. One can verify that
the resulting maximum error over a set of $\Ca$ geometrically functions satisfies
\begin{equation}
\label{eq:waveletorth-estim}
\sup_{f \in \Lambda} {\mathbb E}(\|\hat f - f \|^2) \sim \sigma\, |\log \sigma| .
\end{equation}
It does not depend
upon the Lipschitz exponent $\alpha$ and
only takes advantage of the fact that edges have a finite length, which limits the
number of large wavelet coefficients.
A similar error is obtained over a much larger set of bounded variation functions
\cite{Mallat1989ATF}. 
Figures \ref{fig:noisy_image} to \ref{fig:denoised_threshold_image} gives an example of denoising by thresholding the wavelet orthogonal coefficients of a geometrically regular image. As expected, the thresholding retains wavelet coefficients in the neighbourhood of edges and hence retains noise in this neighbourhood.
The soft thresholding attenuates the amplitude of the noise at the cost of restoring edges which are smoothed out. Figure \ref{fig:wavelet_trans_inv_slopes} shows the mean-squared error $\log \epsilon_{ms} (\sigma)$ as a function of $\log \sigma$ for these $\Ca$ geometrically regular images. The slope does not depend upon $\alpha$ and remains equal to $1$. The implementation was done with a nearly translation invariant thresholding estimator obtained by averaging translated orthogonal wavelet estimators \cite{Coifman1995TranslationInvariantD} (see \ref{appendix:denoisers_and_optim} for details). This averaging reduces the error by a multiplicative factor but does not modify the asymptotic slope.
Numerical errors are computed on average over a random process whose realizations are $\Ca$ geometrically regular functions, whose edge length concentrates as explained in \ref{appendix:discretising_geometry}. Estimation errors thus remain of the same order of magnitude as the maximum error for all typical realizations.

\begin{figure}
    \centering
    \subfigure[]{\label{fig:wavelet_trans_inv_slopes}\includegraphics[width=73mm]{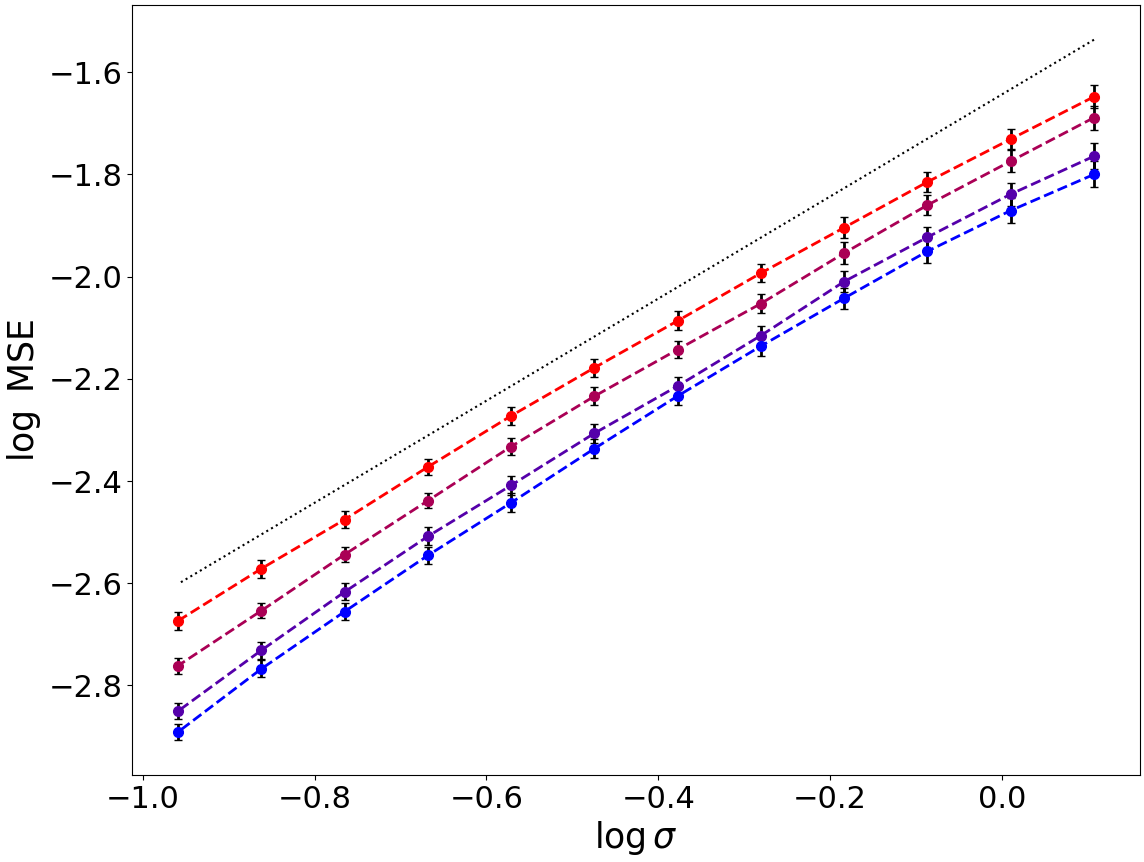}}
    \hspace{1cm}
    \subfigure[]{\label{fig:UNet_slopes}\includegraphics[width=73mm]{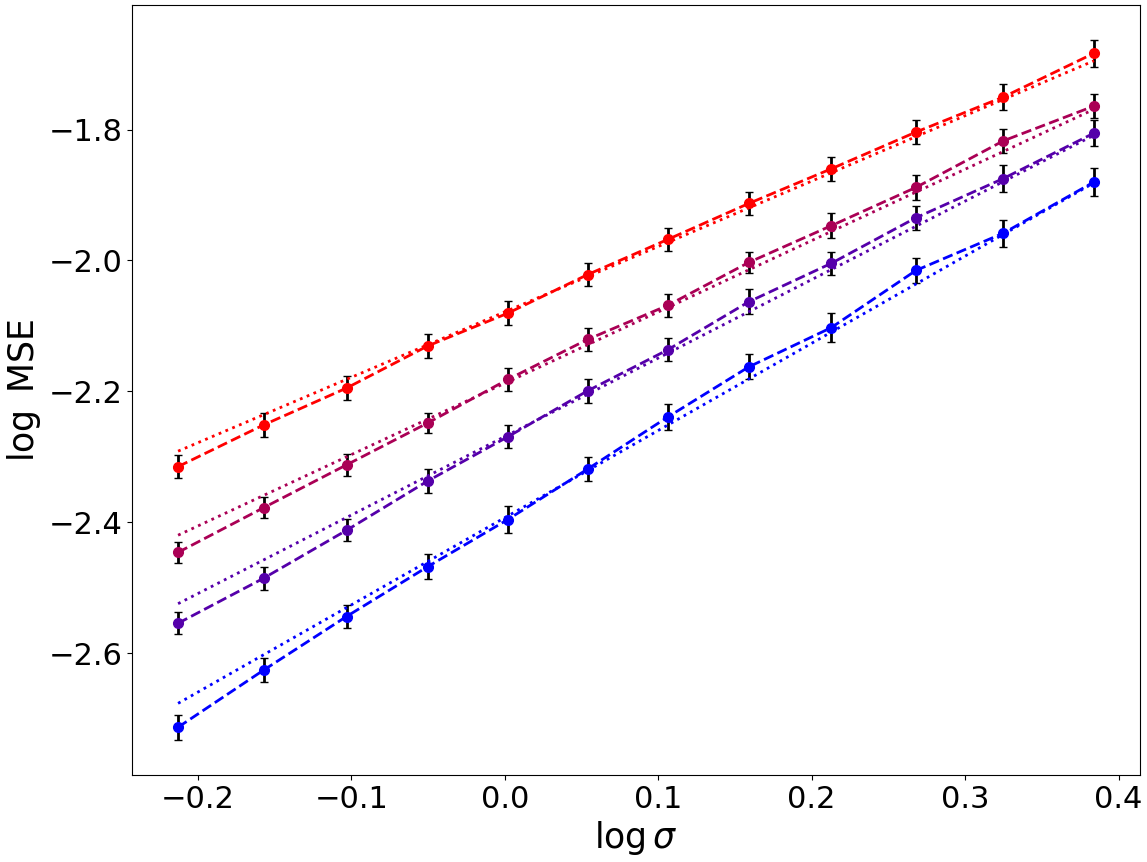}}
    \hspace{1cm}
    \caption{(a): Mean-squared error $\epsilon_{ms} (\sigma)$ of a wavelet thresholding estimator, as a function of the noise variance $\sigma^2$, computed on average over $\Ca$ geometrically regular images defined in \ref{appendix:discretising_geometry}.
    Red to blue curves correspond to Lipschitz exponent $\alpha \in \{2, 1.5, 1.2, 1\}$. The slope of $\log \epsilon_{ms}$ remains equal to $1$ as a function of $\log \sigma$. (b): same error calculations for a convolutional neural network estimator computed with a UNet. For each Lipschitz exponent $\alpha$, $\log \epsilon_{ms}$ nearly superimposes with a straight line of slope $2 \alpha / (\alpha + 1)$ which corresponds to the minimax rate.}
    \end{figure}

\subsection{Optimal denoising of geometrically regular images}
\label{sec:Xlets}

To reduce the error of wavelet thresholding estimators we must take advantage of edges regularity, these edges being $\Ca$ curves.
\citet{Korostelev1993MinimaxTO} proved that the minimax
rate over $\Ca$ geometrically regular images remains
\begin{equation}
\label{minimax-calpha}
\epsilon_m \sim \sigma^{2\alpha/(\alpha + 1)} ,
\end{equation}
which is identical to the minimax rate for functions which are uniformly Lipschitz $\alpha$ without edges. In other words, the presence of edges does not affect the minimax
denoising rate if these edges are themselves $\Ca$ regular curves which intersect over a finite number of junctions. However, this result was not constructive.
An important research effort has been devoted to 
reach this minimax rate through sparse decompositions in appropriate bases or frames \cite{Cands1999CurveletsAS, Peyr2008OrthogonalBB, Do2002ContourletsAD, Labate2005SparseMR}. However, much better results have been obtained with deep convolutional neural networks. 
We briefly review these results to understand where are the difficulties.

Edges of $\Ca$ geometrically regular images have two important properties. They correspond to pointwise discontinuities and these discontinuities propagate along regular $\Ca$ curves. To avoid retaining noise in the neighbourhood of edges, an edge is covered with few elongated vectors with a support which is as narrow as possible. The length of these vectors depends upon the edge regularity as illustrated in Figure \ref{fig:bandlet_filter}. This strategy reduces the number of vectors needed to cover the edge compared to the number of wavelets
illustrated in Figure \ref{fig:wavelet_filter}, but it requires to adapt the support
of the frame vector to the edge regularities. Curvelets, shearlets and bandlets adopt this strategy.

\paragraph{Curvelet frames}
Curvelet frames \cite{Cands1999CurveletsAS, NewCurvelets} have been constructed to obtain optimal denoising of images whose edges are closed $\C^2$ contours, without corners. Curvelets have an elongated rectangular support of width $2^j$ and length $2^{j/2}$. A curvelet frame is constructed by rotating this support and translating it on a rotated anisotropic grid whose intervals is proportional to $2^j$ and $2^{j/2}$.
Whereas $\C^2$ edge curves create of the order of $2^{-j}$ large wavelet coefficients at each scale, they only produce of the order of $2^{-j/2}$ large curvelet coefficients. A curvelet representation of such edge curves is therefore much sparser than a wavelet basis representation.
A curvelet denoising estimator is computed by thresholding the curvelet coefficients 
of a noisy image \cite{Cands2006FastDC}. The resulting estimation error over $\C^2$ edge curves reaches the minimax bound (\ref{minimax-calpha}) for $\alpha = 2$ if there is no corner. This important theoretical result has limited applications because curvelet become suboptimal when $\alpha \neq 2$ and when the edges join at corners.
Shearlets \cite{guo2006sparse, Labate2005SparseMR} provide efficient implementations of curvelets for discretised images, by replacing rotation operators by shearing operators along the horizontal and vertical directions.

\begin{figure}
    \centering
    \subfigure[]{\label{fig:wavelet_filter}\includegraphics[width=60mm]{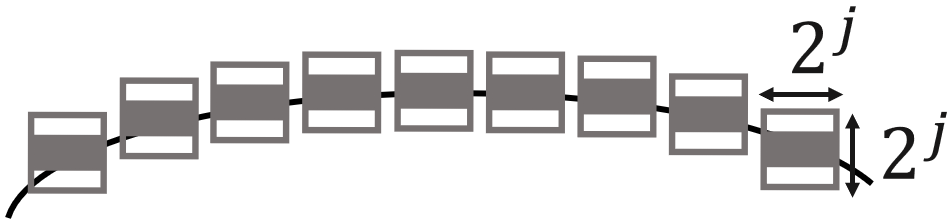}}
    \subfigure[]{\label{fig:bandlet_filter}\includegraphics[width=60mm]{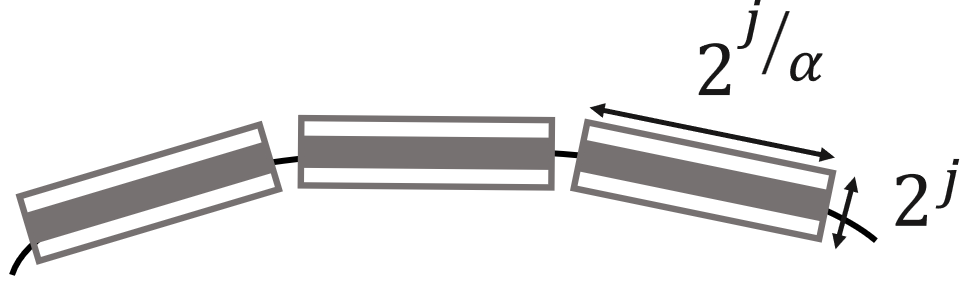}}
    \caption{(a) An edge creates large wavelet coefficients whenever the square wavelet supports intersect the edge and the wavelets oscillate across the edge. The support width is proportional to the scale $2^j$ so the total number of large wavelet coefficients along an edge is proportional to $2^{-j}$. (b) Curvelets, shearlets and bandlets have elongated supports whose width is proportional to $2^j$ but whose length is proportional to $2^{j/\alpha}$, with $\alpha = 2$ for curvelets and shearlets. If the edge is a $\C^{\alpha}$ curve then the number of non-negligible coefficients is proportional to $2^{-j/\alpha}.$}
\end{figure}

\paragraph{Adaptivity to geometric regularity with bandlets}
To build optimally sparse representations of images with edges, the support of basis vectors must be adapted to the regularity of edge curves.
If the edges are $\Ca$ curves then, at a scale $2^j$, a basis vector of width $2^j$ must have a length of the order of $2^{j/\alpha}$, as shown in Figure \ref{fig:bandlet_filter}. It means that different basis vectors must be used when the
geometric regularity changes. This problem was solved by bandlet dictionaries, which include an infinite number of orthonormal bases. These bases are obtained by transforming orthogonal wavelet coefficients with adapted unitary operators \cite{Peyr2008OrthogonalBB}.
A best basis, regrouping vectors having the appropriate geometry is selected with a best basis selection algorithm
\cite{bandletLePennec, Peyr2008OrthogonalBB}. A bandlet denoising estimator selects
such a basis from the noisy signal and thresholds the image coefficients in the selected best basis \cite{Dossal2008BandletIE}. It has a maximum error over $\Ca$ geometrically regular images which reaches the minimax bound (\ref{minimax-calpha})
for all $\alpha$, despite the presence of corners.
The difficulty of this approach is that the size of the dictionary grows quickly with the diversity of geometric image properties. 

\paragraph{Denoising with deep neural networks}
Despite interesting asymptotic results, the effort to denoise image with geometrically adaptive bases has not been so effective on natural images, which have a complex geometry. Non-linear estimations with deep neural networks provide much better results and seem to follow very different principles. These neural networks involve cascades of filters with a fixed short support, typically of size 3x3, and pointwise non-linearities such as $ReLU$. They do not have filters with elongated supports, adapted to the image geometry. Still, they produce much smaller denoising error.
Deep neural networks are trained to compute denoising estimators $\hat f(g)$ by minimising a mean-squared error 
$\EE_{P_\V f,g} (\| \hat f(g) - P_\V f\|^2)$ where the expected value is computed over the noise distribution and over a set of training data. Impressive denoising results are
obtained over images which do not belong to the training set, which considerably improve previous state of the art results \cite{elad2023image, thakur2021image}. 
Figure \ref{fig:MMS-denoising} shows denoising examples trained over a dataset of natural images (the LSUN bedroom dataset \cite{yu15lsun}) and over a dataset of piecewise regular images. The deep neural network has a UNet architecture \cite{Ronneberger2015UNetCN}. In both cases, these estimators retrieve the image geometry, even at very low signal to noise ratios. The network estimation restores sharp contours as soon as long enough portions of these contours emerge from the noise. Objects delimited by edges which are not long enough or not sufficiently contrasted (given the noise level) are not restored. This is the case for the painting in the top right corner of the LSUN bedroom example.
Figure \ref{fig:UNet_slopes} shows the mean-squared denoising error $\epsilon_{ms} (\sigma)$ of this UNet over a set of $\Ca$ geometrically regular images defined in \ref{appendix:discretising_geometry}. This mean-squared error reaches the optimal minimax rate as a function of $\sigma$, for all $1 \leq \alpha \leq 2$. This had already been observed in
\cite{kadkhodaie2024generalization}. However, there is currently no mathematical understanding of the performance of deep neural network estimators.

\begin{figure}
    
    \centering
    
    \includegraphics[width=0.2\textwidth]{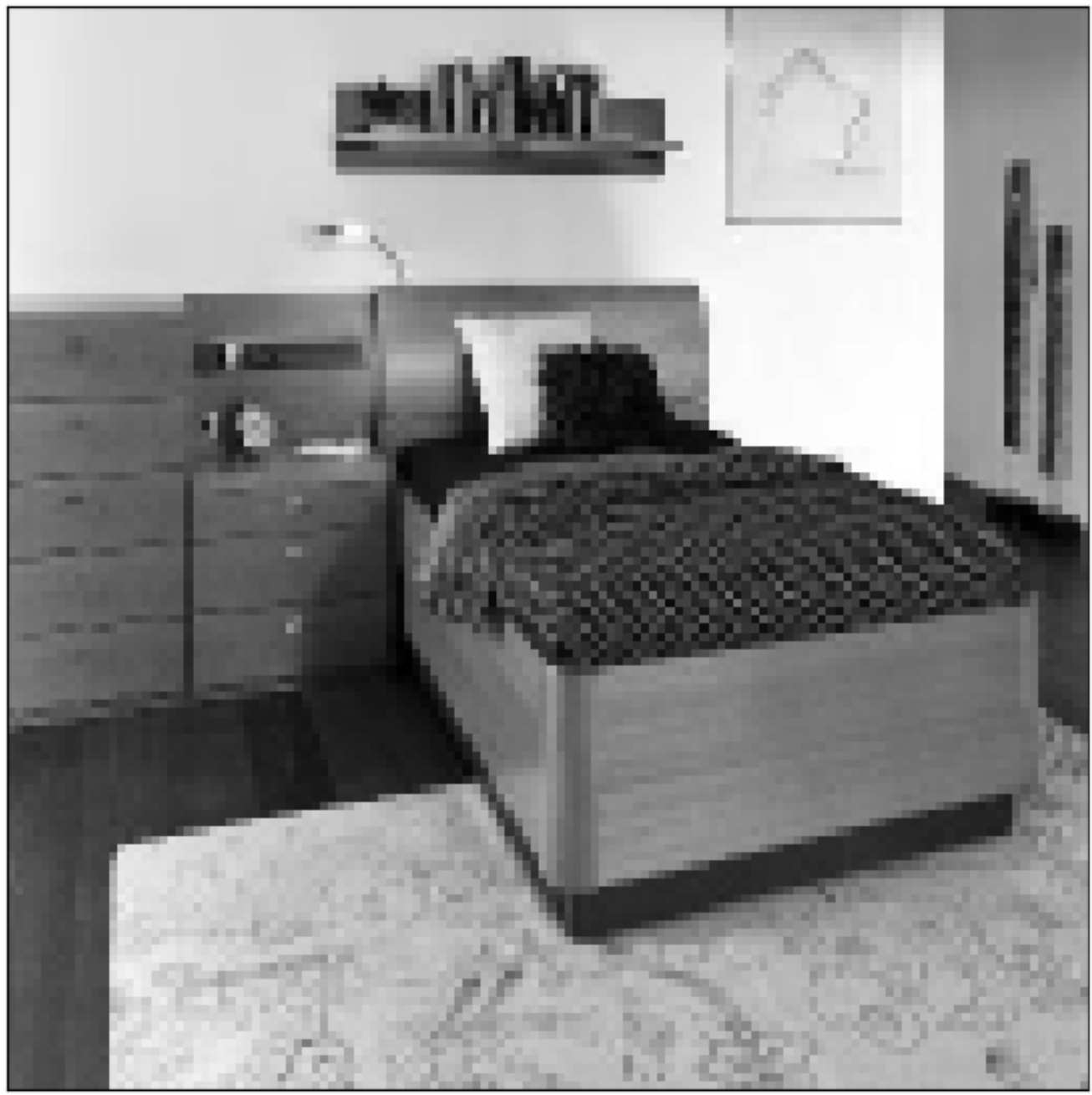}
    \includegraphics[width=0.2\textwidth]{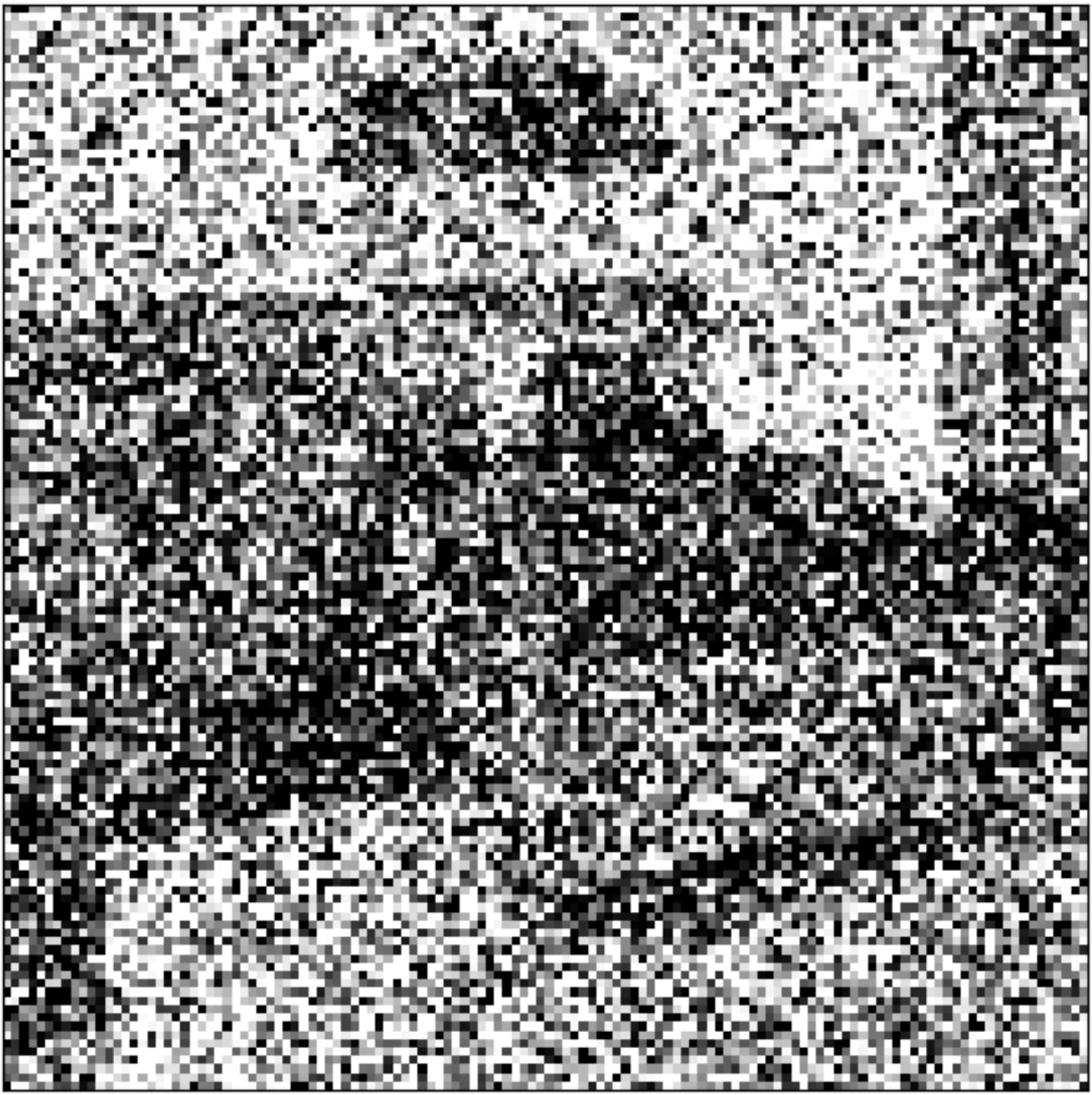}
    \includegraphics[width=0.2\textwidth]{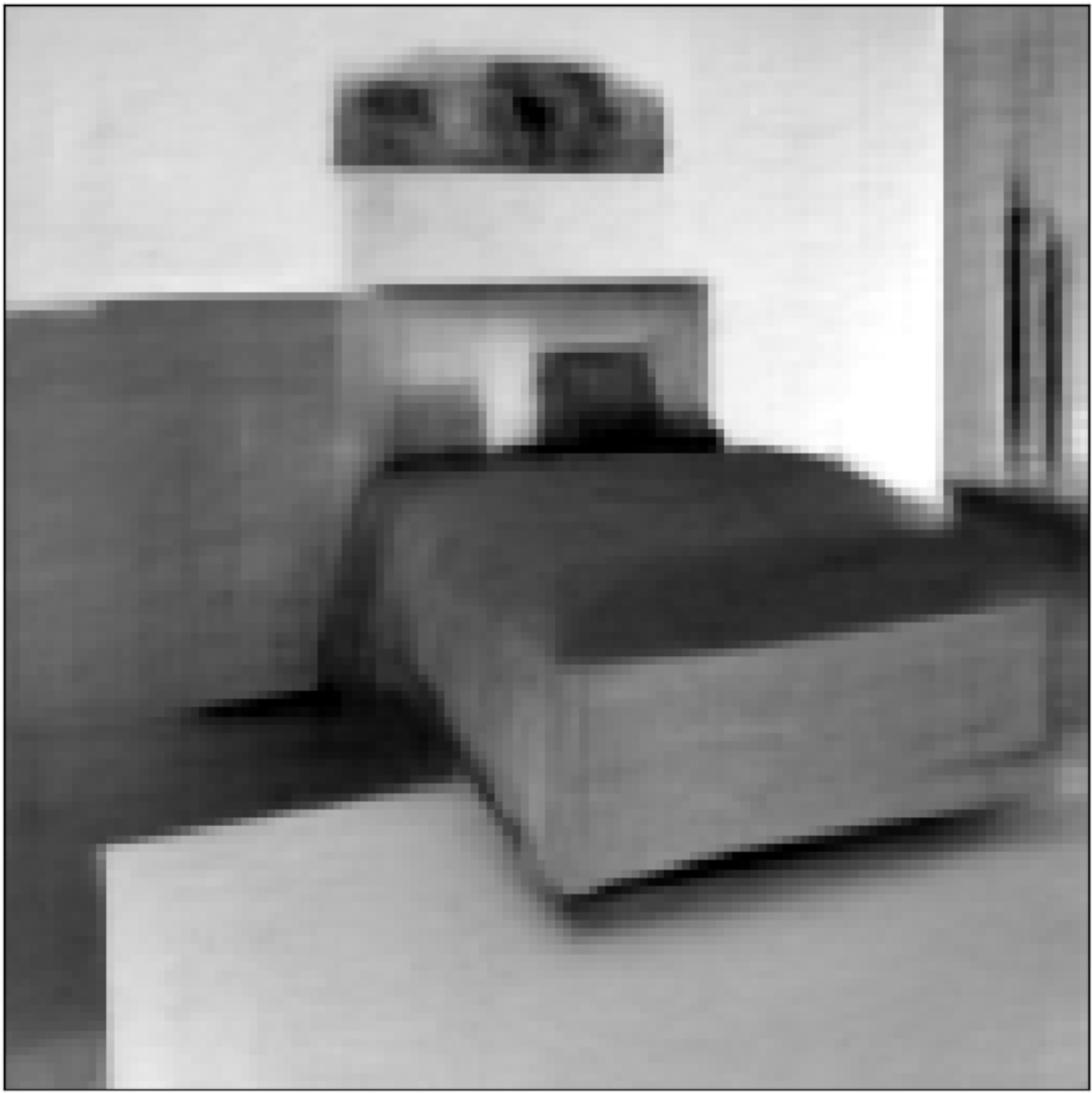}
    \includegraphics[width=0.2\textwidth]{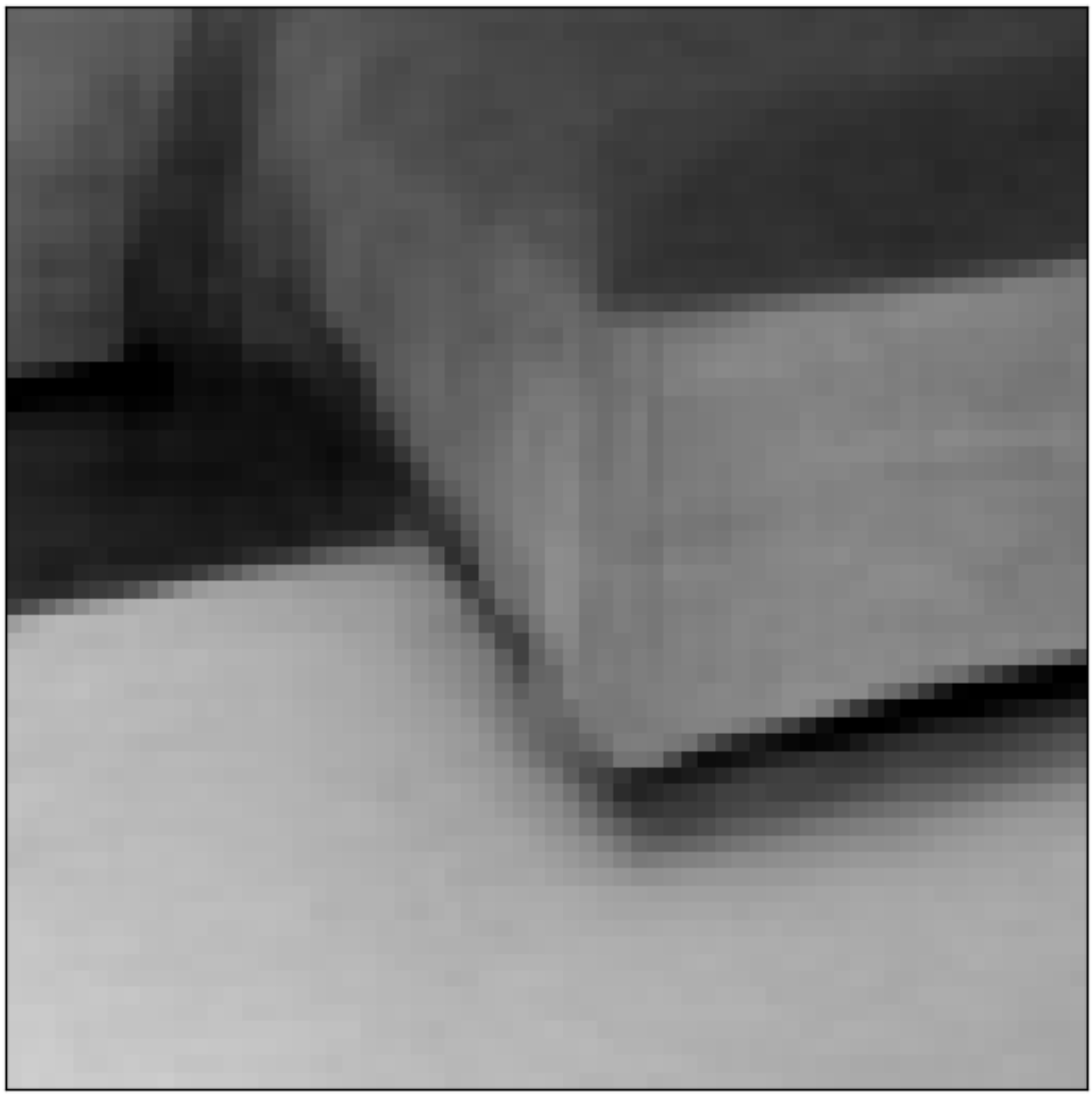}\\
    \includegraphics[width=0.2\textwidth]{Figure_1.png}
    \includegraphics[width=0.2\textwidth]{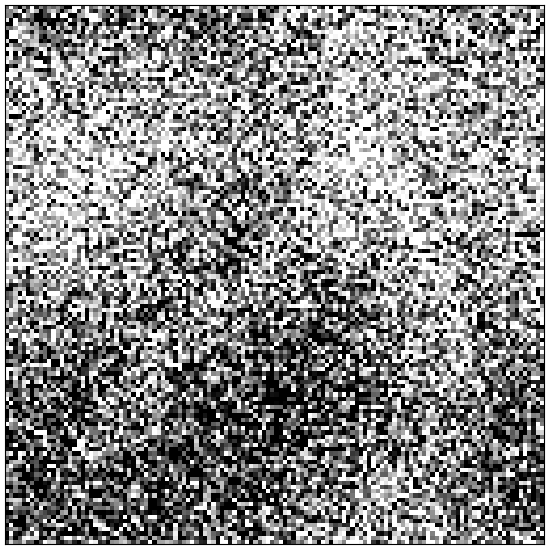}
    \includegraphics[width=0.2\textwidth]{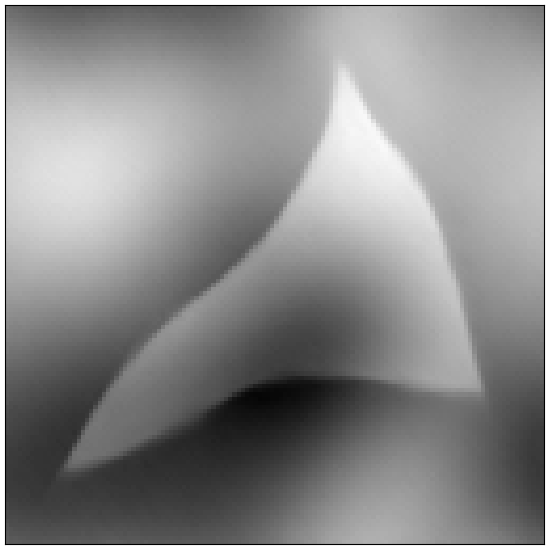}
    \includegraphics[width=0.2\textwidth]{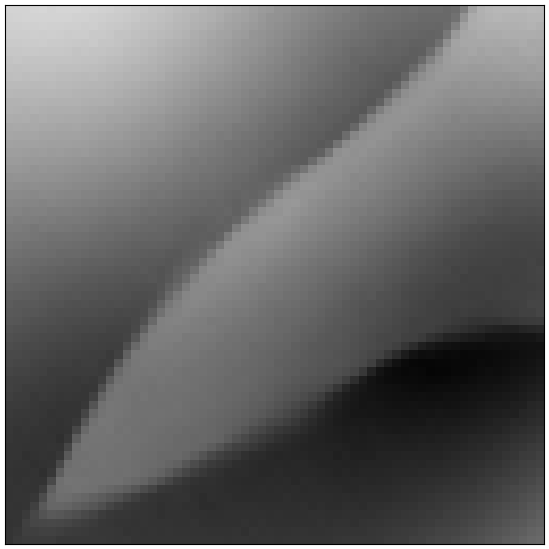}
    
    \caption{This figure shows the denoising performance of a UNet trained on two datasets. For the top row image, the UNet was trained on the LSUN bedrooms dataset \cite{yu15lsun} in dimension $d=128^2$ pixels.
    On the second row it was trained on $\Ca$ geometrically regular images of same size. The image amplitudes are normalized in $[-1,1]$. The first column shows an original test image which does not belong to the training set. The second column shows a noisy image with a noise variance $\sigma^2=1$, which corresponds to a PSNR of $6$dB. The third column displays the denoised image computed by the UNet. 
    The PSNR is $22.03$ dB for the bedroom image $30.03$dB for the $\Ca$ geometrically regular image.
    The fourth column shows a zoom near edges. In both case the estimator reproduces very well the image edges and corners.
    }
    
    \label{fig:MMS-denoising}
\end{figure}

\section{Denoising with a scattering transform}
\label{chapter:scattering_geometry}

The previous section shows that curvelets, shearlets and bandlets can define optimal
estimators of $\Ca$ geometrically regular images by taking advantage of their directional regularity to produce a sparse representation. 
Numerical results also show that deep neural networks can define better estimators
having asymptotically optimal denoising performances over $\C^\alpha$ geometrically regular images. As opposed to curvelets, Xlets or bandlets, they do not involve linear filters having elongated support in multiple directions but they are computed with a cascade of linear filters and pointwise non-linearities.
A scattering transform has been introduced in \cite{mallat2012group, ScatteringBruna} as a simplified model of a two layers convolutional neural network, with complex wavelet filters and a modulus non-linearity. Section \ref{sec:numerics} shows numerically that an estimator regularized by scattering $\ell^1$ norms can achieve an asymptotically minimax
denoising over $\C^\alpha$ geometrically regular images, which adapts to $\alpha$. 
We state a mathematical conjecture which formalises these numerical observations,
opening a mathematical bridge with deep neural network estimators.
Although we do not prove the conjecture, we give intermediate mathematical results which partly explains these observations.
Section \ref{sec:dyadic_wavelet} begins by a brief review of complex directional wavelet transforms. 
Section \ref{sec:geometric_regularity_scattering} proves that the image regularity is partly captured by a sparse set of scattering coefficients.
However, scattering coefficients do not involve elongated filters such as curvelets or bandlets, whose support are adapted to the edge regularity. We show that this is replaced by another subset of scattering coefficients, which capture the edge singularity with large $\ell^1$ norms.

\subsection{Directional complex wavelet transform}
\label{sec:dyadic_wavelet}

Wavelets can capture anisotropic regularities if they have directional 
vanishing moments aligned with the direction where the image is regular. 
Such wavelets can be created by rotating a single mother wavelet \cite{Mallat1989ATF}.
In the following we review the properties of a dyadic directional wavelet transform and 
prove that the regularity of a contour can be captured by a subset of wavelets coefficients.
We write $\|f \|_1 = \int |f(u)| du$ the $\Lu$ norm of a function. For simplicity, we shall sometime loosely address these $\Lu$ norms as $\ell^1$ norms because they are 
numerically computed as $\ell^1$ norms of functions discretised on a sampling grid.

\paragraph{Dyadic directional wavelet transform} We define $4$ wavelets
$\psi^k(u) = \psi(r_k u)$ for $u \in \R^2$ which are rotations of a single complex wavelet $\psi$,
where $r_k$ is a rotation in $\R^2$ by an angle $-k \pi/4$ for $0 \le k < 4$. 
A dyadic wavelet transform of $f$ is computed as convolutions with dilated wavelets 
$\psi^k_j (u) = 2^{-2j} \psi^k(2^{-j} u)$ for $2^j \leq 2^\JJ$, without subsampling.
The multiplicative factor $2^{-2j}$ normalises the $\Lu$ norm of all wavelets
\[
\|\psi^k_j \|_1 = \int |\psi^k_j (u)|\, du =  \int |\psi (u)|\, du = 1.
\]
At the largest scale $2^\JJ$,
the lower frequencies are retained by a convolution with a scaling function $\phi_\JJ (u) = 2^{-2\JJ} \phi(2^{-\JJ} u)$, where $\phi \geq 0$
is positive and $\int \phi(u)\, du = 1$:
\begin{equation}
\label{undec-wave-cof2}
W f = \Big(f * \phi_\JJ \, , \, f *  \psi_{j}^k \Big)_{0 \le k < 4\,,\, j \leq j_M }.
\end{equation} 
Such a dyadic wavelet transform is invertible and has a stable inverse \cite{Mallat1989ATF} if the wavelet satisfies
Littlewood-Paley inequalities: 
\begin{equation}
\label{eq:Littlewood-Paley}
\exists \delta < 1 ~,~\forall \omega \neq 0~,~ 1 - \delta \le |\hat \phi_\JJ(\omega)|^2+ 
\sum_{j=-\infty}^{\JJ} \sum_{k=0}^{3} |\hat \psi_{j}^{k}(\omega)|^2 \le 1 + \delta.
\end{equation}
We define wavelet transforms in $[0,1]^2$ with a standard periodization of wavelets in $\R^2$ \cite{mallat1999wavelet}. In numerical computations, the image is known at a finite resolution $2^{j_m}$ over a discrete grid with an orthogonal projection on a space $\V$. Section \ref{sec:wavelet} explains that it is equivalent to limit all wavelet scales to $2^j \geq 2^{j_m}$. All convolutions are performed on this sampling grid, without subsampling. Figure \ref{fig:first_order_scattering} shows the real and imaginary parts of $f * \psi^k_j $ along $4$ orientations and $3$ dyadic scales, for $f$ being a geometrically regular function shown in Figure \ref{fig:original}.

\begin{figure}
\begin{minipage}[c]{0.24\textwidth}
  \centering
  \subfigure[]
    {\label{fig:original}\includegraphics[width=\linewidth]{Figure_1.png}}

\end{minipage}\hspace{2cm}
\begin{minipage}[c]{.6\linewidth}
  \centering

  \subfigure[]
    {\label{fig:first_order_scattering}\includegraphics[width=\linewidth]{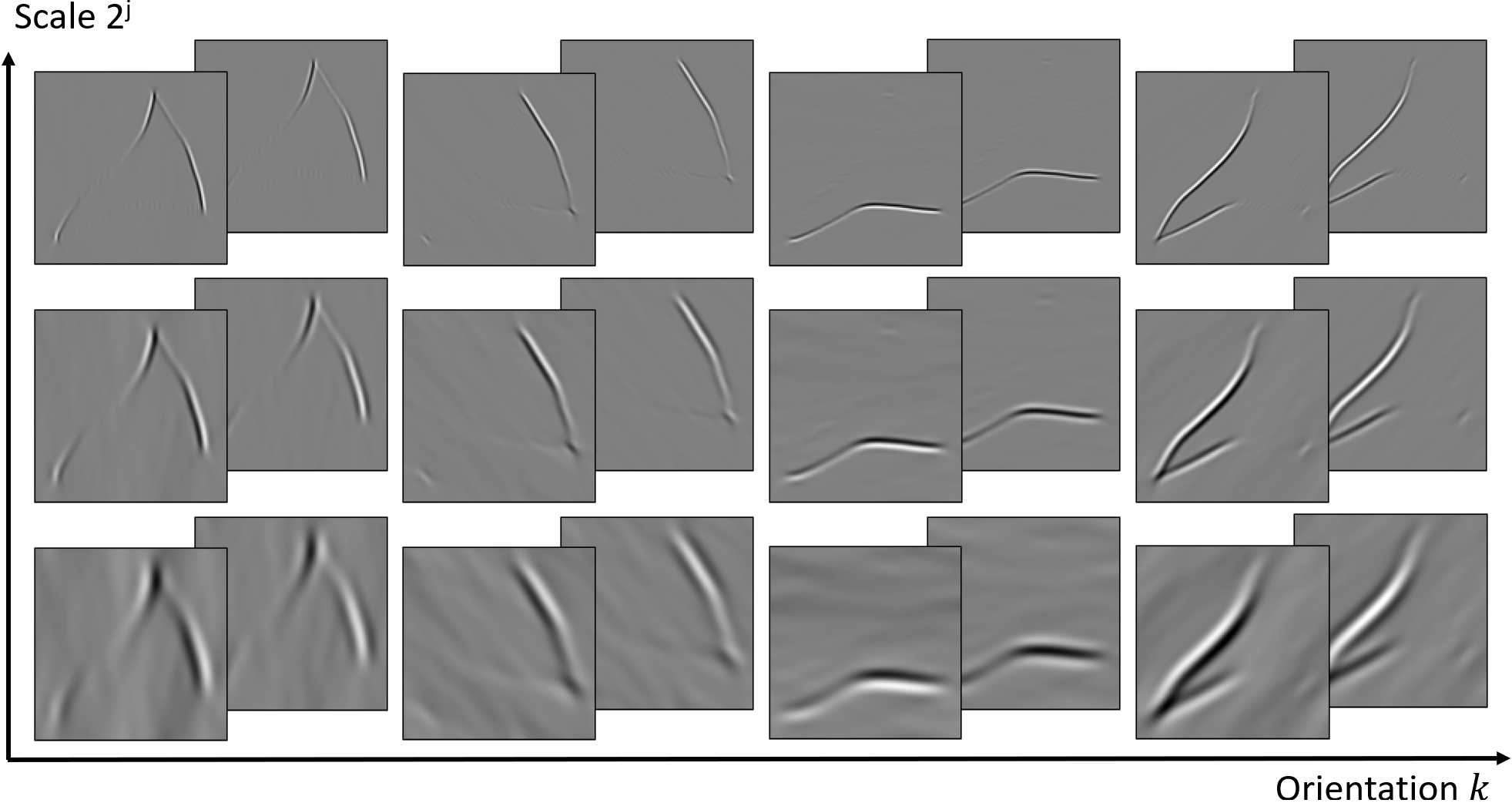}}

    \vspace{-2mm}

    \subfigure[]
  {\label{fig:first_order_scattering_modulus}\includegraphics[width=\linewidth]{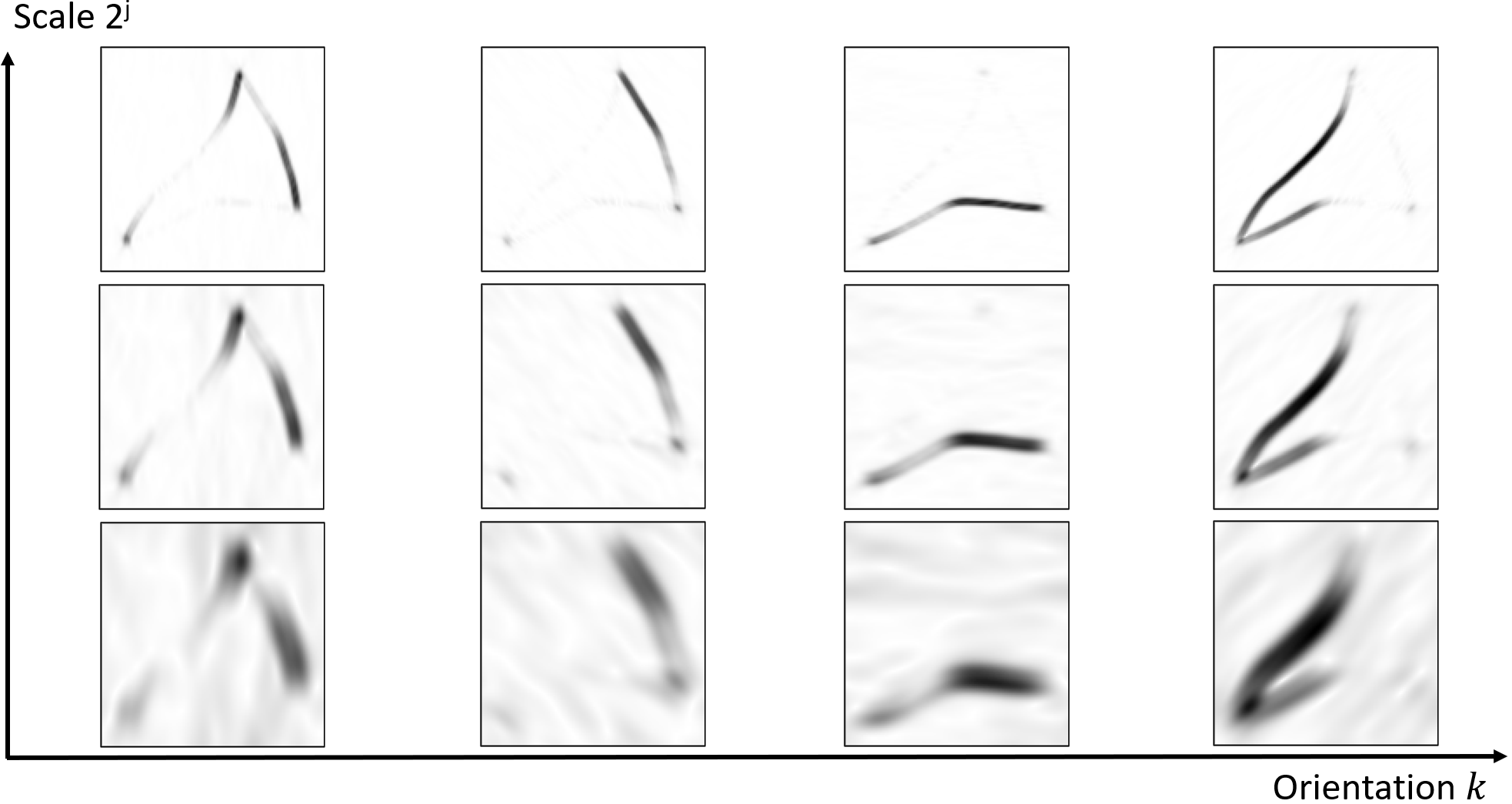}}
\end{minipage}
\caption{Dyadic directional wavelet transform of a $\mathbf C^2$ geometrically regular image, discretised in dimension $d=128^2$. (a): Original image. (b): Real and imaginary parts of wavelet coefficients $f * \psi^k_j$ of image (a), for scales $2^j$ and orientations $k$ which respectively vary vertically and horizontally. (c): Modulus 
$|f * \psi^k_j|$ of these wavelet coefficients. They are large when the vanishing moments of $\psi^k_j$ are not aligned with the tangent to the edge. Plots (b) and (c) are normalized such that images on a same row share the same minimum and maximum value.}
\end{figure}

To capture directional image regularities, we shall impose that $\psi$ has $m$ directional vanishing moments in directions $\theta \in [-\pi/4 , \pi / 4]$.
 It means that wavelets rotated in these directions are orthogonal to any one-dimensional polynomial function $u_1 \mapsto q(u_1)$ of degree $m-1$:
\[
\forall \theta \in [-\pi/4 , \pi / 4]~~,~~\forall u_2 \in [0,1]~,~ \int q(u_1) \psi(r_\theta(u_1,u_2)) du_1 = 0 .
 \]
\ref{app:wavelet-design} gives a sufficient condition on the Fourier transform of wavelets to obtain directional vanishing moments. 
Figure \ref{fig:wavelets_numerics} shows the real and imaginary parts of a wavelet having $m=2$ vanishing moments in directions $\theta \in [-\pi/4 , \pi / 4]$, along with its modulus and phase. This latter oscillates in the vanishing moments directions and remains nearly constant in the orthogonal direction.

\begin{figure}
    \centering
    \subfigure[]
    {\includegraphics[width=30mm]{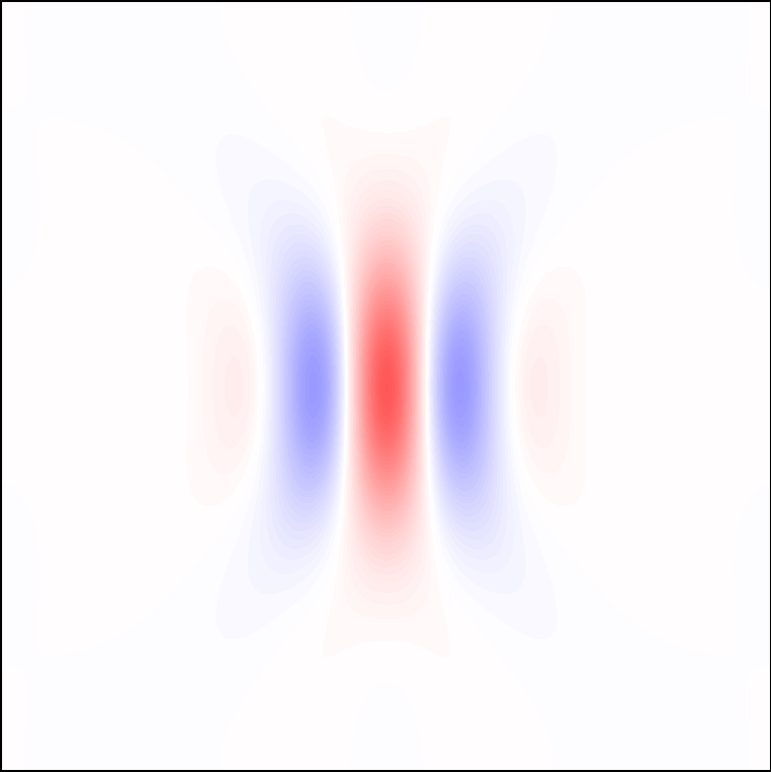}}
    \subfigure[]
    {\includegraphics[width=30mm]{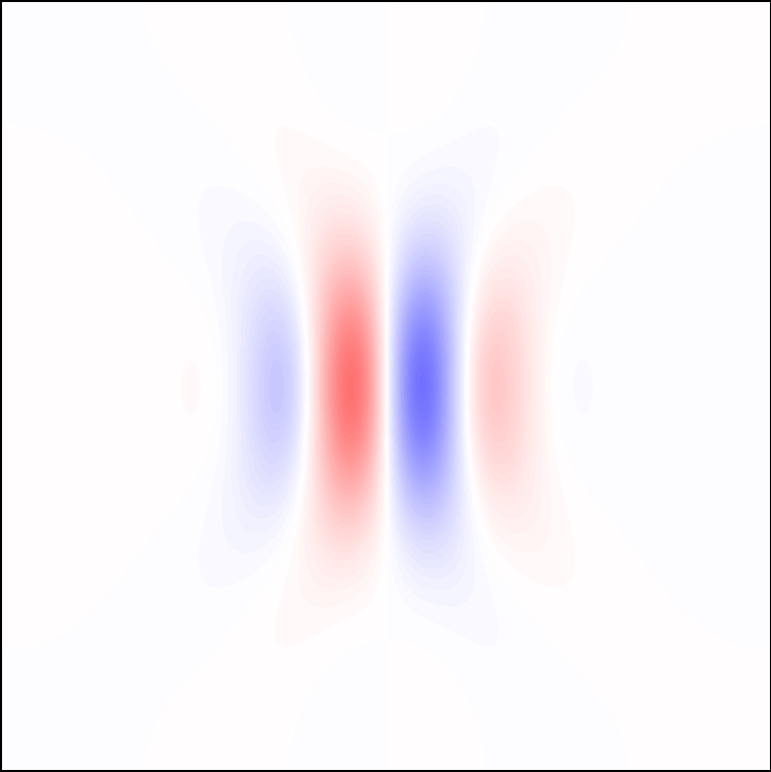}}\hspace{1cm}
    \subfigure[]
    {\includegraphics[width=30mm]{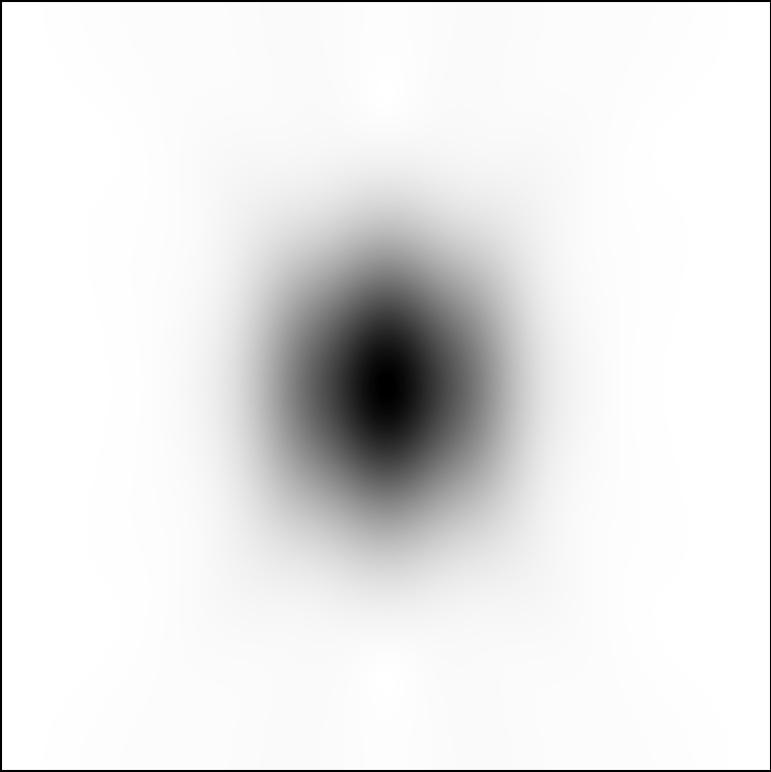}}
    \subfigure[]
    {\includegraphics[width=30mm]{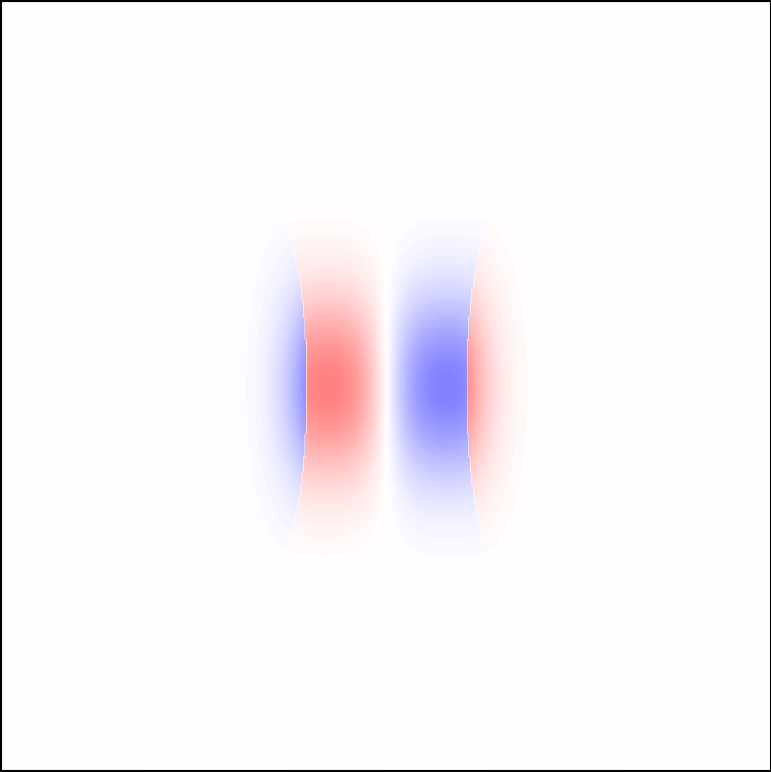}}
    \caption{Real (a) and imaginary parts (b) of a wavelet $\psi$ having two vanishing moments in directions $\theta \in [-\pi/4 , \pi / 4]$. Figures (c) and (d) show respectively $\psi$'s modulus and phase (where $|\psi|$ is non negligible). One should notice that lines of constant phase are nearly vertical.}
    \label{fig:wavelets_numerics}
\end{figure}

If $f$ is uniformly Lipschitz $\alpha \leq m$, since each wavelet $\psi^k_j$ has $m$ vanishing moments, one can verify as in (\ref{decayCalpha}) (modulo the difference of normalisation) that there exists $C > 0$ such that
\begin{equation}
\label{eq:isotropic-lipsch}
\forall u \in [0,1]^2, ~ |f * \psi^k_j (u) | \leq C\, 2^{ j \alpha } .
\end{equation}
For $\Ca$ geometrically regular images, Section \ref{sec:wavelet} explains that wavelet coefficients are large along the edge where $f$ is discontinuous. Standard wavelets thus fail to exploit the regularity of contours, yielding representations that are less sparse than theoretically possible and consequently suboptimal denoising performance. In what follows, we demonstrate that this geometric regularity can be effectively captured using wavelets with vanishing moments in the direction tangent to the edge. Since the image intensity varies regularly in this direction, the corresponding wavelet coefficient is small. This is formalized by the following theorem, which considers a periodic image having a single edge whose tangents have an angle in $[(k-1) \pi / 4 , (k+1) /4 ]$.

\begin{theorem}
\label{th:first_order}
Suppose that $\psi$ is a $\C^2$ function with two vanishing moments and two directional vanishing moments in the cone of angles $[- \pi/4, \pi/4]$, such that $|\psi|$ decays faster than any rational function. Let $\alpha \in [1,2]$. Let $f$ be a periodic image on $[0,1]^2$ which is uniformly Lipschitz $\alpha$ above and below an edge curve, which is itself Lipschitz $\alpha$ and whose tangents have an angle in $[(k-1) \pi / 4 , (k+1) \pi / 4 ]$. 
For all $s < \alpha$, there exists $C$ such that:
\begin{equation}
\label{eq:first_order0}
\forall j \le \JJ,~ \forall u \in [0,1]^2,~~| f * \psi^k_j (u)| \leq C\, 2^{(s-1) j} .
\end{equation}
\end{theorem}

This theorem proves that a wavelet having vanishing moments in the direction of an edge tangent can take advantage of the edge regularity to produce small amplitude coefficients. 
On the contrary, wavelets having no directional vanishing moments in these directions produce large amplitude wavelet coefficients. It
appears in Figure \ref{fig:first_order_scattering_modulus} which
shows that $|f * \psi_j^k|$ has a large amplitude along edges
where the tangents have an angle sufficiently different from $k \pi/4$.

\subsection{Scattering transform for geometric regularity}
\label{sec:geometric_regularity_scattering}

To define minimax denoising estimators for $\Ca$ geometrically regular functions, we need to precisely capture the regularity of these functions. It has three components. Outside edges, the functions are uniformly Lipschitz $\alpha$, which is captured by the fast decay of wavelet coefficients $|f * \psi^k_j|$. Theorem \ref{th:first_order} proves that anisotropic regularity of $f$ along edges is partly captured by the fast decay of $|f * \psi^k_j|$ when the wavelet directional vanishing moments are aligned with the edge tangent. However, the other wavelet coefficients are unable to characterize the sharp profile of the edge, which is a pointwise discontinuity when we move in the direction orthogonal to the edge curve. This issue is addressed by a scattering transform, which applies a second wavelet transform on $|f * \psi^k_j |$ to simultaneously capture the edge regularity and its sharp profile. A wavelet scattering transform applies a second wavelet transform to each $|f * \psi^k_j |$
\[
S f = \Big( f * \phi_{j_M} , W |f * \psi_{j,k}| \Big)_{j \leq j_M , k < 4} 
\]
with
\[
W|f * \psi^k_j | = \Big(|f * \psi^k_j|* \phi_\JJ \, , \, |f *  \psi_{j}^k| * \psi^{k'}_{j'} \Big)_{j' \le j_M, k' < 4} .
\]
Real and imaginary parts of such coefficients are shown in Figure \ref{fig:second_order_scattering} for a fixed $|f * \psi^k_j |$, shown in Figure \ref{fig:selected_first_order_modulus}, along $4$ different orientations $k'$ and 3 dyadic scales $j' > j$.
We compute the $\bf L^1$ norm of each of these second order wavelet coefficients. Since $\phi_\JJ \geq 0$ and $\int \phi_\JJ (u) du = 1$, it results that
$\| | f * \psi_j^k | * \phi_\JJ \|_1 = \|f * \psi_j^k \|_1 $. The scattering coefficients $\Lu$ norm can thus be written
\[
 \|W |f * \psi_j^k |\,\|_1  = \Big( \|f *  \psi_j^k \|_1 \,,\,  \||f * \psi_j^k| * \psi_{j'}^{k'}\|_1 \Big)_{j' \le j_M,k'<4} .
\]
Each $|f * \psi_j^k| * \psi_{j'}^{k'}$ computes the variations of $|f * \psi_j^k|$ 
in a direction $\pi k' / 4$, in neighbourhoods of size proportional to $2^{j'}$.
If $\psi$ is regular with $m$ derivatives then \citet{mallat2012group} prove that 
$\||f * \psi_j^k| * \psi_{j'} ^{k'}\|_1$ becomes negligible 
when $2^{j'} < 2^j$ because $|f * \psi_j^k|$ has a Fourier transform concentrated at lower frequencies than $\psi_{j'}^{k'}$. We thus only compute scattering coefficients for $j' > j$ as \citet{ScatteringBruna}. These scattering
coefficients are shown in Figure \ref{fig:second_order_modulus}.

\begin{figure}
\begin{minipage}[c]{0.24\textwidth}
  \centering
  \subfigure[]
    {\label{fig:selected_first_order_modulus}\includegraphics[width=\linewidth]{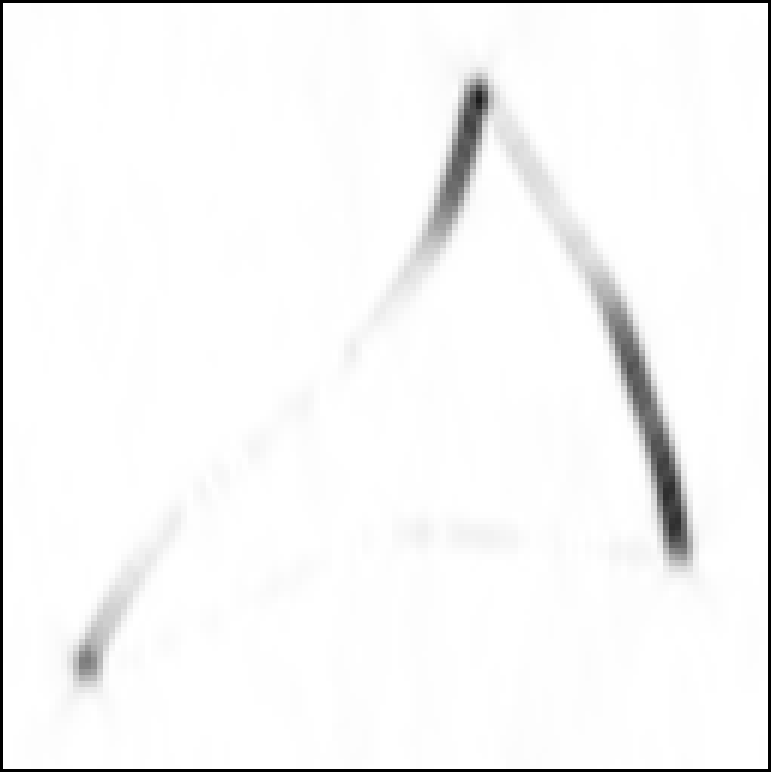}}

\end{minipage}\hspace{2cm}
\begin{minipage}[c]{.6\linewidth}
  \centering

  \subfigure[]
    {\label{fig:second_order_scattering}\includegraphics[width=\linewidth]{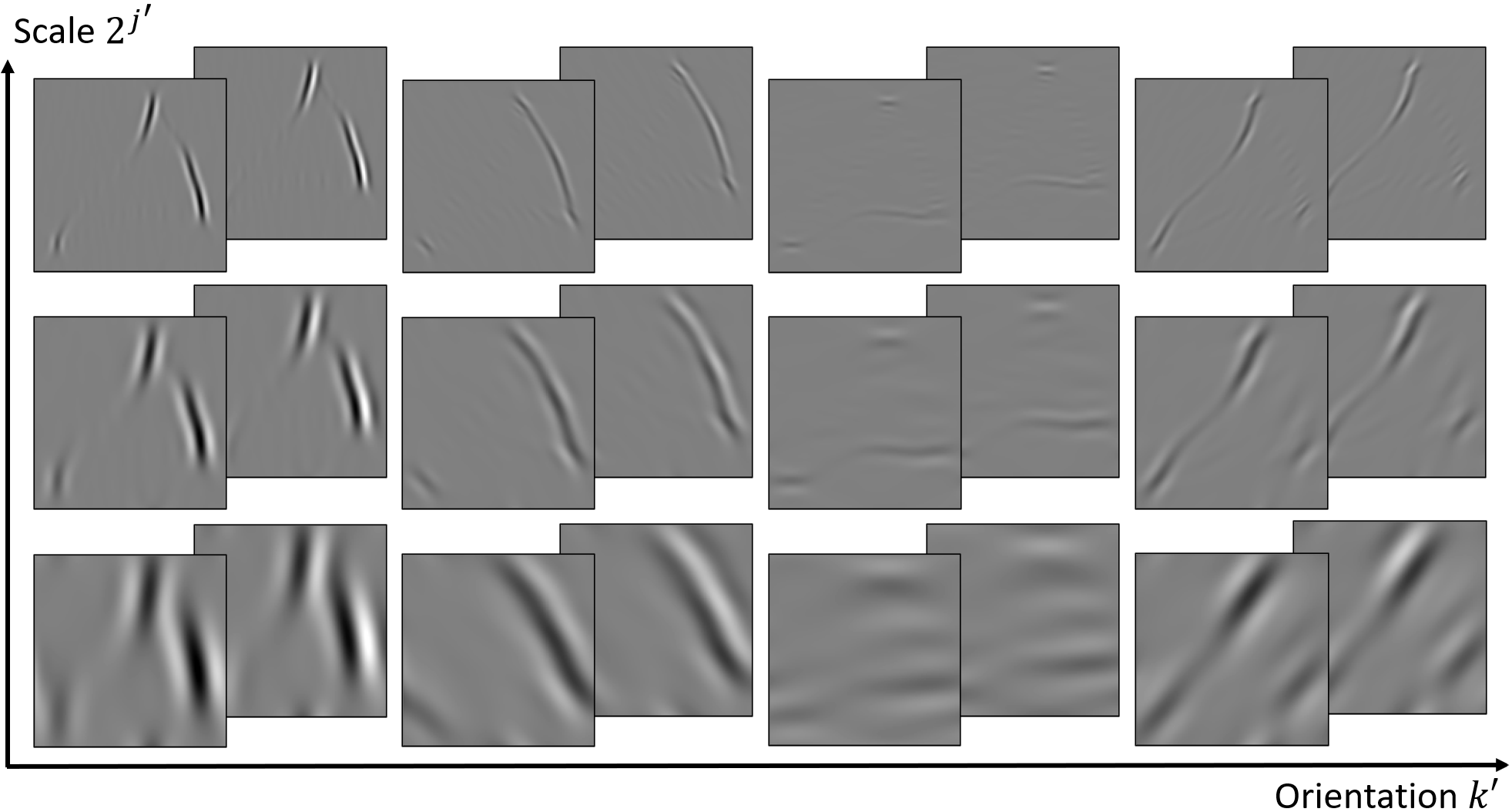}}

  \vspace{-2mm}

    \subfigure[]
  {\label{fig:second_order_modulus}\includegraphics[width=\linewidth]{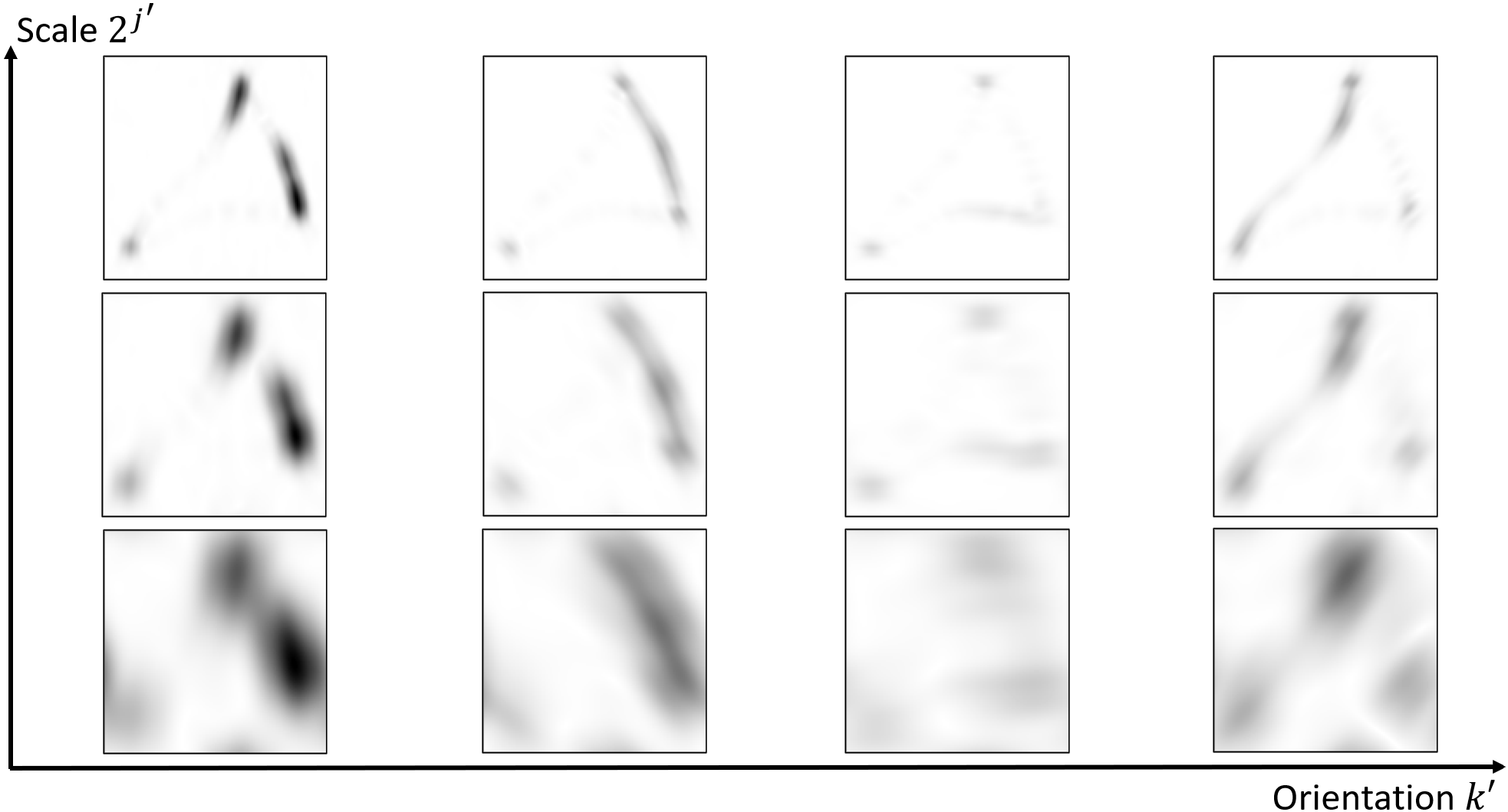}}
\end{minipage}
\caption{(a): Wavelet transform modulus $|f * \psi_j^k|$ for $j=j_m$ and $k = 0$. (b): Scattering coefficients $|f * \psi_j^k| * \psi_{j'}^{k'}$ for $j' > j$ varying over three dyadic scales and 4 orientations indexed by $k'$. (c): Modulus $||f * \psi_j^k| * \psi_{j'}^{k'}|$ of the previous scattering coefficients. Large amplitudes are obtained near edges which are oriented neither along $k \pi/4$ nor $k' \pi/4$. Coefficient maps from the third column, corresponding to $k' = k^{\perp}$, exhibit large amplitudes only in the vicinity of corners. Plots (b) and (c) are normalized so that images on a same row share the same minimum and maximum value.}\label{fig:1}
\end{figure} 

\paragraph{Directional regularity appears through small ${\bf L}^1$ norms}

For $\Ca$ geometrically regular images, the amplitude of $|f * \psi^k_j|$ has a fast decay at small scales $2^j$ besides the neighbourhoods of edges whose tangents are not aligned with $\psi^k_j$ vanishing moments, i.e. when the edge tangent has an angle which is not in $[(k-1) \pi / 4, (k+1) \pi / 4]$.
Applying $\psi^k_j$ therefore only partially sparsifies $f$, since on average half of the large wavelet coefficients remain (as contours are equally distributed across all orientations). In this paragraph, we show how to overcome this alignment issue by applying a wavelet $\psi^{k^\perp}_{j'}$, with $k^\perp = (k + 2) \mod 4$, to $|f * \psi^k_j|$.
Indeed, large amplitude values of $|f * \psi^k_j|$ are located at edge points where the direction of the tangent is in the complementary interval $[(k+1) \pi/4, (k+3) \pi / 4]$ modulo $\pi$. They vary regularly along the edge curve whose tangent 
is within this angular interval. However, this interval corresponds to the directions where $\psi^{k^\perp}_{j'}$ has vanishing moments, as illustrated in Figure \ref{fig:scattering_intuition}. 
With the same argument as in Theorem \ref{th:first_order}, a convolution with
$\psi^{k^\perp}_{j'}$ will produce small amplitude coefficients everywhere outside junctions between edges, which define corners.
The following theorem derives an upper bound on $\| |f * \psi_j^k| * \psi_{j'}^{k^\perp} \|_1$ which decreases to zero when the
second scale $2^{j'}$ decreases, for all $1 \leq \alpha \leq 2$. 
This result is obtained by proving that large coefficients at junctions do not
affect the asymptotic decay of the $\Lu$ norm.

\begin{figure}
    \centering
\includegraphics[width=50mm]{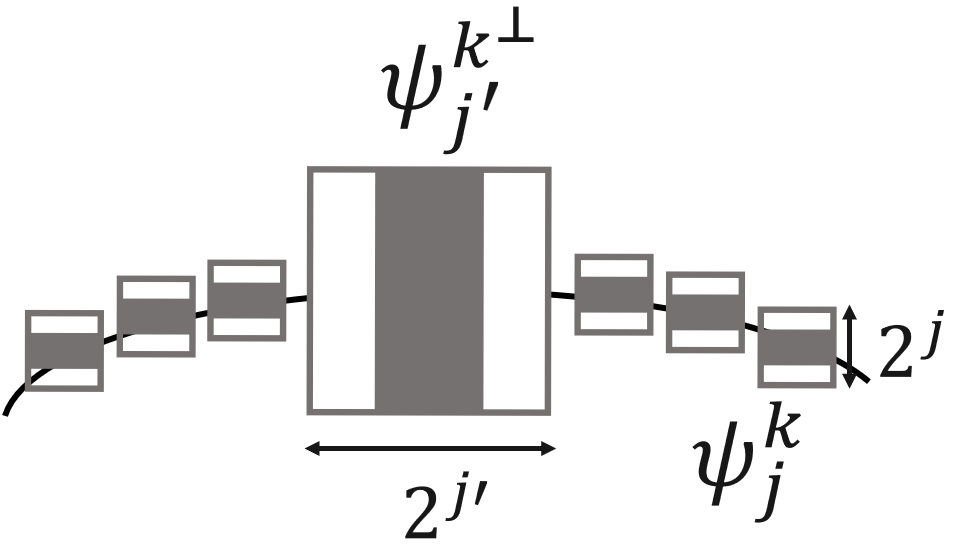}
    \caption{Wavelet coefficients $|f * \psi_j^k|$ are large at edge locations where $\psi_j^k$, shown as small squares, oscillate across the edge discontinuity. These coefficients have the same geometric regularity as the edge. Their wavelet transform $|f * \psi_j^k| * \psi_{j'}^{k^\perp}$ have a small amplitude because the larger scale wavelet $\psi^{k^\perp}_{j'}$, shown as a large square, oscillates along the perpendicular direction, i.e. along the edge tangent, where the edge is regular.
    }
    \label{fig:scattering_intuition}
\end{figure}

\begin{theorem}
\label{theorem:scattering_bounds}
Suppose that $\psi$ satisfies the same hypothesis as in Theorem \ref{th:first_order}.
If $f$ is a $\mathrm C^{\alpha}$ geometrically regular function with $1 \le \alpha \le 2$ then, for all $s < \alpha$, there exists $C$ such that:
\begin{equation}
\label{eq:theorem_nosum}
\forall j < j' \le \JJ,~ \forall k < 4, ~~ \| |f * \psi_j^k| * \psi_{j'}^{k^{\perp}} \|_1 \leq C 2^{ s j'},
\end{equation}
and
\begin{equation}
\label{eq:theorem}
    \forall s<\alpha,~ \forall k < 4, ~~ \sum_{j=-\infty}^\JJ \sum_{j'=j+1}^{[j/\alpha]} 2^{-s j'} \| |f * \psi_j^k| * \psi_{j'}^{k^{\perp}} \|_1  < + \infty.
\end{equation}

\end{theorem}

The proof is in \ref{chap:thm_proof}. It shows that 
$||f * \psi_j^k| * \psi_{j'}^{k^{\perp}}|$ is small where
$f$ is uniformly Lipschitz $\alpha$ because $|f * \psi_j^k|$ is small. This quantity remains small in the neighbourhood of edge points which are not corners, because
either $\psi_j^k$ or $\psi_{j'}^{k^{\perp}}$ can take advantage of the directional
image regularity along the edge tangent. These coefficients may however become large at corners where the image has no directional regularity, but since there is
a finite number of such corners and since wavelets have a square support, they do not
affect the asymptotic decay of the $\Lu$ norm. 
This is illustrated by
Figure \ref{fig:second_order_modulus} which shows that the only
large amplitudes of $||f * \psi_j^k| * \psi_{j'}^{k^{\perp}}|$ correspond to edge junctions.
The result and the proof 
remain valid if the modulus $|f * \psi^k_j|$ is replaced by 
a rectification $ReLU(f * \psi^k_j)$ of the real and imaginary parts, where $ReLU(a) = \max(a,0)$ is the rectifier used in most convolutional networks.  

As explained in Section \ref{sec:sparse_denoising}, the decay rate of decomposition coefficients in a frame is directly related to the denoising error obtained with a sparse prior (\ref{sparse-prior}). Theorem \ref{theorem:scattering_bounds} specifies the decay of second order wavelet coefficient amplitudes, which depends on $\alpha$ despite the presence of discontinuities along contours. Nevertheless, this decay is not sufficient to get an optimal denoiser by replacing the sparse prior of equation (\ref{sparse-prior}) with a combination of scattering $\Lu$ norms $(\| |f * \psi_j^k| * \psi_{j'}^{k^{\perp}} \|_1)_{j\le j'\le \JJ, ~k<4}$. Indeed, 
we are missing the other scattering coefficients $(|f * \psi_j^k| * \psi_{j'}^{k'})_{j \le j' \le \JJ, ~k' \neq k^\perp}$ to get an invertible representation of $(|f * \psi_j^k|)_{j \le \JJ, ~k<4}$.
For a $\Ca$ geometrically regular function,
these missing coefficients are not sparse but are concentrated in the neighbourhood of edge curves. To obtain an optimal denoising estimator, we must enforce this property with appropriate conditions on these scattering coefficients.

\paragraph{Edge profile captured by large $\Lu$ norms}
The edges of $\Ca$ geometrically regular functions correspond to sharp pointwise discontinuities if we move in the direction perpendicular to their tangents.
Section \ref{sec:Xlets} explains that it requires to adapt the support of curvelets, shearlets and bandlets to obtain a sparse representation. Wavelet coefficients $|f * \psi_j^k|$ have a large amplitude along the edge, with a narrow support when moving away from the edge, which is equal to the support of $\psi_j^k$ and hence proportional to $2^j$.
We shall see that this property is captured by imposing that scattering
coefficients $|f * \psi_j^k| * \psi^{k'}_{j'}$  have a sufficiently large $\Lu$ norm
for $k'\neq k^\perp$.
It appears through Young's inequality, which implies that
\begin{equation}
\label{eq:young}
\||f * \psi_j^k| * \psi_{j'}^{k'}\|_1  \leq \|f * \psi_{j}^k\|_1  \, \|\psi_{j'}^{k'} \|_1. 
\end{equation}
Since $\|\psi_{j'}^{k'} \|_1 = 1$ and $|f * \psi_j^k| \ge 0$, 
\begin{equation}
\label{eq:young-diff}
\|f * \psi_{j}^k\|_1 \|\psi_{j'}^{k'} \|_1 - \||f * \psi_j^k| * \psi_{j'}^{k'}\|_1  = 
\int  \left( \int |f * \psi_j^k (u)|\,  |\psi_{j'}^{k'} (v-u)| du  -
 \left | \int |f * \psi_j^k(u)|\, \psi_{j'}^{k'}(v-u) du \right| \right) dv
\geq 0 .
\end{equation}
This difference is minimised if, for all $v$, $|f * \psi_j^k|(u)$ is small except for values of $u$ such that $\psi_{j'}^{k'}(v-u)$ has a nearly constant phase or is negligible.
This property is studied by \citet{bruna2019multiscalesparsemicrocanonicalmodels} to characterize point processes.
For directional wavelets shown in Figure \ref{fig:wavelets_numerics}, $|\psi|$ is nearly isotropic and the lines of constant phase are nearly vertical.
Since $k'$ is a rotation index, 
the domain where $|\psi^{k'}_{j'}|$ is non-negligible 
does not depend upon $k'$. Over this domain the lines of constant phase 
are nearly vertical straight lines rotated by ${k'} \pi/4$ and thus
have an angle $k' \pi/4 + \pi/2$.
Remember that $|f * \psi^k_j|$ has a large amplitude near edge points where the edge tangent has an angle close to $k \pi / 4 + \pi/2$. 
If $k' = k$, lines of constant phase have thus an angular discrepancy of, at most, $\pi/8$ with edges' tangents. 

Given a fixed $\|f * \psi_j^k\|_1$, maximising $\||f * \psi_j^k| * \psi_{j'}^{k}\|_1$ requires to concentrate the support of $|f * \psi_j^k|$ on regions where the phase of $\psi_{j'}^{k}$ varies as little as possible.
Section \ref{sec:numerics} shows empirically that such regions could adapt to the tangent, despite the potential misalignment between lines of constant phase and edge tangents. Indeed, as explained in Section \ref{sec:numerics}, $\||f * \psi_j^k| * \psi_{j'}^{k}\|_1$ is jointly optimized with a data fidelity term which forces $f$ to be close from the noisy observation. In particular, for $j'$ big enough compared to the noise variance $\sigma$, tangents are constrained by this observation.

For a fixed edge tangent close to $k \pi/4 + \pi/2$, $\||f * \psi_j^k| * \psi_{j'}^{k}\|_1$ provides a measure of $|f * \psi_j^k|$'s support width.
If $k' = k \pm 1$, $||f * \psi_j^k| * \psi_{j'}^{k'}|$ is small when the edge tangent has an angle close to $k' \pi/4$, thanks to $\psi_{j'}^{k'}$ vanishing moments (the proof is nearly identical to that of Theorem \ref{theorem:scattering_bounds}). If not, $||f * \psi_j^k| * \psi_{j'}^{k'}|$ behaves as $||f * \psi_j^k| * \psi_{j'}^{k}|$, providing big coefficients whose amplitudes measure $|f * \psi_j^k|$'s support width. Such coefficients dominate in $\||f * \psi_j^k| * \psi_{j'}^{k'}\|_1$.
Maximising these $\Lu$ norm thus still concentrates the support width of $|f * \psi_j^k|$, but not as strongly as $\||f * \psi_j^k| * \psi_{j'}^{k}\|_1$ since the lines of constant phase are further away of the edge tangent.

These properties show that the amplitude of $\||f * \psi_j^k| * \psi_{j'}^{k'}\|_1$, for $k' \neq k^{\perp}$,
provides a measure of the concentration of wavelet coefficients $|f * \psi_j^k|$ around the edge curve, by taking advantage of cancellation effects due to the phase of $\psi_{j'}^{k'}$.
A parallel can be drawn with the phenomenon of activation maximization observed in neural networks. 
Whereas the first layer essentially encodes wavelet filters \cite{luan2018gabor}, in deeper layers it has been shown that the maximisation of neuron activations corresponds to the detection of important image patterns \cite{zeiler2014visualising, qin2018convolutional}. Scattering coefficients correspond to neurons in the second hidden layer, computed with wavelet filters. The maximisation of
$||f * \psi_j^k| * \psi_{j'}^{k'} (u)|$ for $k' \neq k^\perp$ corresponds to the detection of an edge point at $u$.
An outstanding issue is to relate the values of $\|f * \psi_j^k\|_1 - \| |f * \psi_j^k| * \psi_j^{k'} \|_1$ to precise mathematical properties of edge profiles.

\subsection{Numerics and conjecture on the minimax optimality of scattering denoising}
\label{sec:numerics}

This section shows numerically that the scattering $\Lu$ norms do provide asymptotically minimax estimations of $\Ca$ geometrically regular images contaminated by additive noise and gives a mathematical conjecture.
A noisy discretised signal is defined in equation (\ref{eq:discrete-noise})
by an orthogonal projection $P_\V g = P_\V f + Z$ of a function $f \in \Lambda$, 
where $Z$ is a a Gaussian white noise of variance $\sigma^2$. An estimation $\hat f(g)$ of $f$ is calculated by finding the solution of a variational problem
\begin{equation*}
    \hat f(g) = \argmin_{h \in \V} \frac 1 {2} \|h - P_\V g \|^2 + \sigma^2 U(h) .
\end{equation*}
If $U$ is defined by $\Lu$ norms of wavelet coefficients, then Section \ref{sec:wavelet} shows that it is suboptimal to estimate $\Ca$ geometrically regular images. This section verifies it numerically with a dyadic wavelet transform. We then show numerically that if this denoising estimator is also regularized by $\Lu$ norms of scattering coefficients, then it reaches the asymptotic minimax rate for all $1 \leq \alpha \leq 2$.  We state a mathematical conjecture on the minimax optimality of such scattering denoisers. 

\paragraph{Dyadic wavelet transform denoising}
A dyadic wavelet transform denoising is computed from the 
$\Lu$ norms of wavelet coefficients, limited to the finest discretisation scale $2^{j_m}$, up to a larger scale $2^{j_M}$:
\begin{equation}
\label{eq:scattering_first_order_energy}
U (h) =  \lambda \sum_{j=j_m}^{j_M} \sum_{k = 0}^{3} 2^{-j} \|h * \psi_j^k \|_1 .
\end{equation}
In numerical computations, $h * \psi^k_j$ is calculated with a discrete convolution
over the discretisation grid and the $\Lu$ norm is computed with an $\ell^1$ norm over this sampling grid. 
This estimator can be interpreted as a translation invariant wavelet denoising estimator \cite{Coifman1995TranslationInvariantD}.
The normalisation factor $2^{-j}$ is different from wavelet orthonormal bases because dyadic wavelets $\psi^k_j$ have an $\Lu$ normalisation and the wavelet transform is not subsampled.

The mean-squared error $\epsilon_{ms} (\sigma) = {\mathbb E}_{f,g}(\|\hat f(g) - f \|^2)$ 
is computed on average over a set $\Lambda$ of $\Ca$ geometrically regular images, defined in \ref{appendix:discretising_geometry}, for $\alpha = 2$. Typical images in this set have a fixed Lipschitz bound and an edge length which concentrates so that the mean-squared error is of the order of the maximum error. 
Images are discretised over a grid of dimension $d=128^2$.
The $\Ca$ images in $\Lambda$ are generated following the algorithm proposed by \citet{kadkhodaie2024generalization} to sample random edges which are $\Ca$ with
a bounded Lipschitz constant. 
The algorithm is further described with examples of images in \ref{appendix:discretising_geometry}. For each value of $\alpha$, the mean-squared error
$\epsilon_{ms}$ is computed on average over $100$ realizations in $\Lambda$.

We optimize numerically $\lambda$ in order to minimise $\epsilon_{ms}$, which gives
$\lambda = 1.2$.
Figure \ref{fig:mse_first_second_order} shows in red the value of the mean-squared error $\log \epsilon_{ms} (\sigma)$ as a function of $\log \sigma$, which has a slope of $1$. 
Although $\epsilon_{ms} \leq \epsilon_M$ it remains of the same order because it is computed over images having edges of nearly the same length and which therefore yield an estimation
error which is close to the maximum error $\epsilon_M(\sigma)$. The slope
is compatible with the theoretical bound computed over
orthogonal wavelet estimators (\ref{eq:waveletorth-estim}), where $\epsilon_M (\sigma) \sim \sigma\, |\log \sigma|$.
The effect of $\log \sigma$ term is difficult to detect numerically. 
However, the numerics shows that the error rate is well above
the minimax error $\epsilon_m (\sigma) \sim \sigma^{4/3}$ for $\alpha = 2$.

\paragraph{Scattering denoising}
A scattering transform denoising is computed from the $\Lu$ norms of scattering
coefficients. According to the analysis of Section \ref{sec:geometric_regularity_scattering}, the geometric regularity of $f$ is specified by minimising the $\Lu$ norms $\|f * \psi_j^k\|_1$ which depend on the total length of the edge curves, by minimising $\| |f * \psi_j^k| * \psi^{k^\perp}_{j'}\|_1$ which specifies the directional image regularity parallel to the edge curves, and by maximising $\| |f * \psi_j^k| * \psi^{k'}_{j'}\|_1$ for $k' \neq k^\perp$ which specifies the edge profile across the edge. This joint optimization is computed by defining a regularity metric which is a weighted sum of these 3 terms, with a negative sign for the $\Lu$ norms that must be maximised
\begin{equation}
\label{eq:scattering_energy}
U (h) =  \sum_{j=j_m}^{j_M} \sum_{k = 0}^{3} \left(
\lambda  \,2^{-j} \|h * \psi_j^k \|_1 
+ \gamma  
 \sum_{j' > j}^{0} 2^{-j'}  \||h *  \psi_j^k| * 
 {\psi}_{j'}^{k^\perp} \|_1
 - \sum_{k' \neq k^\perp} \eta_{k'-k} 
 \sum_{j' > j}^{0} 2^{-j}  \||h *  \psi_j^k| * 
 {\psi}_{j'}^{k'} \|_1 \right)
\end{equation}
As in the dyadic wavelet transform case, experiments are performed over sets of $\Ca$ geometrically regular images with corners, for different $1 \leq \alpha \leq 2$.
We optimize numerically the constants $\lambda$, $\gamma$, $\eta_0$, $\eta_{\pm 1}$
in order to minimise the mean-squared error $\log \epsilon_{ms}$ over a range of
noise values $\sigma$,
with an optimization procedure discussed in \ref{appendix:denoisers_and_optim}.
We impose that $\eta_1 = \eta_{-1}$ because $\psi^{k+1}_{j'}$ and $\psi^{k-1}_{j'}$ have a symmetric role relatively to $|f * \psi_j^k|. $
The computed optimized constants are independent of $\sigma$ and $\alpha$:
\[
\lambda = 1.9~~,~~\gamma = 1.4~~,~~\eta_0 = 0.52~~,~~\eta_{\pm 1} = 0.10 .
\]
Figure \ref{fig:mse_first_second_order} shows in blue the value of the mean-squared error $\log \epsilon_{ms} (\sigma)$ as a function of $\log \sigma$ for $\alpha = 2$, computed on average over
$\Ca$ geometrically regular images. This error is below the error produced by the dyadic wavelet transform estimator, which shows the importance of the scattering $\Lu$ norms. More important, it has a different slope equal to $4/3$. It indicates that
$\epsilon_{ms} (\sigma) \sim \sigma^{4/3}$ up to a $|\log \sigma|$ term that cannot be detected due to numerical uncertainties, which means that it has an asymptotic minimax rate.
Figure \ref{fig:mse_varying_alpha} gives the value of $\log \epsilon_{ms} (\sigma)$ as a function of $\log \sigma$ for $4$ Lipschitz exponents $ \alpha \in \{2, 1.5, 1.2, 1\}$. We verify numerically that $\log \epsilon_{ms} (\sigma)$ has a linear
increase in $\log \sigma$ which superimposes with lines of slope
$2 \alpha / (\alpha + 1)$ shown in dashed line.  It indicates that
$\epsilon_{ms} (\sigma) \sim \sigma^{2 \alpha / (\alpha + 1)}$ up to a $|\log \sigma|$ term, which is the asymptotic minimax rate for all $\alpha$. This shows that these scattering denoisers reaches
the minimax rate for all $1 \leq \alpha \leq 2$, despite the presence of corners, and
can thus adapt to unknown $\alpha \leq 2$.

The noise range on which the experiment is performed is identical for dyadic wavelet transform and scattering but differs from the one used for the UNet of Figure \ref{fig:UNet_slopes}. Indeed, the UNet outperforms both estimators and is able to exploit the global shape of the deformed triangular foreground to denoise its contour.

\begin{figure}
    \centering
    \subfigure[]{\label{fig:mse_first_second_order}\includegraphics[width=75mm]{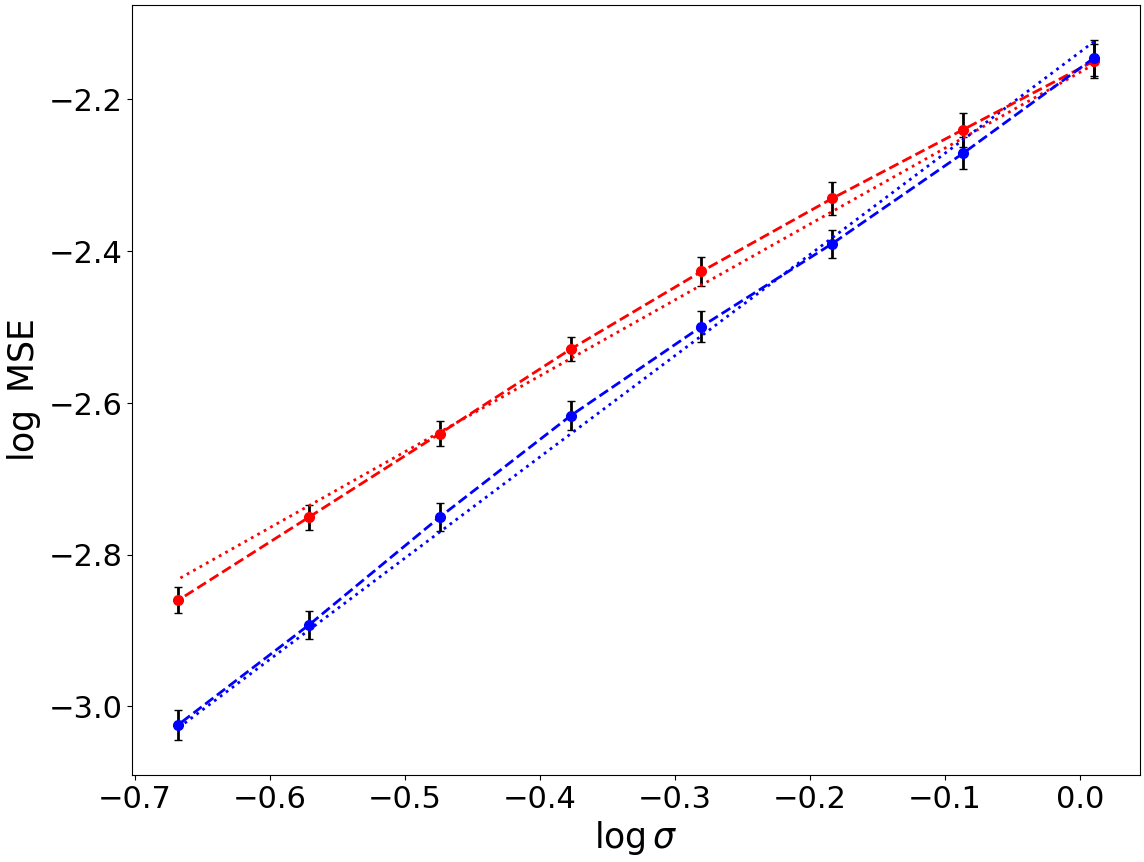}}
    \hspace{1cm}
    \subfigure[]{\label{fig:mse_varying_alpha}\includegraphics[width=75mm]{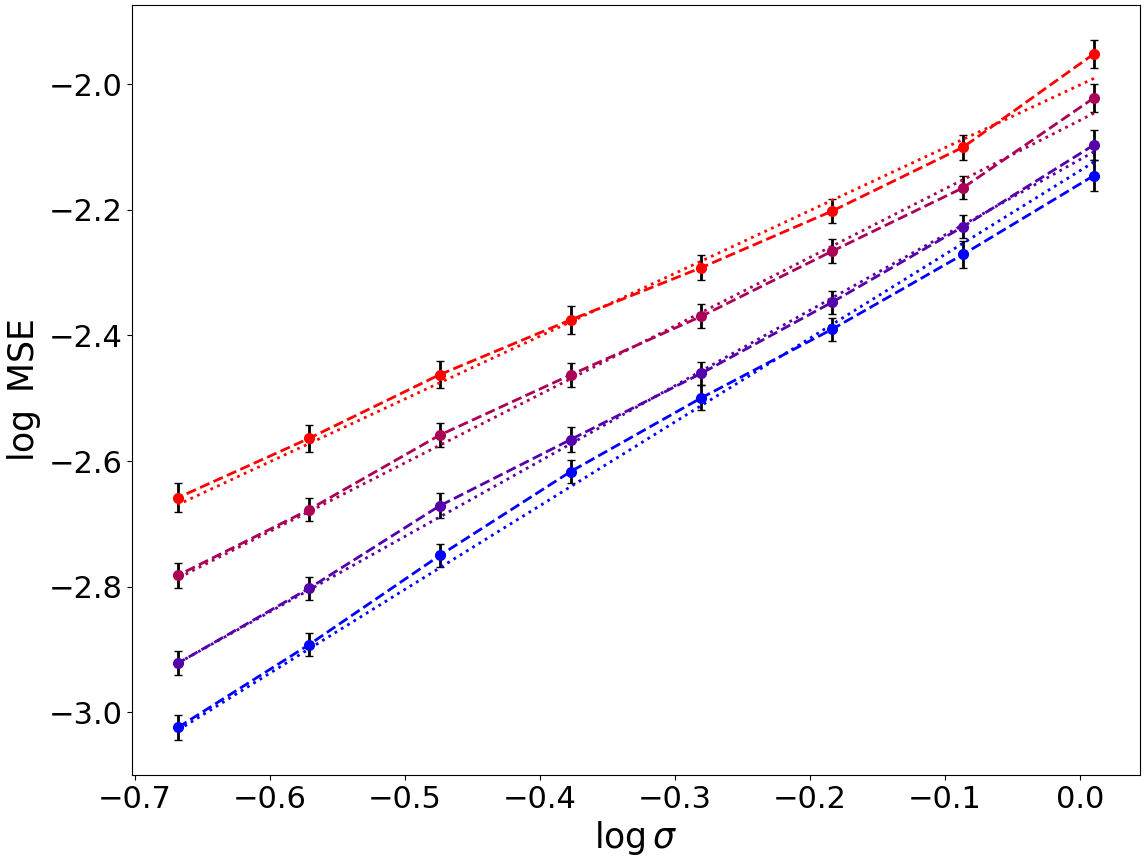}}
    \caption{(a): Comparison between the mean-squared errors $\log \epsilon_{ms} (\sigma)$ of a dyadic wavelet estimator (red dashed curve) and a scattering estimator (blue dashed curve), computed on average over $\bf C^2$ geometrically regular images defined in \ref{appendix:discretising_geometry}. The blue and red dotted lines have a slope of $1$ and $4/3$. (b): mean-squared errors $\log \epsilon_{ms} (\sigma)$ for a scattering denoiser and $C^{\alpha}$ geometrically regular images, for different $\alpha$. The different dashed coloured curve from red to blue corresponds to $\alpha \in \{2, 1.5, 1.2, 1\}$. These curves superimpose with the dotted lines having a slope $2 \alpha/(\alpha+1)$.}
    \label{fig:scattering_denoising_slopes}
\end{figure}

\begin{figure}
    
    \centering
    \label{fig:denoised_scattering}
    \subfigure[Original]{\includegraphics[width=40mm]{Figure_1.png}}
    \subfigure[Noisy]{\includegraphics[width=40mm]{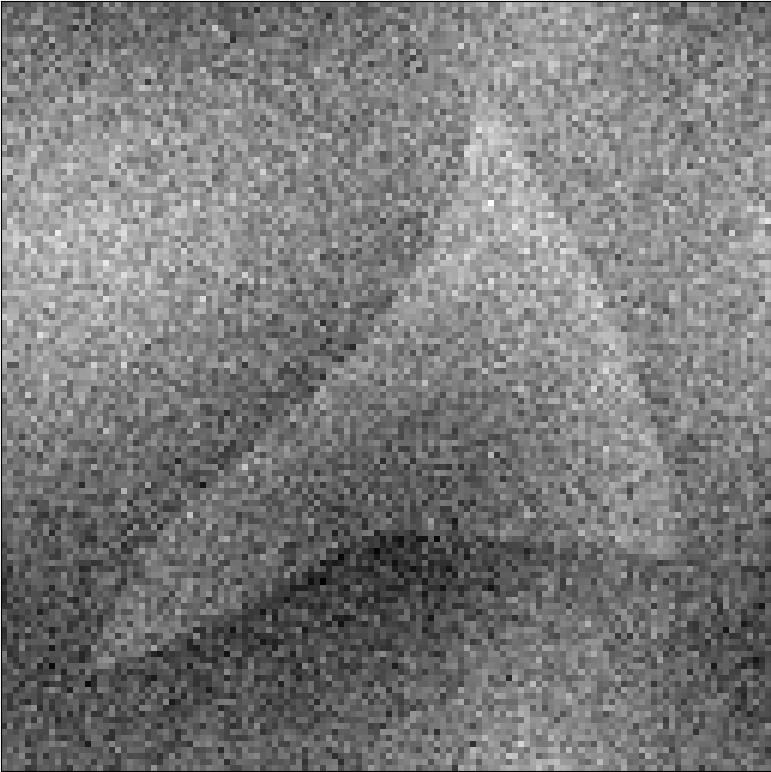}}
    \subfigure[Dyadic wavelet estimator]{\includegraphics[width=40mm]{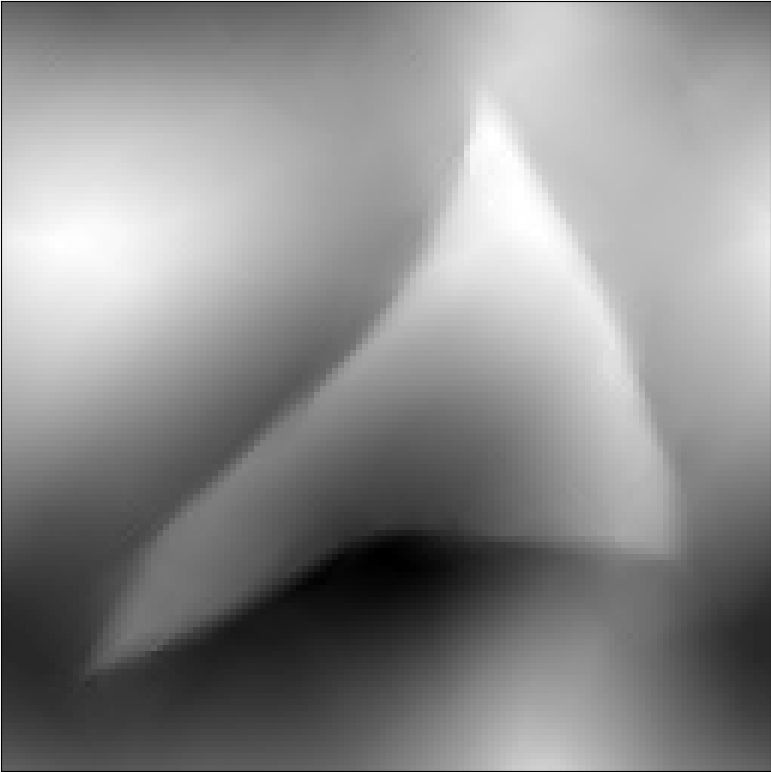}}
    \subfigure[Scattering estimator]{\includegraphics[width=40mm]{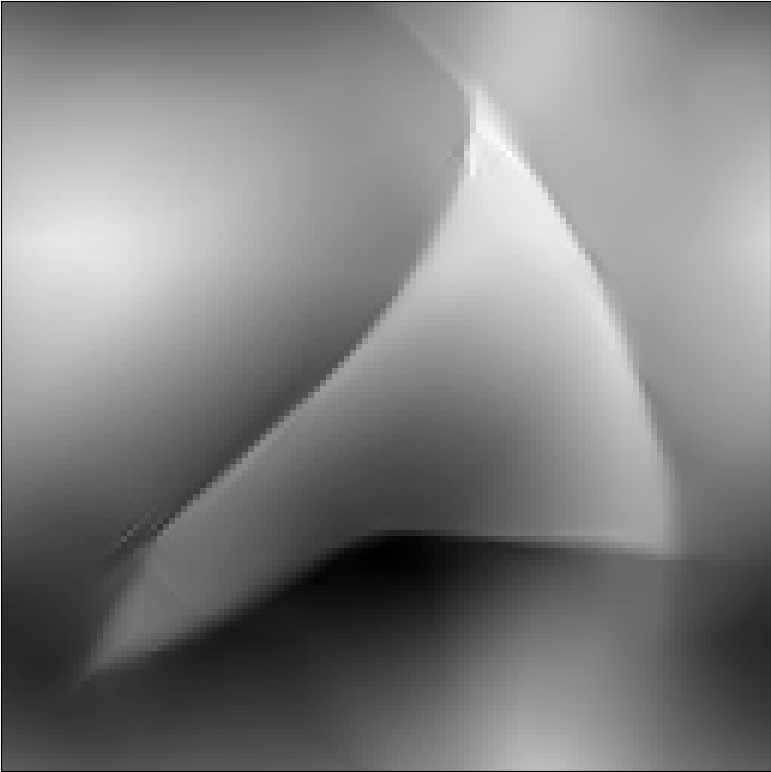}}
    \subfigure[Original, zoom]{\includegraphics[width=40mm]{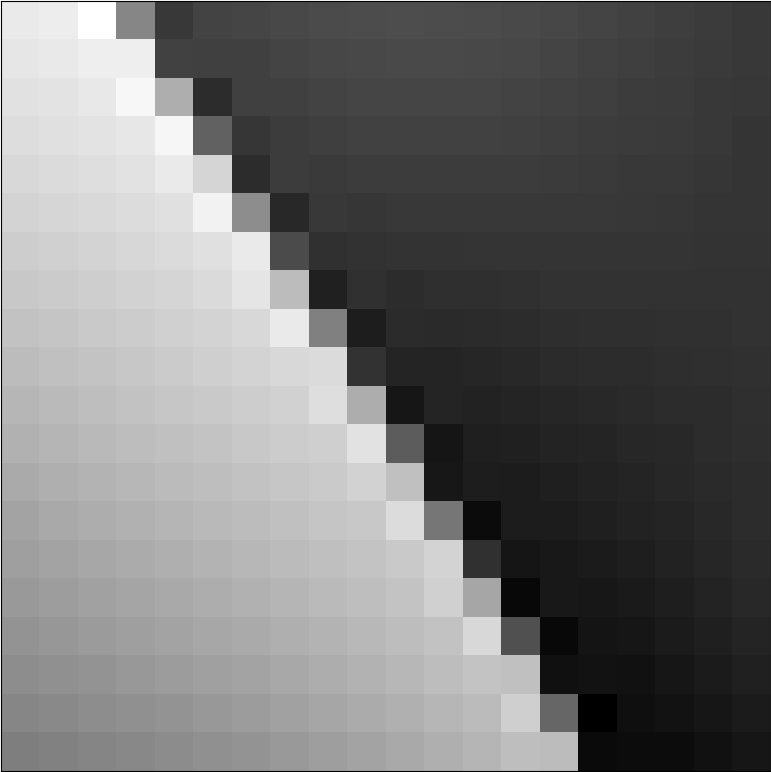}}
    \subfigure[Noisy, zoom]{\includegraphics[width=40mm]{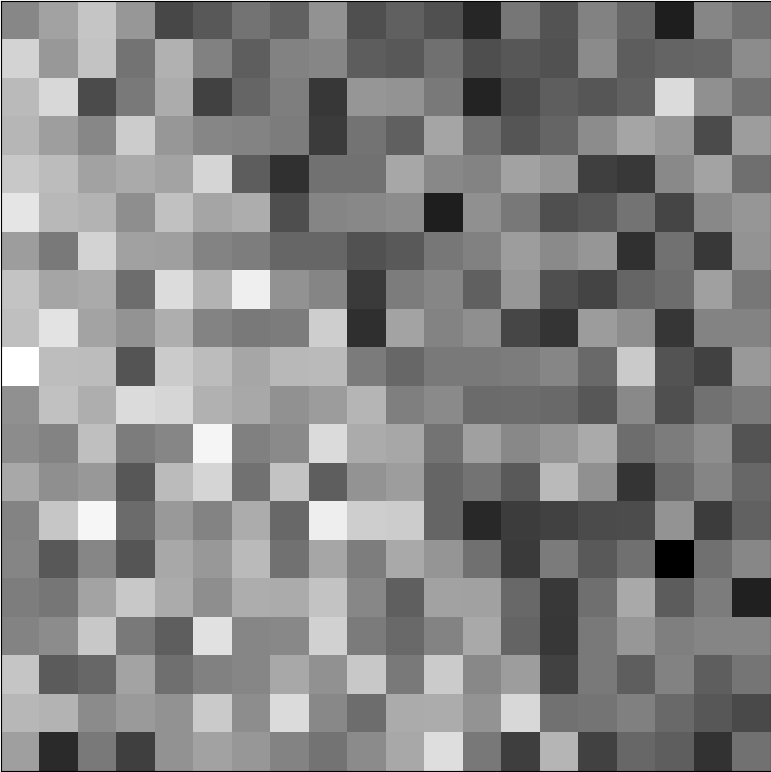}}
    \subfigure[Dyadic wavelet estimator, zoom]{\includegraphics[width=40mm]{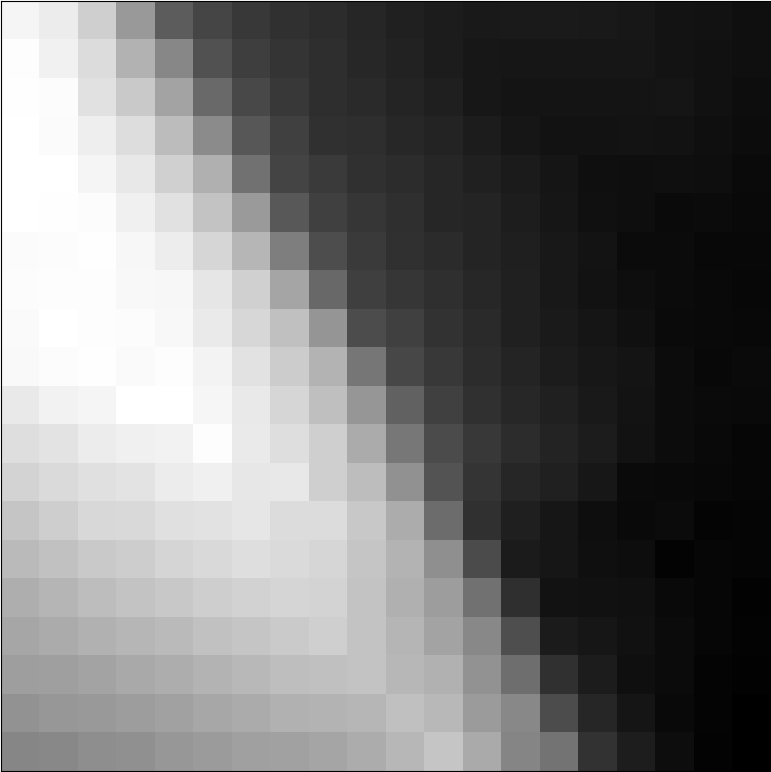}}
    \subfigure[Scattering estimator, zoom]{\includegraphics[width=40mm]{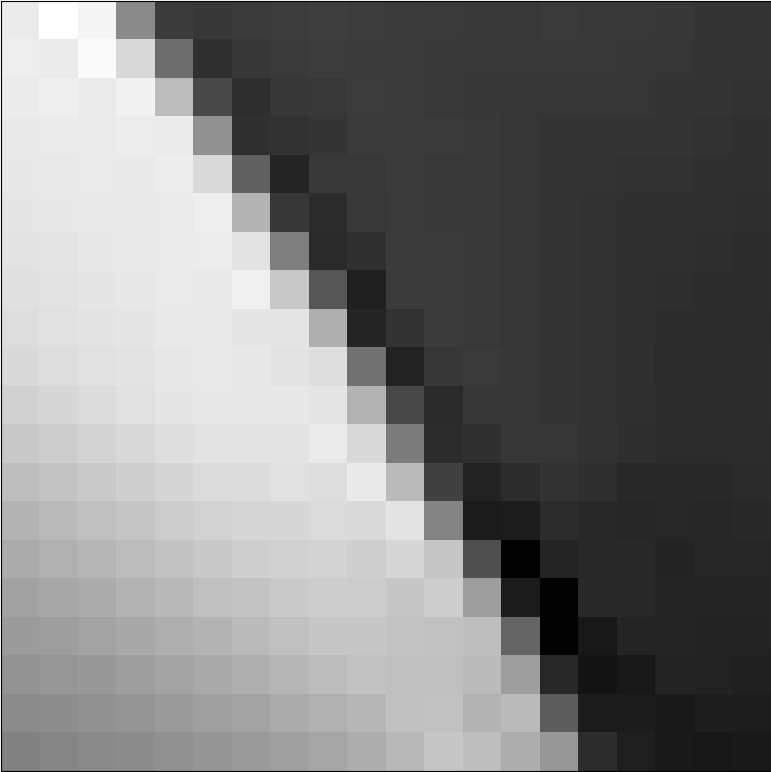}}

    \caption{Top row: (a) uncorrupted $\mathrm C^2$-geometrically regular image discretised in dimension $d=128^2$, ranging in $[-1,1]$, (b) noisy version of this image, corrupted with an additive Gaussian white noise of variance $\sigma^2 = 0.1$, (c) denoised image obtained with a dyadic wavelet estimator, (d) denoised image obtained with a scattering estimator. Bottom row: images from the top row are zoomed on a region with a corner and edges ((e) original, (f) noisy, (g) dyadic wavelet estimator, (h) scattering estimator). The PSNR of the noisy image is improved from 16.01dB to 32.21 dB by the dyadic wavelet estimator and to 34.55 dB by the scattering estimator.
    }

    \label{fig:scattering_denoising}
\end{figure}

Figure \ref{fig:scattering_denoising} illustrates the denoising results of 
dyadic wavelet transform and scattering transform for
a particular image $f \in \Lambda$, for $\alpha =2 $. It is normalized between -1 and 1 and corrupted with a noise of variance $\sigma^2=0.1$. As expected, 
the mean-squared error improvement of the scattering estimator corresponds to the
recovery of sharper edges, because the $\Lu$ norm of scattering coefficients give information on the edge profile. At corners, the scattering transform does not do any better than the wavelet estimators, but it has no impact on the asymptotic performances of the estimator because there is a finite number of corners which have
no contribution to the asymptotic rate of the error. The UNet estimation shown in Figure \ref{fig:MMS-denoising} provides a much better recover of corners. 
A scattering transform estimator is much simpler but its performances are still far from the performances of a deep neural network estimator.

\paragraph{Conjecture on minimax scattering optimality}
Numerical denoising experiments indicate that a denoising estimator regularized by
scattering $\Lu$ norms has a maximum error over $\Ca$ geometrically regular images which reaches the asymptotic minimax rate. This is formalized by the following conjecture.

\begin{conjecture}
Let $\Lambda$ be a set of $\Ca$ geometrically regular functions with bounded
Lipschitz constants, for $1 \leq \alpha \leq 2$.
Let $\hat f(g)$ be a scattering estimator of $f$ contaminated by a Gaussian white
noise of variance $\sigma$. 
There exists a constant $C$ which does not depend upon $\alpha$
and $\sigma$ such that
\begin{equation}
\sup_{f \in \Lambda} {\mathbb E_g}(\|\hat f(g) - f \|^2)
\leq C \,|\log \sigma|\, \sigma^{2 \alpha/(\alpha+1)} .
\end{equation}
\end{conjecture}

The main difficulty to prove this conjecture is to prove the impact of the
scattering $\Lu$ norms $\| |\hat f * \psi_j^k| * \psi_{j'}^{k'} \|_1$ for $k' \neq k^\perp$ on the denoising error. These $\Lu$ norms are maximised and not minimised as it is usually done in sparse representations, and it is the lines of constant phase of $\psi_{j'}^{k'}$ which regularize $|\hat f * \psi_j^k|$ and hence the edge profile of $\hat f$. If the modulus is replaced by a rectifier, then the phase is replaced by the sign of real coefficients. Clearly, the maximisation of these $\Lu$ norms provide a very different approach than minimisations of $\Lu$ norms of frame coefficients, to specify the geometric regularity of a function. The adaptivity of the resulting scattering estimator shows that it opens the possibility to create much more flexible estimators, as observed with deep neural networks. 
This conjecture gives a mathematical framework and a first step to better understand the mathematics underlying the performances of deep neural network estimators.

\paragraph{Acknowledgements} This work was supported by PR[AI]RIE-PSAI-ANR-23-IACL-0008 and the DRUIDS projet ANR-24-EXMA-0002. It was granted access to the HPC resources of IDRIS under the allocations 2024-AD011014371R1 and 2025-AD011016159R1 made by GENCI. The authors thank Etienne Lempereur for his valuable feedback on the manuscript.

\appendix

\section{Wavelet with directional vanishing moments}
\label{app:wavelet-design}

In the following we explain how to define a wavelet having directional vanishing moments in
an angular cone. Lemma \ref{lemma:cone} gives a sufficient condition to define the
Fourier transform of such wavelets. We then detail how wavelets used in experiments of Section \ref{sec:numerics} are built.

\begin{lemma}
\label{lemma:cone}

If the Fourier transform $\hat \psi$ of $\psi$ is $\C^{m-1}$ and satisfies $\hat \psi(r_{-\theta}(\om_1,\om_2)) = O(|\om_1|^m)$ for all $\theta \in [-\pi/4 , \pi / 4]$, where $r_{\varphi}$ denotes the rotation of angle $\varphi$, then for any
one-dimensional polynomial function $u_1 \mapsto q(u_1)$ of degree $m-1$:
\begin{equation}
\label{eq:vanish}
\forall \theta \in [-\pi/4 , \pi / 4]~~,~~\forall u_2 \in \R~,~ \int q(u_1) \psi(r_\theta(u_1,u_2)) du_1 = 0 .
\end{equation}
Moreover, $\psi$ has $m$ vanishing moments, i.e. for any
two-dimensional polynomial function $(u_1, u_2) \mapsto q(u_1, u_2)$: $\int  q(u_1, u_2) \psi(u_1,u_2) du_1 du_2 = 0 $.
\end{lemma}

\begin{proof}

Assume that $\hat \psi$ is $\C^{m-1}$.
We first show that $\psi$ has $m$ directional vanishing moments in directions $\theta \in [-\pi/4 , \pi / 4]$.
If $\hat \psi (\om_1,\om_2) = O(\om_1^m)$ then $\partial^n_{\om_1} \hat \psi (0,\om_2) = 0$ for all $n \leq m-1$.
It results that $\int u_1^n\psi(u_1,u_2) du_1 = 0$. Performing a rotation of angle $\theta$ proves that $\int u_1^n\psi(r_{\theta}(u_1,u_2)) du_1 = 0$ if
$\hat \psi (r_{-\theta}(\om_1,\om_2)) = O(|\om_1|^m)$, which is the first statement of the lemma.
We prove that $\psi$ has $m$ vanishing moments using, again, that $\partial^n_{\om_1} \hat \psi (0,\om_2) = 0$ for all $n \leq m-1$. Indeed, it implies that all partial derivatives of $\hat \psi$ of order $n \leq m-1$ vanish at $(0,0)$, from which we deduce the result.

\end{proof}

Numerical experiments of Section \ref{sec:numerics} are performed with Morlet/Gabor wavelets \cite{gabor1946theory, morlet1982wavea, morlet1982waveb}, modified to fulfil conditions of Lemma \ref{lemma:cone} with $m=2$. The mother wavelet $\psi$ is defined in the Fourier space as a product between the Fourier transform $\hat \psi_{\text{Morlet}}$ of a Morlet wavelet with two vanishing moments and a mask $\mu$, defined in equation (\ref{eq:mask}), satisfying $\mu(\om_1,\om_2) = 0$ for all $(\om_1,\om_2)$ having a polar coordinate angle greater than $\pi / 4$:

\begin{equation}
\label{eq:wavelet}
    \hat \psi = \hat \psi_{\text{Morlet}} \times \mu.
\end{equation}

The resulting wavelet $\psi$ is $\C^{\infty}$, has two vanishing moments and two directional vanishing moments in the cone of angles $[-\pi/4, \pi/4]$.
Figure \ref{fig:wavelets_numerics} shows the real and imaginary parts of $\psi$, as well as its modulus and phase. Contrary to Theorem \ref{theorem:scattering_bounds}'s hypothesis, $\psi$ does not have a fast decay, $\hat \psi$ being only $\C^1$. This slowest asymptotic decay rate has no numerical impact, since we observe empirically that $\psi$ inherits the spatial localization properties of $\psi_{\text{Morlet}}$ for dimension $d=128^2$ used in experiments of Section \ref{sec:numerics}.

In what follows, we review the definition of Morlet wavelets and we define the mask $\mu$. We then prove that the resulting wavelet $\psi$ satisfies Lemma's \ref{lemma:cone} hypothesis for $m=2$.

\paragraph{Morlet wavelets} A Morlet wavelet $\psi_{\text{Morlet}}$ is defined as a product between a plane wave and a Gaussian envelop:
$$\psi_{\text{Morlet}}(u) = c_{\sigma} e^{-\|u\|^2/(2 \sigma^2)} (e^{i \xi \cdot u} - k_{\sigma, \xi}),$$
where $x \cdot y$ denotes the standard scalar product in $\mathbb R^2$.
The parameter $k_{\sigma, \xi}$ is a scalar, adjusted such that $\psi$ has one vanishing moment: $\int \psi(u) du = 0$. 
Numerically, we use $\sigma = 0.7$ and $\xi = (1.05 \pi, 0)$.
A Morlet wavelet with two vanishing moments, still denoted $\psi_{\text{Morlet}}$, is defined by:
$$\psi_{\text{Morlet}}(u) = c_{\sigma} e^{-\|u\|^2/(2 \sigma^2)} (e^{i \xi \cdot u} - k_{\sigma, \xi} - i \tilde k_{\sigma, \xi} \cdot u ),$$
where $\tilde k_{\sigma, \xi}$ is a vector of $\mathbb R^2$ adjusted such that $\int u_1 \psi(u) du = 0$ and $\int u_2 \psi(u) du = 0$. Adjusting $\tilde k_{\sigma, \xi} = (\tilde k_{\sigma, \xi}^1, \tilde k_{\sigma, \xi}^2)$ is sufficient to impose vanishing partial derivatives of $\hat \psi_{\text{Morlet}}$ at $(0,0)$ since for $i \in \{1,2\}$:
$$\partial_{\om_i} \hat \psi_{\text{Morlet}} (0,0) = \sigma^2 \xi_1 e^{- \sigma^2 \|\xi \|^2/2} - \frac{\sigma^2}{2 \pi} \tilde k_{\sigma, \xi}^i.$$
The corresponding wavelet transform $W$, used in Section \ref{sec:numerics}, satisfies a Littlewood-Paley formula for $\delta \in (0,1)$, guaranteeing its invertibility:
$$\forall \omega \neq 0, 1 - \delta \le |\hat \phi_{J_M}(\omega)|^2+ 
\sum\limits_{j \le J_M,k} |\hat \psi_{j}^k|^2(\omega) \le 1 + \delta.$$

\paragraph{Constraint on the support of $\hat \psi$}

The mask $\mu$ from equation (\ref{eq:wavelet}) used in our numerical experiments is defined in polar coordinates as follows:

\begin{equation}
\label{eq:mask}
\forall r > 0, ~\mu(r,\varphi) = \tilde \mu( \varphi) \mathbb 1_{|\varphi|<\pi/4}= \sin \left(\frac{\pi(1+\sin(4\varphi+\pi/2))}{4} \right) \mathbb 1_{|\varphi|<\pi/4} ~~ \text{and} ~~ \forall \varphi \in [-\pi , \pi), ~\mu(0,\varphi) = 0.
\end{equation}
It is $\C^1$ everywhere except at the origin and satisfies $\mu(r,\varphi) = 0$ for all $r>0$ and $\varphi \notin [-\pi/4 , \pi / 4]$. Moreover, it has vanishing angular derivatives on the borders of the cone of angles $[-\pi/4, \pi/4]$: $$\forall r>0,~ \partial_\varphi \mu(r, \pi/4) = \partial_\varphi \mu(r, - \pi/4) = 0.$$
Denoting $(\mu_k)_k$ rotated versions of $\mu$ in directions $(k \pi/4)_{k < 4}$, this mask verifies: $\sum\limits_k \mu_{k}^2 = 1$.

\paragraph{Fulfillment of Lemma's \ref{lemma:cone} hypothesis}

To show that $\hat \psi$, defined through equation (\ref{eq:wavelet}), satisfies Lemma's \ref{lemma:cone} hypothesis for $m=2$, we need to prove that it is $\C^1$ and that $\hat \psi(r_{-\theta}(\om_1,\om_2)) = O(\om_1^2)$ for all $\theta \in [-\pi/4 , \pi / 4]$.

The Fourier transform $\hat \psi_{\text{Morlet}}$ and $\mu$ being respectively $\C^{\infty}$ on $\mathbb R^2$ and $\C^1$ on $\mathbb R^* \times \mathbb R^*$, $\hat \psi$ is $\C^1$ on $\mathbb R^* \times \mathbb R^*$. To prove that $\hat \psi$ is $\C^1$ on $\mathbb R^2$, one has to verify that its partial derivatives exist and are continuous in $(0,0)$. Since $\hat \psi_{\text{Morlet}}(0,0) = 0$ and since $\hat \psi_{\text{Morlet}}$ is a $\C^{\infty}$ function whose first order derivatives vanish at the origin, it satisfies $\hat \psi_{\text{Morlet}} (r, \varphi) = O(r^2)$ in polar coordinates. $\mu$ being bounded, this implies that $\hat \psi$'s partial derivatives exist and vanish at $(0,0)$. $\partial_{\varphi} \mu$ being bounded too, it also implies that $\lim\limits_{\om \rightarrow (0,0)} \partial_{\om_1} \hat \psi(\om_1, \om_2) = \lim\limits_{\om \rightarrow (0,0)} \partial_{\om_2} \hat \psi(\om_1, \om_2) = 0$, which is proven expressing these partial derivatives in polar coordinates. $\hat \psi$ is thus $\C^1$ on $\mathbb R^2$.

Since $\mu$'s support is included in the cone of angles $[-\pi/4 , \pi / 4]$, it is sufficient to show that $\hat \psi(r_{-\theta}(\om_1,\om_2)) = O(\om_1^2)$ for $\theta = \pm \pi/4$ to get this result for all $\theta \in [-\pi/4 , \pi / 4]$. The indicator function $\mathbb 1_{|\varphi|<\pi/4}$ does not play any role in the decay rate of $\hat \psi(r_{-\theta}(\om_1,\om_2))$ for $\omega_1 \rightarrow 0$. The result can thus be obtained studying the partial derivatives along $\omega_1$ of the function $\hat {\tilde \psi} (\om_1,\om_2)\mapsto (\hat \psi_{\text{Morlet}} \times \tilde \mu) (r_{-\theta}(\om_1,\om_2))$, where $\tilde \mu$ is defined in equation (\ref{eq:mask}). Expressing these derivatives in polar coordinates and using the fact that $\tilde \mu$ has vanishing angular derivatives for $|\varphi| \rightarrow \pi/4$ and that $\hat \psi_{\text{Morlet}} (r, \varphi) = O(r^2)$, one shows that $\partial_{\om_1} \hat {\tilde \psi} (\om_1, \om_2) = 0$ for all $\om_2$ and that $\partial_{\om_1}^2 \hat {\tilde \psi} (\om_1, \om_2) \le C$ for all $\om_2 \neq 0$. We thus have $\hat \psi(r_{-\theta}(\om_1,\om_2)) = O(\om_1^2)$ for all $\om_2 \neq 0$. The case $\om_2 = 0$ is given by $\hat \psi_{\text{Morlet}}$ decay for $\om \rightarrow 0$.

\section{Proof of Theorem \ref{th:first_order}}

\begin{theorem*}
Suppose that $\psi$ is a $\C^2$ function with two vanishing moments and two directional vanishing moments in the cone of angles $[- \pi/4, \pi/4]$, such that $|\psi|$ decays faster than any rational function. Let $\alpha \in [1,2]$. Let $f$ be a periodic image on $[0,1]^2$ which is uniformly Lipschitz $\alpha$ above and below an edge curve, which is itself Lipschitz $\alpha$ and whose tangents have an angle in $[(k-1) \pi / 4 , (k+1) \pi / 4 ]$. 
For all $s < \alpha$, there exists $C$ such that:
\begin{equation}
\forall j \le \JJ, ~ \forall u \in [0,1]^2,~~| f * \psi^k_j (u)| \leq C 2^{(s-1) j} .
\end{equation}
\end{theorem*}

\vspace{.2cm}

\begin{proof}

If $\alpha=1$, the result is immediate, $f$ being bounded and $\psi_{j}^k$ being $\Lu$ normalized. In what follows we assume that $\alpha \in (1,2]$.
Thanks to $\psi$ fast decay, for all $\epsilon>0$, there is a constant $C>0$ such that:
    \begin{equation}
        \label{eq:fast_decay}
        \int_{\| u \| > C_{\psi} 2^{(1-\epsilon)j}} |\psi_j(u)| du \le C 2^{\alpha j},
    \end{equation}
    where $C_{\psi}$ only depends on $\psi$. In the following, we consider $\epsilon \le 1/2$. One has:
    \begin{equation}
        \label{eq:cutting_integral}
        | f * \psi^k_j (u)| \le \left| \int_{\| v \| \le C_{\psi} 2^{(1-\epsilon)j}} f(u-v) \psi_j^k(v) dv \right| + \|f\|_{\infty} \int_{\| v \| > C_{\psi} 2^{(1-\epsilon)j}} |\psi_j^k(v)| dv.
    \end{equation}
    The second integral is controlled thanks to equation (\ref{eq:fast_decay}), which shows that it is $O(2^{\alpha j})$. In the following, we prove that the first integral is $O(2^{(1-\epsilon)(\alpha-1)j})$. There are two cases depending whether the edge curve $\gamma$ intersects or not the ball $B(u, C_{\psi}2^{(1-\epsilon)j})$. If $\gamma$ does not intersect $B(u, C_{\psi}2^{(1-\epsilon)j})$, $f|_{B(u, C_{\psi}2^{(1-\epsilon)j})}$ is uniformly $\Ca$ and can be extended to a uniformly $\Ca$ function $\tilde f$ on $[0,1]^2$. Cutting the integral inside and outside $B(u, C_{\psi}2^{(1-\epsilon)j})$, one gets:
    \begin{equation}
        \label{eq:regular_coef}
        \left| \int_{\| v \| \le C_{\psi} 2^{(1-\epsilon)j}} f(u-v) \psi_j^k(v) dv \right| \le \left| \int \tilde f(u-v) \psi_j^k(v) dv \right| + \|f - \tilde f\|_{\infty} \int_{\| v \| > C_{\psi} 2^{(1-\epsilon)j}} |\psi_j^k(v)|dv \le C 2^{\alpha j},
    \end{equation}
    where the first integral is bounded using $\psi$ vanishing moments through equation (\ref{decayCalpha}) and the second one with equation (\ref{eq:fast_decay}). Finally, combining equations (\ref{eq:cutting_integral}) and (\ref{eq:regular_coef}):
    $$| f * \psi^k_j (u)| \le C 2^{\alpha j}.$$
    If $\gamma$ intersects $B(u, C_{\psi}2^{(1-\epsilon)j})$, the hypothesis on $\gamma$'s tangent imposes that $B(u, C_{\psi}2^{(1-\epsilon)j})$ contains only one connected edge portion, which can be parametrized along the direction $k\pi/4$. Without loss of generality, we assume a horizontal parametrization in what follows, i.e. $k=0$.
    Let $\tilde \gamma$ be the Taylor expansion of order 1 of $\gamma$ in $B(u,C_{\psi}2^{(1-\epsilon)j})$
    : $$\| \gamma - \tilde \gamma \|_{\infty} \le C 2^{(1-\epsilon) \alpha j}.$$
    The function $\tilde \gamma$ being linear, $(u_1, \tilde \gamma(u_1))$ is a straight line, whose angle with the horizontal axis is denoted $\theta$. Thanks to $\psi$ directional vanishing moments, one can consider without loss of generality the case $\theta = 0$. In this case, $|\gamma(\cdot)| = O(2^{(1-\epsilon) \alpha j})$ on $B(u, C_{\psi}2^{(1-\epsilon)j})$.
    One has:
    \begin{equation}
    \label{eq:th_first_order_1D}
        \left| \int_{\| v \| \le C_{\psi} 2^{(1-\epsilon)j}} f(u-v) \psi_j^0(v) dv \right| \le \int\limits_{v_2 =-C_{\psi}2^{(1-\epsilon)j}}^{C_{\psi}2^{(1-\epsilon)j}} \left| \int\limits_{v_1 =-C_{\psi}2^{(1-\epsilon)j}}^{C_{\psi}2^{(1-\epsilon)j}} f(u_1-v_1, u_2-v_2) \psi_j^0(v_1,v_2) dv_1 \right| dv_2.
    \end{equation}
    Thanks to $\psi_j^0$ directional vanishing moments along the horizontal axis, the integral on $v_1$ could be seen as a one-dimensional wavelet coefficient of $v_1 \mapsto f(u-v)$ for a fixed $v_2$. This integral is not taken on the whole space, but one could proceed as in equation (\ref{eq:regular_coef}) to exploit the vanishing moments.
    As illustrated in Figure \ref{fig:first_order_proof}, $v_1 \mapsto f(u-v)$ is discontinuous for $v_2 \in [-C 2^{(1-\epsilon) \alpha j}, C 2^{(1-\epsilon) \alpha j}]$, and $\Ca$ otherwise. The corresponding one-dimensional wavelet coefficients would be $O(1)$ and $O(2^{\alpha j})$ respectively, for an L1 normalized wavelet. $\psi_j^0$ being two dimensional, one should count an additional $2^{-j}$ normalisation factor:
    $$\left| \int\limits_{v_1 =-C_{\psi}2^{(1-\epsilon)j}}^{C_{\psi}2^{(1-\epsilon)j}} f(u_1-v_1, u_2-v_2) \psi_j^0(v_1,v_2) dv_1 \right| \le \left\{
    \begin{array}{ll}
        C \left(2^{\alpha j} + 2^{-j} \right) \text{ if } v_2 \in [-C 2^{(1-\epsilon) \alpha j}, C 2^{(1-\epsilon) \alpha j}], \\
        C \left(2^{\alpha j} + 2^{(\alpha-1)j} \right) \text{ otherwise,}
    \end{array}
    \right.$$
    where the $2^{\alpha j}$ term accounts for $v_1 \notin B(0,C_{\psi}2^{(1-\epsilon)j})$, using equation (\ref{eq:fast_decay}). Integrating on $v_2$ in equation (\ref{eq:th_first_order_1D}) thus leads to:
    $$\left| \int_{\| v \| \le C_{\psi} 2^{(1-\epsilon)j}} f(u-v) \psi_j^0(v) dv \right| \le C \left[2^{(1-\epsilon)j} 2^{(\alpha-1) j}. + 2^{(1-\epsilon)\alpha j}.2^{-j} \right] \le C 2^{[(1-\epsilon) \alpha-1]j}.$$
    Injecting this bound in equation (\ref{eq:cutting_integral}) and replacing $(1-\epsilon) \alpha -1$, $\epsilon>0$, by $s-1$, $s < \alpha$, gives equation (\ref{eq:first_order0}).
\end{proof}

\begin{figure}
        \centering
        \subfigure[]{\label{fig:first_order_proof}\includegraphics[width=75mm]{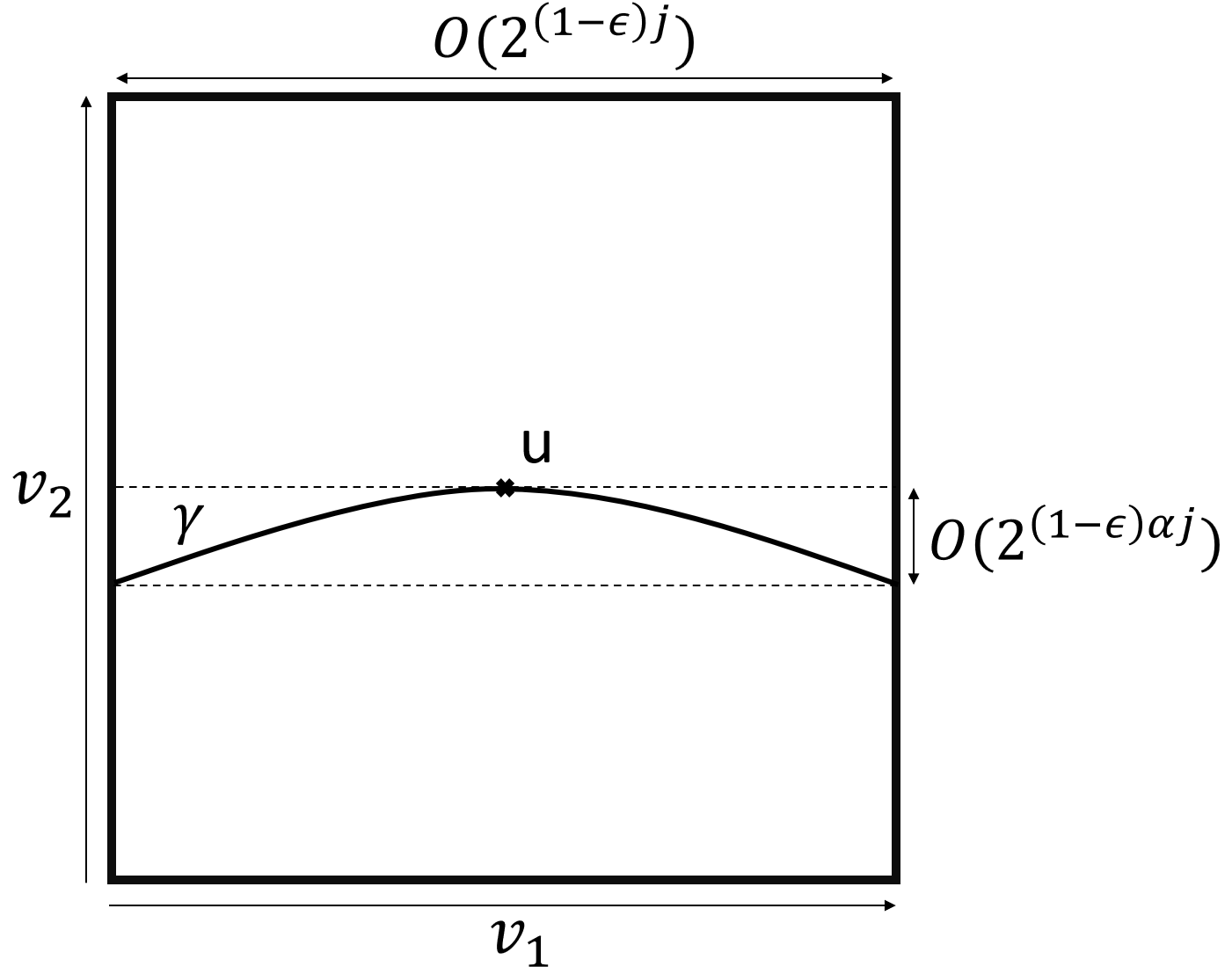}}
        \hspace{1cm}
        \subfigure[]{\label{fig:second_order_proof}\includegraphics[width=74mm]{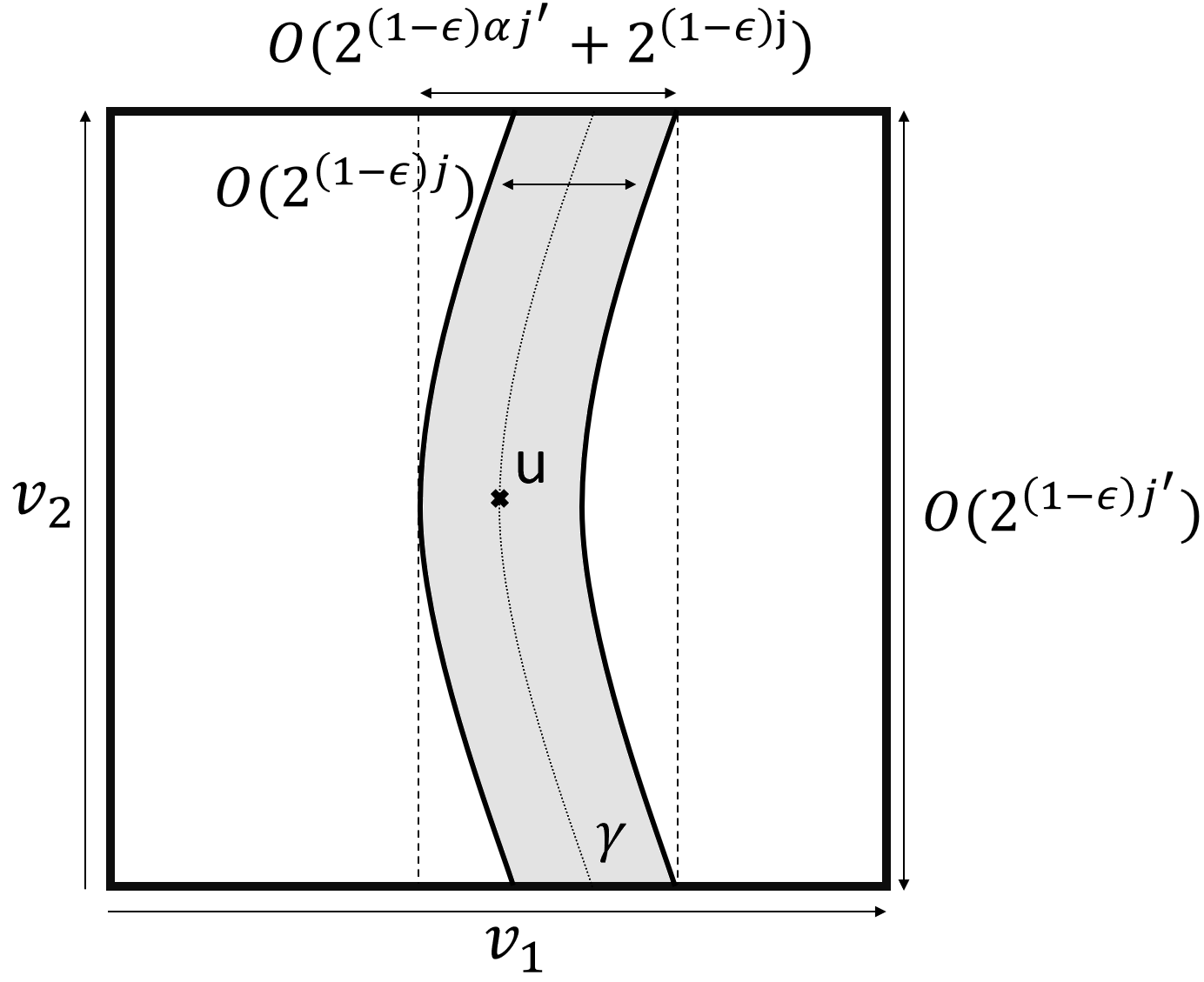}}
        \caption{(a): Representation of $f$ on $B(u, C_{\psi}2^{(1-\epsilon)j})$, for $u$ a point on the edge curve. The tangent of $\gamma$ at $u$ is assumed to be horizontal. Bounding $|f * \psi_j^0|(u)$ requires to exploit (1) the $\Ca$ regularity of $f_{u,v_2} : v_1 \mapsto f(u-v)$ for $v_2$ such that $f_{u,v_2}$ does not intersect $\gamma$, (2) the fact that $f_{u,v_2}$ is discontinuous only for a fraction $O(2^{(1-\epsilon)\alpha j})$ of $v_2$ possible values (between the two dashed lines). (b): Representation of $|f * \psi_j|$ on $B(u, C_{\psi}2^{(1-\epsilon)j'})$. The tangent of $\gamma$ at $u$ is assumed to be vertical. The gray area corresponds to $O(1)$ wavelet coefficients $|f * \psi_j|(v)$ located along the edge. This region has a width $O(2^{(1-\epsilon)j})$ and is $O(2^{(1-\epsilon) j'})$ long. Bounding $|\rho(f * \psi_j) * \psi_{j'}^{\perp}|(u)$ requires to exploit (1) $\rho(f * \psi_j)$ small amplitude far from the edge (outside the region coloured in gray), (2) the fact that $ v_2 \mapsto f(u-v)$ is Lipschitz 1 with a constant controlled thanks to $\gamma$ regularity for a fraction $O(2^{(1-\epsilon) \alpha j'}+2^{(1-\epsilon) j})$ of $v_1$ possible values (between the two dashed lines).}
    \end{figure}

\section{Proof of Theorem \ref{theorem:scattering_bounds}}
\label{chap:thm_proof}

Theorem \ref{theorem:scattering_bounds} states that the decay rate of scattering coefficients of the form $\| |f * \psi_j^k| * \psi_{j'}^{k^{\perp}} \|_1$ depends on the regularity of edge curves. It relies on two main ideas:
\begin{itemize}
    \item first, as shown by Theorem \ref{th:first_order}, the $\Ca$ regularity of an edge can be captured by a wavelet exhibiting enough directional vanishing moments along the edge tangent, if $\alpha \in [1, 2]$. Lemma \ref{lemma:cone} shows how to build wavelets with vanishing moments in a cone of directions, imposing a constraint on their Fourier transform support;
    \item second, since a wavelet cannot have vanishing moments in all directions, the $\ell^1$ norm $\|\cdot * \psi^k\|_1$ will not capture edges regularity due to edge tangents `misaligned' with $\psi^k$ vanishing moments. One can then apply a second wavelet $\psi^{k^{\perp}}$ with complementary directional vanishing moments to get rid of the constraint of tangent alignment.
\end{itemize}
For the sake of clarity, the theorem is proved for the mother wavelet $\psi$, which oscillates along the horizontal axis. This is done without loss of generality, since one can always get back from $\psi^k$ to $\psi$ applying a rotation of angle $-k \pi/4$.
The rotation of angle $\pi/2$ of $\psi$ is denoted $\psi^{\perp}$. Dilated versions of $\psi$ and $\psi^{\perp}$ at each scale $2^j$ are denoted $\psi_j$ and $\psi_j^{\perp}$ respectively. They are L1 normalized:
$$\forall j \le \JJ, ~\|\psi_j\|_1=1.$$
The unit square $[0,1]^2$ is mapped to the torus and periodized wavelets \cite{mallat1999wavelet} are considered. Unless explicitly stated otherwise, the Euclidean distance is used.
In what follows, $C$ is a generic constant which may change from line to line. Theorem \ref{theorem:scattering_bounds}'s proof is in three stages, which are summarized hereafter in a proof sketch.

\paragraph{Directional vanishing moments and edges' regularity}

Let $f$ be a $\Ca[0,1]^2$ geometrically regular function, $\alpha \in [1,2]$, and let $u \in [0,1]^2$. If $u$ is far from $f$'s edges, the wavelet coefficient modulus $|f * \psi_j| (u)$ is controlled thanks to $\psi$ vanishing moments, as exposed in Section \ref{sec:wavelet}. If $u$ is in the vicinity of an edge, one needs to exploit $\psi$ directional vanishing moments. Suppose that the edge tangent at $u$ forms an angle $\theta$ with the horizontal axis and that $\psi$ has more than $2$ vanishing moments along $\theta$. Then, Theorem \ref{th:first_order} proves that the corresponding edge coefficient $|f * \psi_j|(u)$ have a $O(2^{(1-\epsilon)(\alpha-1) j})$ decay rate, $\epsilon>0$, instead of the $O(1)$ obtained without directional vanishing moments. For $\psi$ having vanishing moments in the cone of directions $[-\pi/4,\pi/4]$, edge coefficients are split in two categories:
\begin{itemize}
    \item those for which the edge tangent is in the cone, which have an amplitude $O(2^{(1-\epsilon)(\alpha-1) j})$;
    \item the others, which have an amplitude $O(1)$.
\end{itemize}
Edge coefficients are thus partially sparsified. One should still treat `misaligned' edges.

\paragraph{Second wavelet transform}

Misaligned edge coefficients correspond to edges whose tangents are in the complementary of the cone of angle $[-\pi/4,\pi/4]$. The convolution $f * \psi$ inherits the geometric regularity of $f$, but it cannot be captured directly using the $\psi^\perp$'s directional vanishing moments. Indeed, such wavelets verify
$$\forall j \le j' \le j_M, ~ \psi_{j'}^{\perp} * \psi_j = 0.$$
A non-linearity $\rho$ should be added, leading to coefficients of the form $|\rho(\ccc * \psi_j) * \psi_{j'}^\perp|$, where $\rho$ must be chosen in order to ensure that $\rho(\ccc * \psi_j)$ is still directionally regular. Lemma \ref{lemma:second_order} proves that, for $\rho$ being Lipschitz 1 and verifying $\rho(0)=0$, the scaling of such second order coefficients is in $O(2^{(1-\epsilon)(\alpha-1) jj'})$ in the neighbourhood of edges. In particular, the result holds for $\rho(\cdot) = |\cdot|$ and $\rho(\cdot) = ReLU(\cdot)$.

\paragraph{Theorem's proof} For given $j \le j' \le \JJ$ and $\epsilon >0$, three types of coefficients should be considered:
\begin{itemize}
    \item coefficients scaling in $O(2^{\alpha j})$ in uniformly regular regions, which are of area $O(1)$;
    \item coefficients scaling in $O(2^{(1-\epsilon)(\alpha-1) j})$ in neighbourhoods of edges, which are of area $O(2^{(1-\epsilon) j})$;
    \item coefficients scaling in $O(1)$ in neighbourhoods of junctions, which are of area $O(2^{2 (1-\epsilon) j})$.
\end{itemize}
Thus, for all $\tilde \epsilon>0$, there exists $C>0$ such that $\| \rho(f * \psi_j) * \psi_{j'}^{\perp} \|_1 = C 2^{(1-\tilde \epsilon) \alpha j}$. Replacing $(1-\tilde \epsilon) \alpha$, $\tilde \epsilon>0$, with $s<\alpha$ and summing on $j,j'$ gives equation (\ref{eq:theorem}).

\subsection{Decay of second order scattering coefficients}

\begin{lemma} (partly inspired by \citet{Peyr2008OrthogonalBB}, proposition 1)
\label{lemma:second_order}
Suppose that $\psi$ satisfies the same hypothesis as in Theorem \ref{th:first_order}. Let $f$ be a $\Ca$ geometrically regular image with $\alpha \in [1,2]$ with Lipschitz constant $L$ and $\rho$ be a Lipschitz 1 function verifying $\rho(0)=0$. Let denote $\mathcal C_{\delta} \subset [0,1]^2$ the set of points being $\delta$-close from $f$'s corners.
For all $s < \alpha$ and $\delta > 0$, there exists $C$ such that:
\begin{equation}
\label{eq:lemma_second_order}
\forall j \le j' \le \JJ, ~ \forall u \in [0,1]^2 \backslash \mathcal C_{C_\psi 2^{(1-\delta) j'}},~~|\rho(f * \psi_j) * \psi_{j'}^{\perp}|(u) \le C 2^{(s-1)j'},
\end{equation}
where $C_\psi$ is defined in equation (\ref{eq:fast_decay}).

\end{lemma}

\begin{proof}

If $\alpha=1$, the result is immediate, $\rho(f*\psi_j)$ being bounded and $\psi_{j'}^{\perp}$ being $\Lu$ normalized. In what follows, we thus consider $\alpha \in (1,2]$. Let $u \in [0,1]^2$ and $j \le j' \le \JJ$. For all $\epsilon>0$, one has:
\begin{equation}
\label{eq:bound_coeff_second_order}
        |\rho(f * \psi_j) * \psi_{j'}^{\perp}|(u) \le \left| \int_{\| v \| \le C_{\psi} 2^{(1-\epsilon)j'}} \rho(f * \psi_j)(u-v) \psi_{j'}^\perp(v) dv \right| + \|\rho(f * \psi_j)\|_{\infty} \int_{\| v \| > C_{\psi} 2^{(1-\epsilon)j'}} |\psi_{j'}^\perp(v)| dv.
    \end{equation}
As in the proof of Theorem \ref{th:first_order}, the second integral is controlled thanks to equation (\ref{eq:fast_decay}). Using that $\rho(f*\psi_j)$ is bounded, one gets:
\begin{equation}
\label{eq:bound_tail_second_order}
\|\rho(f * \psi_j)\|_{\infty} \int_{\| v \| > C_{\psi} 2^{(1-\epsilon)j'}} |\psi_{j'}^\perp(v)| dv = O(2^{\alpha j'}).
\end{equation}
If $f$ is uniformly $\Ca$ over the ball $B \left(u, C_{\psi} \left[ 2^{(1-\epsilon)j}+2^{(1-\epsilon)j'} \right] \right)$, then, for all $v \in B \left(0, C_{\psi}2^{(1-\epsilon)j'} \right)$, it is uniformly $\Ca$ over $B \left(u-v, C_{\psi}2^{(1-\epsilon)j} \right)$. Proceeding as in equation (\ref{eq:regular_coef}) thus leads to:
\begin{equation}
\label{eq:lemma_C_uniform_case_firs_order}
\forall v \in B \left(0, C_{\psi}2^{(1-\epsilon)j'} \right), ~~ |f * \psi_j|(u-v) \le C 2^{\alpha j}.
\end{equation}
The function $\rho$ being Lipschitz 1 with $\rho(0)=0$ and $\psi_{j'}^\perp$ being $\Lu$ normalized, we deduce from equation (\ref{eq:lemma_C_uniform_case_firs_order}) a bound on the first integral of equation (\ref{eq:bound_coeff_second_order}):
\begin{equation}
\label{eq:lemma_C_uniform_case}
\left| \int_{\| v \| \le C_{\psi} 2^{(1-\epsilon)j'}} \rho(f * \psi_j)(u-v) \psi_{j'}^\perp(v) dv \right| \le C 2^{\alpha j}.
\end{equation}
Let now assume that $f$ is not uniformly $\Ca$ over $B \left(u, C_{\psi} \left[ 2^{(1-\epsilon)j}+2^{(1-\epsilon)j'} \right] \right)$. The Lipschitz constant $L$ of $f$ is finite and, by definition, its edge curves do no intersect tangentially. Moreover, $u$ is $C_\psi 2^{(1-\delta) j'}$-distant from $f$'s corners and all of the above holds for any $\epsilon>0$. Then, for $j,j'$ small enough and even if it means increasing a multiplicative constant independent of the scale, the ball $B \left(u, C_{\psi} \left[ 2^{(1-\epsilon)j}+2^{(1-\epsilon)j'} \right] \right)$ contains at most one
connected edge portion from an edge curve $\gamma$, which could be parametrized either horizontally or vertically. Without loss of generality, we assume a horizontal parametrization. The angle of its best linear approximation with the horizontal axis is denoted $\theta$. There are two cases, depending wether $\theta$ belongs to $[-\pi/4, \pi/4]$ or not. If $\theta \in [-\pi/4, \pi/4]$, one can apply Theorem \ref{th:first_order} to get: 
\begin{equation}
\label{eq:lemma_C_aligned_case_firs_order}
\forall s<\alpha, ~\exists C>0 ~/ ~ \forall v \in B \left(0, C_{\psi}2^{(1-\epsilon)j'} \right), ~~ |f*\psi_j|(u-v) = C 2^{(s-1)j}.
\end{equation}
Finally, using the same arguments as for equation (\ref{eq:lemma_C_uniform_case}), we deduce from equation (\ref{eq:lemma_C_aligned_case_firs_order}) that:
\begin{equation}
\label{eq:lemma_C_aligned_case_firs_order_bound}
\left| \int_{\| v \| \le C_{\psi} 2^{(1-\epsilon)j'}} \rho(f * \psi_j)(u-v) \psi_{j'}^\perp(v) dv \right| \le C 2^{(s-1)j} \le C 2^{(s-1)j'}.
\end{equation}
If $\theta \notin [-\pi/4,  \pi/4]$, the underlying geometric regularity will be captured by $\psi_{j'}^{\perp}$ instead of $\psi_j$. Thanks to $\psi$ directional vanishing moments, we assume without loss of generality that $\theta = \pi/2$, meaning that $\psi_{j'}^{\perp}$ is aligned with the direction of regularity. Since $\rho$ is not assumed to be differentiable, $f_j := \rho(f * \psi_j)$ does not inherit fully the regularity of $f$ and $\psi_j$. Nevertheless, it inherits from $f * \psi_j$ a Lipschitz constant along the edge tangent. This constant is small, in a sense that will be specified, for $j'$ not too big compared to $j$. The idea of the proof is to exploit this regularity for $j' < j/\alpha$ and to use the fact that $f_j$ is of order 1 only on a small fraction of the support of $\psi_{j'}^{\perp}$ for $j' \ge j/\alpha$. One could proceed as in equation (\ref{eq:regular_coef}) to prove the existence of a constant $C>0$ such that, for all $v \in B\left(0, C_{\psi}2^{(1-\epsilon)j'} \right)$, $f_j(u-v) \le C 2^{\alpha j}$ if the edge curve $\gamma$ does not intersect $B\left(u-v, C_{\psi}2^{(1-\epsilon)j} \right)$. Denoting $\Gamma$ the graph of the edge curve $\gamma$ and using that $\|\psi_{j'}^\perp\|_1=1$, one gets:
\begin{equation}
\label{eq:in_out_gray_region}
\left| \int_{\| v \| \le C_{\psi} 2^{(1-\epsilon)j'}} f_j(u-v) \psi_{j'}^\perp(v) dv \right| \le \left| \int_{\substack{\| v \| \le C_{\psi} 2^{(1-\epsilon)j'} \\ B\left(u-v, C_{\psi}2^{(1-\epsilon)j} \right) \cap \Gamma \neq \emptyset}} f_j(u-v) \psi_{j'}^\perp(v) dv \right| + C 2^{\alpha j}.
\end{equation}
Figure \ref{fig:second_order_proof} illustrates in gray the set $u-\left\{v \in [0,1]^2 ~/~ \| v \| \le C_{\psi} 2^{(1-\epsilon)j'} ~ \text{and} ~ B\left(u-v, C_{\psi}2^{(1-\epsilon)j} \right) \cap \Gamma \neq \emptyset \right\}$. One has:
\begin{equation}
\label{eq:lemma_C_edge_neighborhood}
\int_{\substack{\| v \| \le C_{\psi} 2^{(1-\epsilon)j'} \\ B\left(u-v, C_{\psi}2^{(1-\epsilon)j} \right) \cap \Gamma \neq \emptyset}} dv = O(2^{(1-\epsilon)j+(1-\epsilon)j'}).
\end{equation}
The function $f_j$ is bounded and $\|\psi_{j'}^\perp\|_{\infty} = O(2^{-2j'})$, since it is $\Lu$ normalized. Using equation (\ref{eq:lemma_C_edge_neighborhood}) and assuming that $j' \ge j/\alpha$, one thus gets:
\begin{equation}
\left| \int_{\substack{\| v \| \le C_{\psi} 2^{(1-\epsilon)j'} \\ B\left(u-v, C_{\psi}2^{(1-\epsilon)j} \right) \cap \Gamma \neq \emptyset}} f_j(u-v) \psi_{j'}^\perp(v) dv \right| \le C 2^{(1-\epsilon)(j-j')-2 \epsilon j'} \le C 2^{[(1-\epsilon)(\alpha-1)-2 \epsilon] j'}.
\end{equation}
Replacing $(1-\epsilon)(\alpha-1)-2 \epsilon$, $\epsilon>0$, by $s - 1$, $s < \alpha $, one gets:
\begin{equation}
\label{eq:lemma_C_misaligned_case_big_scales}
\forall s < \alpha, \exists C>0 ~/ ~\left| \int_{\| v \| \le C_{\psi} 2^{(1-\epsilon)j'}} f_j(u-v) \psi_{j'}^\perp(v) dv \right| \le C 2^{(s-1)j'}.
\end{equation}
This bound does not exploit the directional vanishing moments of $\psi$. It is not suitable for $j' < j/\alpha$, since then $j-j'$ will no longer be upper bounded by $(\alpha-1)j'$. For $j' < j/\alpha$, one needs to exploit the directional regularity of $f_j$ to control the remaining integral in equation (\ref{eq:in_out_gray_region}) bound. The first order approximation $\tilde \gamma$ of $\gamma$ have been assumed to be oriented vertically ($\theta = \pi/2$). As in the proof of Theorem \ref{th:first_order}, the maximal distance between $\gamma$ and $\tilde \gamma$ over $B \left(u, C_{\psi}2^{(1-\epsilon)j'} \right)$ is controlled thanks to $\gamma$'s $\Ca$ regularity: $$\| \gamma - \tilde \gamma \|_{\infty} \le C 2^{(1-\epsilon) \alpha j'}.$$
Thus, there exists an interval $I_{j,j'} := \left[-C \left(2^{(1-\epsilon)j}+2^{(1-\epsilon) \alpha j'} \right), C \left(2^{(1-\epsilon)j}+2^{(1-\epsilon) \alpha j'} \right) \right] \subset \left[-C_{\psi} 2^{(1-\epsilon)j'}, C_{\psi} 2^{(1-\epsilon)j'} \right]$ such that:
\begin{equation}
\label{eq:inclusion}
\left\{v \in [0,1]^2 ~/~ \| v \| \le C_{\psi} 2^{(1-\epsilon)j'} ~ \text{and} ~ B\left(u-v, C_{\psi}2^{(1-\epsilon)j} \right) \cap \Gamma \neq \emptyset \right\} \subset I_{j,j'} \times \left[-C_{\psi} 2^{(1-\epsilon)j'}, C_{\psi} 2^{(1-\epsilon)j'} \right].
\end{equation}
The region $I_{j,j'} \times \left[-C_{\psi} 2^{(1-\epsilon)j'}, C_{\psi} 2^{(1-\epsilon)j'} \right]$ is localized between two vertical dashed lines in Figure \ref{fig:second_order_proof}. From equation (\ref{eq:inclusion}), one gets:
\begin{equation}
\label{eq:second_order_1D}
\left| \int_{\substack{\| v \| \le C_{\psi} 2^{(1-\epsilon)j'} \\ B\left(u-v, C_{\psi}2^{(1-\epsilon)j} \right) \cap \Gamma \neq \emptyset}} f_j(u-v) \psi_{j'}^\perp(v) dv \right| \le \int\limits_{v_1 \in I_{j,j'}} \left| \int\limits_{v_2 = -C_{\psi} 2^{(1-\epsilon)j'}}^{C_{\psi} 2^{(1-\epsilon)j'}} f_j(u-v) \psi_{j'}^{\perp}(v) dv_2 \right| dv_1 + C 2^{\alpha j},
\end{equation}
with $f_j$ being directionally regular along $v_2$ (since $\theta = \pi/2$). The $C 2^{\alpha j}$ term accounts for points $v \in I_{j,j'} \times \left[-C_{\psi} 2^{(1-\epsilon)j'}, C_{\psi} 2^{(1-\epsilon)j'} \right]$ such that $B\left(u-v, C_{\psi}2^{(1-\epsilon)j} \right) \cap \Gamma = \emptyset$, for which $f_j(u-v) \le C 2^{\alpha j}$. As in equation (\ref{eq:th_first_order_1D}) from Theorem \ref{th:first_order}'s proof, the integral on $v_2$ could be seen as a 1D wavelet coefficient of $v_2 \mapsto f_j(u-v)$ for a fixed $v_1$, thanks to $\psi$ directional vanishing moments. This integral is not taken on the whole space, but one could proceed as in equation (\ref{eq:regular_coef}) to exploit $\psi_{j'}^{\perp}$ directional vanishing moments and control the integral for $v_2 \notin \left[-C_{\psi} 2^{(1-\epsilon)j'}, C_{\psi} 2^{(1-\epsilon)j'} \right]$. We will use the Lipschitz regularity of this 1D function to get a bound on its wavelet coefficients, as illustrated in Figure \ref{fig:second_order_proof}. Its Lipschitz constant, which depends on $v_1$, will be controlled through a bound on $\left| \frac{\partial f * \psi_j}{\partial u_2} \right|$. The function $f$ is differentiable since $\alpha>1$.
Using triangular and Taylor-Lagrange inequalities, we get:
$$|f_j(u_1-v_1, u_2-v_2)-f_j(u_1-v_1, u_2)| \le C_{j,j'}(v_1) C_{\rho}|v_2|,$$
where $C_{j,j'}(v_1) := \left| \max\limits_{v_2 \in [- C_{\psi} 2^{(1-\epsilon)j'}, C_{\psi} 2^{(1-\epsilon)j'}]} \frac{\partial f*\psi_j(u_1-v_1, u_2-v_2)}{\partial v_2} \right|$ and where $C_{\rho}$ is the Lipschitz constant of $\rho$. The function $v_2 \mapsto f_j(u-v)$ for a fixed $v_1$ is thus Lipschitz 1 with Lipschitz constant $C_{j,j'}(v_1) C_{\rho}$ for $v_2 \in \left[-C_{\psi} 2^{(1-\epsilon)j'}, C_{\psi} 2^{(1-\epsilon)j'} \right]$. Thanks to $\psi_{j'}^{\perp}$ directional vanishing moments along $v_2$, we get:
$$\left| \int\limits_{v_2 = -C_{\psi} 2^{j'}}^{C_{\psi} 2^{j'}} f_j(u-v) \psi_{j'}^{\perp}(v) dv_2 \right| \le C \times C_{j,j'}(v_1) + C 2^{\alpha j'},$$
where the $2^{\alpha j'}$ term accounts for $v_2 \notin \left[-C_{\psi} 2^{(1-\epsilon)j'}, C_{\psi} 2^{(1-\epsilon)j'} \right]$, using equation (\ref{eq:fast_decay}). For an L1-normalized 1D wavelet, the bound would be proportional to $2^{j'}$. Here, this $2^{j'}$ factor, linked to $f_j$ Lipschitz 1 regularity, compensates for the $2^{-2j'}$ normalization of $\psi_{j'}^{\perp}$. Combining this bound with equation (\ref{eq:second_order_1D}) and integrating over $v_1$ leads to:
\begin{equation}
\label{eq:lemma_second_order_intermediate}
\left| \int_{\substack{\| v \| \le C_{\psi} 2^{(1-\epsilon)j'} \\ B\left(u-v, C_{\psi}2^{(1-\epsilon)j} \right) \cap \Gamma \neq \emptyset}} f_j(u-v) \psi_{j'}^\perp(v) dv \right| \le C \times C_{j,j'}^{\text{edge}}   (2^{(1-\epsilon)j}+2^{(1-\epsilon)\alpha j'}),
\end{equation}
where $C_{j,j'}^{\text{edge}} = \max\limits_{v_1 \in I_{j,j'}} C_{j,j'}(v_1)$. One thus needs to control $C_{j,j'}^{\text{edge}}$, which accounts to bound the derivative of $\tilde u \mapsto f * \psi_j(\tilde u)$ along the edge tangent. The computation is similar to the proof of proposition 1 from \citet{Peyr2008OrthogonalBB}, the bound on $f * \psi_j$ derivatives being here obtained on a square region of length $2^{j'}$, $j' < j/\alpha$, instead of $2^{j/\alpha}$. Since $\alpha>1$, $\gamma$ and $f$ along $\gamma$ are differentiable, and one has:
$$\frac{\partial f * \psi_j}{\partial \tilde u_2}(\tilde u) = \frac{\partial}{\partial \tilde u_2} \int f(\tilde u_1-v_1, \tilde u_2 - v_2) \psi_j(v_1, v_2) dv_1 dv_2 = \int \frac{\partial}{\partial \tilde u_2} \left[f(S(\tilde u, v)) \psi_j(T(\tilde u, v)) \right] dv_1 dv_2,$$
where $S(\tilde u, v) := (\tilde u_1+\gamma(\tilde u_2-v_2)-v_1, \tilde u_2-v_2)$ and $T(\tilde u, v) := (v_1-\gamma(\tilde u_2-v_2), v_2)$. The function $f$ being regular along $\gamma$, $\left| \frac{\partial f(S(\tilde u, v))}{\partial \tilde u_2} \right|$ is uniformly bounded for $v_2 \in \left[-C_{\psi} 2^{(1-\epsilon)j}, C_{\psi} 2^{(1-\epsilon)j} \right]$ and $\tilde u_2 \in  \left[-C_{\psi} 2^{(1-\epsilon)j'}, C_{\psi} 2^{(1-\epsilon)j'} \right]$. Moreover, $\gamma$ being $\Ca$, one has for all $\tilde u_2 \in  \left[-C_{\psi} 2^{(1-\epsilon)j'}, C_{\psi} 2^{(1-\epsilon)j'} \right]$: $$\forall v_2 \in \left[-C_{\psi} 2^{(1-\epsilon)j}, C_{\psi} 2^{(1-\epsilon)j} \right], ~\left| \frac{\partial \left[ \psi_j(T(\tilde u, v)) \right]}{\partial \tilde u_2} \right| = \left| \gamma'(\tilde u_2 - v_2) \frac{\partial \psi_j}{\partial \tilde u_1} (T(\tilde u, v)) \right| \le C 2^{(1-\epsilon)(\alpha-1)j'} 2^{-3j}.$$
By construction, $f$ is bounded and $\|\psi_j\|_{\infty} = 2^{-2j}$. Integrating on $v$ and using $\psi_j$'s fast decay through equation (\ref{eq:fast_decay}) then leads to:
$$\left|\frac{\partial f * \psi_j}{\partial \tilde u_2} \right|(\tilde u) \le C \int_{\| v \| \le C_{\psi} 2^{(1-\epsilon)j}} \left( \left|  \psi_j(T(\tilde u, v)) \right| + \left| \frac{\partial \psi_j(T(\tilde u, v))}{\partial \tilde u_2} \right| \right) dv + C 2^{\alpha j} \le C 2^{-(1+2\epsilon)j+(1-\epsilon)(\alpha-1)j'}.$$
Injecting the result in equation (\ref{eq:lemma_second_order_intermediate}) and using that $j'< j/\alpha$, one gets:
\begin{equation}
\label{eq:nearly_final_second_order}
\left| \int_{\substack{\| v \| \le C_{\psi} 2^{(1-\epsilon)j'} \\ B\left(u-v, C_{\psi}2^{(1-\epsilon)j} \right) \cap \Gamma \neq \emptyset}} f_j(u-v) \psi_{j'}^\perp(v) dv \right| \le C 2^{[(1-\epsilon)(\alpha-1)-3 \epsilon \alpha]j'}. 
\end{equation}
Replacing $(1-\epsilon)(\alpha-1)-3 \epsilon \alpha$, $\epsilon>0$, by $s-1$, $s < \alpha$, one gets from equations (\ref{eq:in_out_gray_region}) and (\ref{eq:nearly_final_second_order}):
\begin{equation}
\label{eq:lemma_C_misaligned_case_small_scales}
\forall s < \alpha, \exists C>0 ~/ ~\left| \int_{\| v \| \le C_{\psi} 2^{(1-\epsilon)j'}} f_j(u-v) \psi_{j'}^\perp(v) dv \right| \le C 2^{(s-1)j'}.
\end{equation}
Finally, combining equations (\ref{eq:bound_tail_second_order}), (\ref{eq:lemma_C_uniform_case}), (\ref{eq:lemma_C_aligned_case_firs_order_bound}), (\ref{eq:lemma_C_misaligned_case_big_scales}) and (\ref{eq:lemma_C_misaligned_case_small_scales}) with the bound of equation (\ref{eq:bound_coeff_second_order}) gives the expected result. Equations (\ref{eq:lemma_C_aligned_case_firs_order_bound}), (\ref{eq:lemma_C_misaligned_case_big_scales}) and (\ref{eq:lemma_C_misaligned_case_small_scales}) have been established assuming that $j,j'$ were small enough such that $B \left(u, C_{\psi} \left[ 2^{(1-\epsilon)j}+2^{(1-\epsilon)j'} \right] \right)$ contains at most one
connected edge portion. This limiting scale depends only on $f$, which is fixed. Equation (\ref{eq:lemma_second_order}) thus remains true for $j,j'$ above this scale, even if it means increasing the constant.
    
\end{proof}

\subsection{Main proof}

\begin{theorem*}
Suppose that $\psi$ satisfies the same hypothesis as in Theorem \ref{th:first_order}.
If $f$ is a $\mathrm C^{\alpha}$ geometrically regular function with $1 \le \alpha \le 2$ then, for all $s < \alpha$, there exists $C$ such that:

\begin{equation*}
\forall j \le j' \le \JJ,~ \forall k < 4, ~~ \| |f * \psi_j^k| * \psi_{j'}^{k^{\perp}} \|_1 \leq C 2^{ s j'},
\end{equation*}
and
\begin{equation*}
    \forall s<\alpha,~ \forall k < 4, ~~ \sum_{j=-\infty}^\JJ \sum_{j'=j+1}^{[j/\alpha]} 2^{-s j'} \| |f * \psi_j^k| * \psi_{j'}^{k^{\perp}} \|_1  < + \infty.
\end{equation*}

\end{theorem*}

\begin{proof}

Let $j \le j' \le \JJ$ and $\epsilon>0$. The unit square $[0,1]^2$ is divided in three distinct regions $\mathcal R_{j,j'}^\epsilon$, $\mathcal E_{j,j'}^\epsilon$ and $\mathcal C_{j,j'}^\epsilon$, corresponding to varying degrees of regularity of $f$. The set $\mathcal R_{j,j'}^\epsilon$ contains points $u \in [0,1]^2$ such that $f$ is uniformly regular on the ball $B \left(u, C_{\psi} \left[ 2^{(1-\epsilon)j}+2^{(1-\epsilon)j'} \right] \right)$. Combining equations (\ref{eq:bound_coeff_second_order}), (\ref{eq:bound_tail_second_order}) and (\ref{eq:lemma_C_uniform_case}) from Lemma \ref{lemma:second_order}'s proof, one obtains:
\begin{equation}
\label{eq:thm_regular_points}
\exists C>0 ~ / ~ \forall u \in \mathcal R_{j,j'}^\epsilon, ~||f * \psi_j| * \psi_{j'}^{{\perp}}|(u) \le C 2^{\alpha j'}.
\end{equation}
The set $\mathcal E_{j,j'}^\epsilon$ contains points such that one and only one edge portion intersects $B \left(u, C_{\psi} \left[ 2^{(1-\epsilon)j}+2^{(1-\epsilon)j'} \right] \right)$. Proceeding as in Lemma \ref{lemma:second_order}'s proof and noticing that the leading term in the bound is given by equation (\ref{eq:nearly_final_second_order}), one gets:
\begin{equation}
\label{eq:thm_edge_points}
\exists C>0 ~ / ~ \forall u \in \mathcal E_{j,j'}^\epsilon, ~||f * \psi_j| * \psi_{j'}^{{\perp}}|(u) \le C 2^{[(1-\epsilon)(\alpha-1)-3 \epsilon \alpha]j'}.
\end{equation}
Finally, the set $\mathcal C_{j,j'}^\epsilon$ is defined as the complementary in $[0,1]^2$ of $\mathcal R_{j,j'}^\epsilon$ and $\mathcal E_{j,j'}^\epsilon$. For $j,j'$ small enough, it contains only points at the vicinity of $f$'s corners. The function $f$ being bounded, there is a constant $C>0$ such that:
\begin{equation}
\label{eq:thm_corner_points}
\forall u \in \mathcal C_{j,j'}^\epsilon, ~||f * \psi_j| * \psi_{j'}^{{\perp}}|(u) \le C.
\end{equation}
Combining equations (\ref{eq:thm_regular_points}), (\ref{eq:thm_edge_points}) and (\ref{eq:thm_corner_points}) one thus gets:
\begin{equation}
\label{eq:thm_bound_l1_intermediate}
\exists C>0 ~ / ~ \||f * \psi_j| * \psi_{j'}^{{\perp}}\|_1 \le C \left[ 2^{\alpha j'} \int\limits_{\mathcal R_{j,j'}^\epsilon} du + 2^{[(1-\epsilon)(\alpha-1)-3 \epsilon \alpha] j'} \int\limits_{\mathcal E_{j,j'}^\epsilon} du + \int\limits_{\mathcal C_{j,j'}^\epsilon} du \right].
\end{equation}
The integral $\int\limits_{\mathcal R_{j,j'}^\epsilon} du$ is bounded by the area of the unit square. By construction, $f$ has a finite number of corner points, from what one deduces that $\int\limits_{\mathcal C_{j,j'}^\epsilon} du \le C 2^{2(1-\epsilon) j'}$. Finally, since the length of edge curves is fixed for a given $f$, one has: $\int\limits_{\mathcal E_{j,j'}^\epsilon} du \le C 2^{(1-\epsilon)j'}$. Combining these bounds with equation (\ref{eq:thm_bound_l1_intermediate}) leads to:
$$\| |f * \psi_j| * \psi_{j'}^{\perp} \|_1 \le C_{\text{regular}} 2^{\alpha j'} + C_{\text{edge}} 2^{[(1-\epsilon)(\alpha-1)-3 \epsilon \alpha] j'+ (1-\epsilon) j'} + C_{\text{corners}} 2^{2(1-\epsilon)j'} \le C 2^{ [(1-\epsilon)\alpha-3 \epsilon \alpha] j'}.$$
Replacing $(1-\epsilon)\alpha-3 \epsilon \alpha$, $\epsilon>0$, by $s < \alpha$ gives the first statement of the theorem, even if it means increasing the constant to cover $j,j'$ for which $\mathcal C_{j,j'}^\epsilon$ is not reduced to a union of balls centred on $f$'s corners:
\begin{equation}
\label{eq:thm_first_statement}
\forall s< \alpha, \exists C>0 ~/~ \forall j \le j' \le \JJ, ~~ \| |f * \psi_j| * \psi_{j'}^{\perp} \|_1 \leq C 2^{ s j'}.
\end{equation}
In particular, for any $s'<\alpha$, choosing $s > s' $ in equation (\ref{eq:thm_first_statement}) gives a constant $C>0$ such that:
$$\sum_{j=-\infty}^\JJ \sum_{j'=j+1}^{[j/\alpha]} 2^{-s' j'} \| |f * \psi_j| * \psi_{j'}^{\perp}\|_1 \le C \sum_{j=-\infty}^0 2^{(s-s')j/\alpha} < +\infty,$$
which proves the second statement of the theorem.

\end{proof}

\section{Numerics}
\label{appendix:numerics}

\subsection{Discretising geometry}
\label{appendix:discretising_geometry}
 
\paragraph{Data generation}
Numerical experiments are performed on data generated following an adapted version of \citet{kadkhodaie2024generalization}. It defines a random process whose realizations are geometrically $\Ca [0,1]^2$ functions, projected on a finite dimensional space $\V$ of dimension $d$ (fixed to $128 \times 128$ in our experiments). Samples live on a uniform 2D grid of $\sqrt d$ pixels width - the grid dimension, $2$, and the space dimension, $d$, should not be confused. A realization $P_\V f$ of this process is obtained in four steps. The projections on the 2D grid of two uniformly regular $\Ca[0,1]^2$ functions are first sampled to form the background and the foreground, respectively denoted $B$ and $F$. The contours of the foreground shape, which are the projections of three $\Ca[0,1]$ functions, are then sampled on a 1D grid of length $\sqrt d$. They are combined to create a binary mask  $M$ defining the shape of the foreground on the 2D grid. The foreground is eventually multiplied by the mask $M$ and added to the background:
$$P_\V f = B + M \times F.$$
Denoting $\gamma_1, \gamma_2, \gamma_3$ the contours sampled on the 1D grid and indexing the 2D grid with $(n_1, n_2) \in \llbracket 1, \sqrt d \rrbracket^2$, the mask $M$ is defined as follows:
$$M(n_1, n_2) = \prod\limits_{i} r_{\theta_i} \left(\mathbb 1_{n_2/\sqrt{d} \le \gamma_1(n_1)} \right),$$
where $r_{\varphi}$ is the rotation of angle $\varphi$. For each $i=1,2,3$:
$$\theta_i \sim \mathcal U([(1+3(i-1)) 2\pi/9, (2+3(i-1)) 2\pi/9]).$$
This constraint on contours orientation prevents boundary issues, since most of the time the resulting edges will intersect to form a triangular shape. If this shape crosses the boundary of the grid anyway, this boundary is considered as a fourth contour. This additional contour being a straight line, it is $\C^{\infty}$ and has no impact on the asymptotic denoising performances. The projection on a uniform 2D grid of width $\sqrt d$ (resp. 1D grid of length $\sqrt d$) of a random uniformly $\Ca$ function is obtained convolving a sample of Gaussian white noise of dimension $d$ (resp. $\sqrt d$) and a filter defined in the Fourier space as $(c+\omega_1^2 + \omega_2^2)^{-(\alpha+1)/2}$ (resp. $(c+\omega)^{-(\alpha+1)}$). The parameter $c$ controls the Lipschitz constant of the sampled uniformly regular signals.
These Lipschitz constants define the set $\Lambda \subset \Ld[0,1]^2$ from which functions $f$ are drawn. The previous sampling procedure defines a random process on the projection of such functions in dimension $d$. It induces average Lipschitz constants and contours length, which are reported in table \ref{table:lip_cst} for the different values of $\alpha$ used in our numerical experiments. The concentration of total contours length around its average value justifies the usage of $\epsilon_{mms}$ as a proxy of $\epsilon_m$. Figure \ref{fig:examples_c_alpha} gives examples of sampled images for $\alpha$ varying in $\{1,1.2,1.5,2 \}$.
\begin{table}[ht]
\centering
\begin{tabular}{|l||c|c|c|c|}
\hline
& $\alpha = 1$ & $\alpha = 1.2$ & $\alpha = 1.5$ & $\alpha = 2$ \\ \hline
Lipschtiz constant of contours        & 1.81/0.432     &   3.38/0.853    &   4.10/1.07    &   11.4/2.83  \\ \hline
Lipschitz constant of uniform regions &    16.8/2.84 &    39.8/8.68   &   153/34.8    &   120/32.4  \\ \hline
Contours length &    1.91/0.282 &    1.84/0.261   &   1.81/0.264    &   1.79/0.260  \\ \hline
\end{tabular}
\caption{Lipschitz constant and total contours length of the sampled functions (mean/standard deviation), estimated over $1000$ i.i.d. realizations in dimension $d=128^2$.}
\label{table:lip_cst}
\end{table}

\begin{figure}
    
    \centering
    
    \subfigure[$\alpha=1$]{\includegraphics[width=38mm]{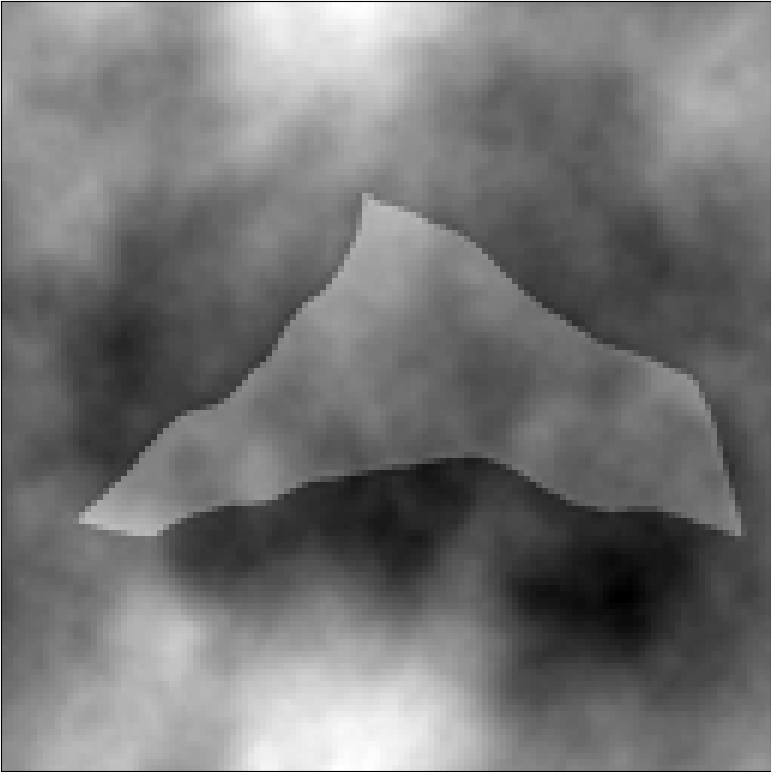}}
    \subfigure[$\alpha=1.2$]{\includegraphics[width=38mm]{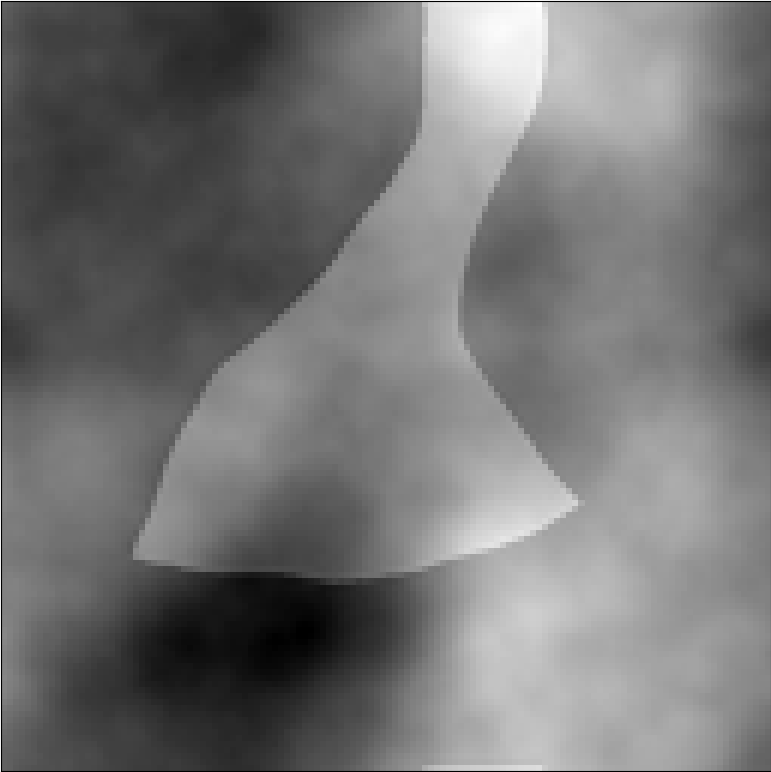}}
    \subfigure[$\alpha=1.5$]{\includegraphics[width=38mm]{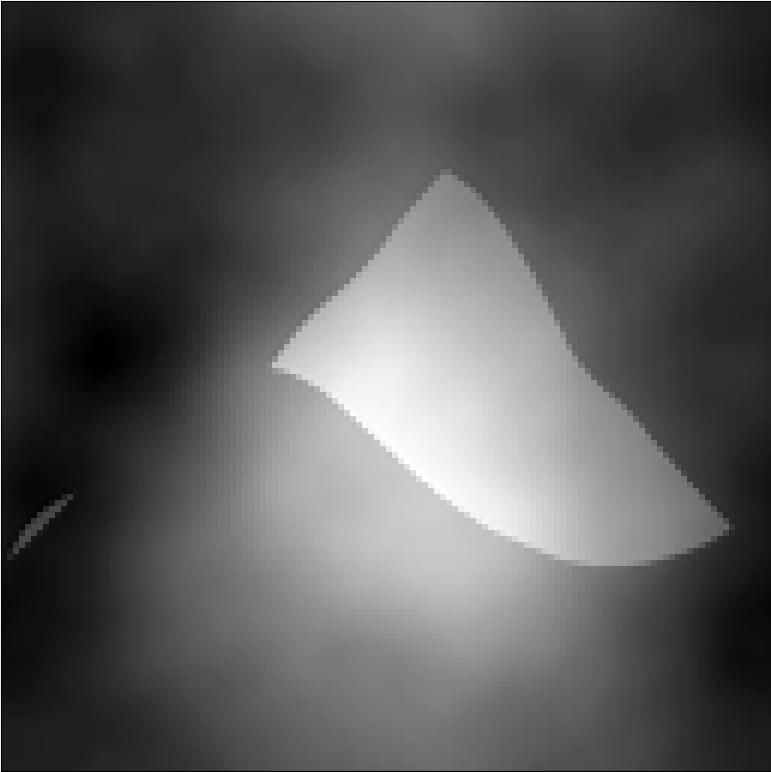}}
    \subfigure[$\alpha=2$]{\includegraphics[width=38mm]{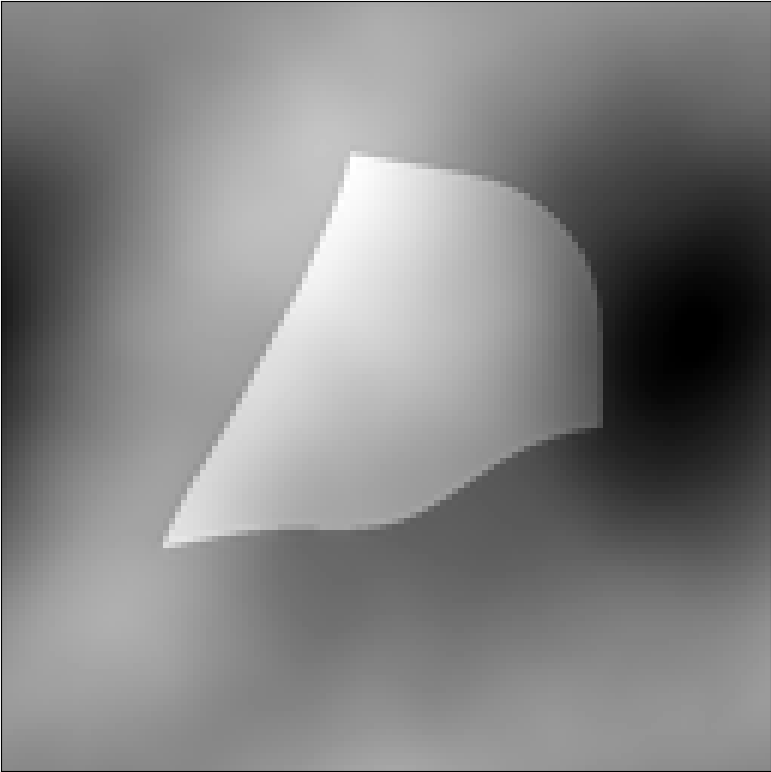}}

    \caption{Examples of $\Ca[0,1]^2$ images discretised in dimension $d=128^2$ and normalized in $[-1,1]$, for $\alpha$ varying in $\{1,1.2,1.5,2 \}$. }

    \label{fig:examples_c_alpha}
\end{figure}

Taken as one-dimensional objects, contours are regular $\Ca[0,1]$ functions. Nevertheless, this regularity is lost when they are discretised on the two-dimensional grid, since stair case discontinuities appear. The central aspect of this work being to study the ability of the scattering estimator to exploit and to be adaptive to frontiers' regularity, this issue must be solved cautiously. To mitigate it, $d$ pixels images are generated as the low frequency components of a two levels FWT, with Symlet4 wavelets, of $4^2 d$ pixels images.

\paragraph{Noise range} Discretization implies that geometric regularity properties of $\Ca$ images used in Section \ref{sec:numerics} are verified only on a given range of scales. This scale range implies a noise range on which one could expect to observe the theoretical MSE slope of $2 \alpha / (\alpha+1)$. Indeed, observing such a slope requires the denoiser to have an uncertainty on the contour's location at a scale for which the contour is still $\Ca$. For the smallest scales, the contour is no longer regular, as illustrated in Figure \ref{fig:coeff_decay}.
\begin{figure}
    \centering
    \includegraphics[width=80mm]{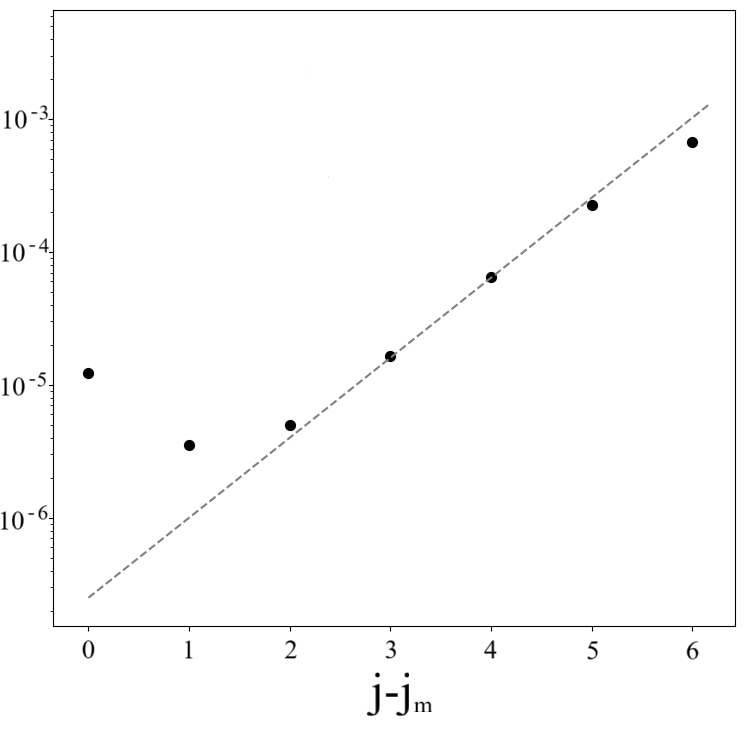}
    \caption{Scaling of $\| |f * \psi_j^k| * \psi_{j}^{k^{\perp}} \|_1$ for $j \in \{\LL,...,\LL+7\}$ and $f$ being a $\mathbf C^2$ geometrically regular image with constant background and foreground - in order to focus on the regularity of the contour - discretised in dimension $d=1024^2$. Theorem \ref{theorem:scattering_bounds} predicts a scaling with a slope of $2$ for such coefficients, which is represented by the gray dotted line. Geometry being discretised, the finest scales do not behave as expected. The y-axis is log-scaled.}
    \label{fig:coeff_decay}
\end{figure}
Consequently, one must be cautious on the value of the gap that separates the background and the foreground. Indeed, if it is too high, the contour could be estimated directly, for instance through a thresholding. There will be no uncertainty on its location (due to the limited precision allowed by the discretization scale), except for very high noise levels, and the MSE will vary as if the image was uniformly regular\footnote{For a wavelet thresholding denoising, all edge coefficients would be very high valued. They will thus get out quickly from the noise, letting only coefficients from uniform regions for smaller noise levels. This issue would not exist in a continuous setting, since there will be edge coefficients of all amplitudes.}. Conversely, if this gap is too low, the image will look uniformly regular up to a certain scale. If this scale is too low compared to the scale range on which the theoretical behaviour is expected, then the experimental MSE slope will not be correct. Correctly bounding the gap between background and foreground is thus necessary so that there exists a noise range over which to observe the correct MSE slope. The precise value of this range will then depend on the performance of the used denoiser. The better the denoiser, the more accurately it will be able to estimate the position of the contour at high noise levels, and the faster the uncertainty about its location will disappear. Thus, applied to the same data as the scattering estimator, the UNet whose performances are reported in Figure \ref{fig:UNet_slopes} has a $2\alpha/(\alpha+1)$ MSE slope for noise levels much higher than those used for the Figure \ref{fig:scattering_denoising_slopes}.
In summary, given the finite size of the denoised images, the theoretical behaviour can only be observed over a range of upper and lower bounded noise. For $\Ca$ geometrically regular images, discretised in dimension $d=128^2$, the gap value between background and foreground has been experimentally set between $0.4$ et $0.6$ for data normalized between $-1$ and $1$. These values were chosen ensuring that the dyadic wavelet estimator MSE has a slope of $1$, while the scattering estimator gives the theoretically expected slope.

\subsection{Denoisers and optimization}
\label{appendix:denoisers_and_optim}

\paragraph{UNet architecture}

The UNet used for experiments of Figures \ref{fig:UNet_slopes} and \ref{fig:MMS-denoising} follows the architecture proposed in \citet{ho2020denoising}. Implementation was obtained from \url{https://huggingface.co/blog/annotated-diffusion}. It has been trained on a dataset of $44000$ $\Ca[0,1]^2$ geometrically regular images discretised in dimension $d=128^2$, composed of $11000$ samples from each degree of regularity $\alpha \in \{1,1.2,1.5,2\}$. It has been split into train ($36000$ samples), validation ($4000$ samples) and test ($4000$ samples) datasets. Early stopping with patience $20$ and delta $0.1$ was used. The network was trained with the Adam optimizer \cite{Kingma2014AdamAM} and a MSE loss function.

\paragraph{Wavelet thresholding denoiser}

Figure \ref{fig:denoised_threshold_image} presents the result of a denoising by thresholding in a Symlet4 wavelet basis. Slopes from Figure \ref{fig:wavelet_trans_inv_slopes} are obtained with a thresholding denoiser, again in a Symlet4 basis, average over 10 translations along each direction. In both cases, thresholding was performed following the rule prescribed by \citet{Donoho1994IdealSA}.

\paragraph{Scattering estimator}

Optimizing parameters $\{\lambda, \gamma, \eta_0, \eta_{\pm 1}\}$ in order to minimize the MSE of the variational denoiser of equation (\ref{eq:variational_denoising}) could be written as a bi-level optimization problem \cite{Zhang2023AnIT}. More precisely, since we want to minimize its MSE for a whole range of noise, we consider the following problem:
\begin{equation}
\label{eq:bi_level_problem}
    \argmin_{\lambda, \gamma, \eta_0, \eta_{\pm 1}} \sum\limits_{\sigma} \log( \mathbb E_{P_\V f,g}(\|\hat f_{\lambda, \gamma, \eta_0, \eta_{\pm 1}}(g)-P_\V f\|_2^2)) \text{,  s.t.  } \hat f_{\lambda, \gamma, \eta_0, \eta_{\pm 1}}(g) = \argmin_{h \in \V} \frac 1 {2} \|h - P_\V g \|^2 + \sigma^2 U_{\lambda, \gamma, \eta_0, \eta_{\pm 1}}(g).
\end{equation}
The need to adapt the scale parameters $\LL$ et $\JJ$ depending on the noise level $\sigma$ prevents the usage of implicit differential methods to solve straightforwardly problem \ref{eq:bi_level_problem} \cite{Zucchet_2022, Dontchev2009ImplicitFA}. As the number of parameters to be optimized was limited, a grid search approach was adopted. For the dyadic wavelet estimator penalization term defined in equation (\ref{eq:scattering_first_order_energy}), there is only one parameter to optimize: $\lambda$. This is done computing the MSE obtained for each $\lambda$ of the grid, at each noise level $\sigma^2 \in \{1.05, 0.67,0.43,0.27, $ $0.18, 0.11, 0.07, 0.05\}$, for the best choice of $\JJ$. The one maximising the criterion of equation (\ref{eq:bi_level_problem}) is then selected, leading to $\lambda = 1.2$.
For the scattering penalization term defined in equation (\ref{eq:scattering_energy}), the same procedure is applied to optimize $\{\lambda, \gamma, \eta_0, \eta_{\pm 1}\}$. 
The criteria of equation (\ref{eq:bi_level_problem}) is computed with three noise levels $\sigma^2 \in \{0.43, 0.11, 0.03\}$. One should notice a subtlety compared to the first order case: $\LL$ requires also to be adjusted for the scattering estimator. Indeed, as the image size is not reduced according to the noise level, $\LL$ must be adapted for the terms controlling the singularities' support, otherwise artificially sharp contours - badly positioned - would be reconstructed.
Finally, the finer scale being special due to the discretization issues discussed in \ref{appendix:discretising_geometry}, the corresponding scattering coefficients - those for which the first convolution is performed with the finest wavelets $\psi^k_{\LL}$ - are regularized with a multiplicative factor $0.55$, which was adjusted empirically.

\subsubsection{Variational denoising optimization}

The optimization problem defined in equation (\ref{eq:variational_denoising}) with $U$ given by equation (\ref{eq:scattering_energy}) is non-standard. Unlike energies defined in equations (\ref{sparse-prior}) or (\ref{eq:scattering_first_order_energy}), the regularizer is non-convex and its proximal operator is unknown. This precludes the use of standard proximal methods such as proximal gradient descent \cite{polson2015proximal}, ISTA \cite{Daubechies2003AnIT}, and pFISTA \cite{liu2016projected}.
The literature proposes extensions of these approaches for non-convex regularizers, but they require assumptions that are not verified by the scattering energy, such as the property to be written as a difference of convex functions \cite{gong2013general}.
Furthermore, while the scattering transform can be interpreted as a two layers neural network, Plug-and-Play (PnP) optimization methods \cite{romano2017little} are not applicable in this context. PnP approaches typically employ neural networks to approximate MSE denoisers rather than energy functionals, with the denoiser subsequently integrated into an ADMM framework \cite{shoushtari2023convergence, sun2021scalable} for denoising \cite{Laumont2022OnMA}. The neural network is generally not differentiated.
Even in cases where neural networks are used to model energy functionals directly, as in \citet{Hurault2021GradientSD}, modifications are required to ensure Lipschitz continuity of the gradient, which is not guaranteed for the scattering transform.
We therefore evaluated several gradient-based optimization methods (L-BFGS \cite{Liu1989OnTL}, Stochastic Gradient Descent, and Adam \cite{Kingma2014AdamAM}), ultimately selecting L-BFGS based on its superior performance. To ensure a fair comparison with the dyadic wavelet estimator defined in equation (\ref{eq:scattering_first_order_energy}), we applied the same optimization algorithm to both approaches.
Given the non-convex nature of the optimization problem, we tested multiple initialization strategies: random initialization, initialization from the noisy image, and initialization from the image denoised using the dyadic wavelet estimator. These different initialization schemes yielded no significant differences in denoising performance.

\bibliographystyle{plainnat}
\bibliography{bibliography}

\end{document}